\documentclass[letterpaper]{article} 
\usepackage{aaai2027}  
\usepackage[hyphens]{url}  
\usepackage{graphicx} 
\usepackage{natbib}  
\usepackage{caption} 
\usepackage[table]{xcolor}
\usepackage{algorithm}
\usepackage{algorithmic}
\usepackage{url}
\usepackage{amssymb}
\usepackage{graphicx}
\usepackage{amssymb}
\usepackage{graphicx}
\usepackage{mathtools}
\usepackage{amsthm}
\usepackage{amsmath}

\usepackage{newfloat}
\usepackage{listings}
\DeclareCaptionStyle{ruled}{labelfont=normalfont,labelsep=colon,strut=off} 
\floatstyle{ruled}
\newfloat{listing}{tb}{lst}{}
\floatname{listing}{Listing}

\usepackage{booktabs}

\usepackage{comment}

\title{Improving Sample Complexity of Critic-Actor Algorithms}
\author{
    Prashansa Panda \textsuperscript{1},
    Shalabh Bhatnagar\textsuperscript{2}
}

\affiliations{
    \textsuperscript{1}Department of Computer Science and Automation, IISc Bengaluru, prashansap@iisc.ac.in\\
    \textsuperscript{2}Department of Computer Science and Automation, IISc Bengaluru, shalabh@iisc.ac.in
}

\begin{document}
\newcommand{\fix}{\marginpar{FIX}}
\newcommand{\new}{\marginpar{NEW}}
\newcommand{\Ab}{\mathbf{A}}
\newcommand{\EE}{\mathbb{E}}
\newcommand{\btheta}{\bm{\theta}}
\newcommand{\bgamma}{\bm{\gamma}}
\newcommand{\bbb}{\mathbf{b}}
\newcommand{\cP}{\mathcal{P}}
\newcommand{\RR}{\mathbb{R}}
\newcommand{\cS}{{\mathcal{S}}}
\newcommand{\cE}{\mathcal{E}}
\newcommand{\PP}{\mathbb{P}}
\newcommand{\norm}[1]{\|#1\|}
\newcommand{\cO}{\mathcal{O}}

\ifx\theorem\undefined
\newtheorem{theorem}{Theorem}

\newtheorem{lemma}{Lemma}
\newtheorem{sublemma}{lemma}[lemma]

\ifx\remark\undefined
\newtheorem{remark}{Remark}

\ifx\proposition\undefined
\newtheorem{proposition}{Proposition}

\ifx\assumption\undefined
\newtheorem{assumption}{Assumption}

\ifx\corollary\undefined
\newtheorem{corollary}{Corollary} 

\maketitle

\begin{abstract}
There has been significant research activity on studying the theoretical convergence guarantees (both asymptotic and non-asymptotic) for actor-critic (AC) algorithms. In a couple of recent papers, two-timescale critic-actor algorithms have been proposed where the role of the actor and critic is reversed, resulting in the algorithm mimicking value iteration as opposed to policy iteration. Finite time analysis has been carried out for a two-timescale critic-actor (CA) algorithm with function approximation and a sample complexity of $\mathcal{O}(\epsilon^{-(2+\delta)})$ is obtained for some $\delta>0$ that is a function of the separation between the two step sizes. In this paper, we obtain improved sample complexity 
of $\mathcal{O}(\epsilon^{-2})$ that had previously been achieved only by the single-timescale actor-critic algorithm. We propose modified learning rates  
for the actor and critic recursions that result in the  improved sample complexity of the algorithm. This we obtain after weakening a key assumption in prior work. The sample complexity we obtain is the best so far reported in the literature for any multi-scale actor-critic or critic-actor algorithm. While the single-timescale actor-critic algorithm also achieves this sample complexity, it has been pointed out in prior work that this algorithm is not guaranteed to be asymptotically stable almost surely unlike our algorithm since we use a timescale separation for the critic and the actor. We additionally fully perform a finite-time analysis of the constrained natural critic-actor algorithm using the modified learning rates and also obtain a sample complexity bound of $\mathcal{O}(\epsilon^{-2})$. We further show the results of experiments on  different Gym and Safety-Gym environments where our algorithms are  found to show good performance and are competitive in comparison to other well-known algorithms.
\end{abstract}


\begin{table*}[t]
\label{sample-table}
\begin{center}
\begin{small}
\begin{tabular}{|p{2 cm}|c|c|c|p{1.6 cm}|}
\toprule
Reference & Algorithm & Sampling  & Sample Complexity & Critic\\     
\midrule
Wu (2022)    & \multicolumn{1}{c|}{Two-timescale AC} & \multicolumn{1}{c|}{Markovian}  & $\tilde{\mathcal{O}}(\epsilon^{-2.5})$ & TD(0) \\ \hline
Olshevsky(2023) &  \multicolumn{1}{c|}{Single-timescale (1AC)} 
& \multicolumn{1}{c|}{i.i.d}   & $\tilde{\mathcal{O}}(\epsilon^{-2})$ & TD(0)\\ \hline
Chen (2024) & \multicolumn{1}{c|}{Single-timescale AC (1AC)}  & \multicolumn{1}{c|}{Markovian}   & $\tilde{\mathcal{O}}(\epsilon^{-2})$ & TD(0)\\ \hline
Suttle (2023) &    \multicolumn{1}{c|}{Two-timescale MLAC}  & \multicolumn{1}{c|}{Markovian}  & $\widetilde{\mathcal{O}}(\tau^{2}_{mix}\epsilon^{-2})$ & MLMC\\ \hline
 Panda (2025)    & \multicolumn{1}{c|}{Two-timescale CA}  & \multicolumn{1}{c|}{Markovian}  & $\tilde{\mathcal{O}}(\epsilon^{-(2+\delta)})$ & TD(0) \\ \hline
 Panda (2024) & \multicolumn{1}{c|}{Three-timescale CAC and CNAC} & \multicolumn{1}{c|}{Markovian}  & $\tilde{\mathcal{O}}(\epsilon^{-(2.5)})$ & TD(0) \\ \hline
 \rowcolor{blue!10} Our work    & \multicolumn{1}{c|}{Two-timescale CA}  & \multicolumn{1}{c|}{Markovian}  & $\tilde{\mathcal{O}}(\epsilon^{-2})$ & TD(0) \\ \hline
 \rowcolor{blue!10} Our work    & \multicolumn{1}{c|}{Three-timescale CNCA}  & \multicolumn{1}{c|}{Markovian}  & $\tilde{\mathcal{O}}(\epsilon^{-2})$ & TD(0) \\
\bottomrule
\end{tabular}
\caption{Comparison with Related Works: \citep{olshevsky} uses Discounted Reward Setting while others are for Average Reward. References are identified by first authors alone here to save space.
}
\end{small}
\end{center}
\end{table*}

\section{Introduction}

Actor-Critic (AC) algorithms are reinforcement learning (RL) methods that incrementally find effective policies under lack of model information. Pure actor-based methods, like REINFORCE, often suffer from high variance  while critic-only approaches such as Q-learning perform well in tabular settings but may become unstable or diverge when combined with function approximation. AC methods mitigate these issues by integrating policy with value based techniques. In this framework, the actor’s role is to learn the optimal policy guided by value estimates from the critic, whereas the critic aims to evaluate the value function for the policy defined by the actor. Asymptotic stability in these algorithms is typically achieved by employing distinct timescales for the updates of the actor and the critic.

The AC framework is structured to emulate the policy iteration (PI) method for Markov Decision Processes (MDPs) \citep{puterman}. AC algorithms employ coupled stochastic recursions that operate on two distinct timescales, with the actor typically updating on a slower schedule than the critic, see for instance, \citep{wu2022finitetimeanalysistimescale}. This separation of timescales results in the AC algorithm mimicking the nested loop structure in policy iteration and plays a crucial role in achieving stability and asymptotic convergence of the iterates. 
Specifically, from the perspective of the faster timescale, the slower (actor) recursion appears quasi-static, while from the slower timescale's viewpoint, the faster (critic) recursion appears equilibrated. This allows AC to mimic the nested loop structure of PI and converge to the optimal policy. 
There has also been significant research activity on  constrained reinforcement learning, driven primarily by applications in safe reinforcement learning (Safe-RL). In a constrained RL framework, each state transition is associated not only with a single-stage reward reflecting the action's effectiveness and the resulting next state, but also with additional single-stage constraint costs that may capture safety considerations. The objective is to maximize the long-term average reward while ensuring that the long-term average constraint costs remain within prescribed thresholds. 

In \citep{bhatnagar2023actorcritic}, a critic-actor (CA) algorithm was introduced under the lookup table setting for the infinite-horizon discounted cost problem. In contrast to conventional AC schemes, the roles of actor and critic were interchanged by reversing their learning rates, with the critic (actor) now operating on the slower (faster) timescale. This reversed configuration leads the CA algorithm to mimic value iteration instead of policy iteration. In \cite{Panda_Bhatnagar_2025}, the asymptotic and non-asymptotic convergence properties of a two-timescale CA algorithm with linear function approximation in the average cost setting have been analyzed using standard learning rates and a sample complexity of $\mathcal{\tilde{O}}(\epsilon^{-(2 + \delta)})$ obtained where $\delta>0$ depends on the difference in learning rates of the actor and critic recursions. 

In this work we propose a new set of learning rates for the two-timescale CA algorithm that we show improves the sample complexity from $\mathcal{\tilde{O}}(\epsilon^{-(2 + \delta)})$ to $\mathcal{\tilde{O}}(\epsilon^{-2})$. Subsequently, we advance the CA framework by proposing the first natural CA algorithm under inequality constraints, which also integrates function approximation and is tailored for the long-run average reward setting. This is the first work that studies such a combination and the first work to propose a constrained CA algorithm. The algorithm functions on three different timescales. The average cost estimate and the actor operate on the fastest timescale, followed by the critic, while the Lagrange multiplier is updated on the slowest timescale. The critic update employs linear function approximation, while the actor uses a natural policy gradient approach. The learning rates for the actor, critic, and average reward estimate are the same as those used in the CA algorithm. We conduct a non-asymptotic analysis of  this algorithm and derive sample complexity bounds. This analysis enables us to determine optimized learning rates for the actor and critic updates. We obtain a sample complexity of  $\mathcal{\tilde{O}}(\epsilon^{-2})$, which matches the sample complexity of the single-timescale unconstrained AC algorithm studied in \cite{chen_zhao}. Algorithms with the modified learning rates that we present have not been previously analyzed for their finite-time performance.

\noindent \textbf{Main Contributions:}
 (\textit{a}) We propose a new set of learning rates (not considered in prior literature) for CA and carry out a finite-time analysis of the algorithm CAM (critic-actor with modified learning rates) under a weaker Assumption~\ref{assum:policy-lipschitz-bounded}. We show that the resulting sample complexity improves upon the best-known sample complexity previously established for this and other two-timescale AC algorithms.
\\
 (\textit{b}) Subsequently, we extend the CA framework by proposing the first constrained natural CA algorithm with modified learning rates (C-NCAM) for constrained MDPs. 
We establish a finite-time (non-asymptotic) convergence analysis of the three-timescale C-NCAM algorithm and again derive a sample complexity of $\tilde{\mathcal{O}}(\epsilon^{-2})$ even for this setting.
\\
(\textit{c}) We also compare the empirical performance of our (unconstrained) CAM and (constrained) C-NCAM algorithms with other well-known algorithms on multiple OpenAI Gym and SafetyGym environments, respectively, and observe that the proposed algorithms are competitive against the other well-studied algorithms and are the best performers in most cases. In the case of CAM, these comparisons are conducted both when neural network based as well as linear function approximators are used.
\\
Finally, we mention that we also carried out the finite-time
analyses of (unconstrained) natural critic-actor with modified learning rates (NCAM)  and constrained critic-actor with modified learning rates (C-CAM), see appendix for a description of the algorithms, and obtained a sample complexity of $\tilde{\mathcal{O}}(\epsilon^{-2})$ for both. We do not show those calculations here for lack of space. 
Note that the notation $\tilde{\mathcal{O}}(\cdot)$ suppresses logarithmic factors.

\section{Other Related Work}

We provide a brief overview of related work. In \citep{konda_actor_critic_type}, two-timescale AC algorithms for look-up table representations were presented and their asymptotic convergence analyzed. Subsequently, in \citep{konda_onactorcritic}, AC algorithms with function approximation based on the Q-value function were proposed and analyzed. A natural gradient-based AC algorithm was presented in \citep{kakade_2001}. Further studies, including \citep{dicastro2009convergent} and \citep{zhang2020provably}, have also conducted asymptotic convergence analyses of AC algorithms. In \citep{BHATNAGAR20092471}, natural AC algorithms were presented that perform bootstrapping in both the actor and critic updates, and asymptotic stability and convergence shown. More recently, \citep{pmlr-v247-zeng24a} proposed a novel two-timescale optimization method that achieves improved convergence speed.

In recent years, substantial research has focused on conducting finite-time analysis of RL algorithms that provides sample complexity estimates and non-asymptotic convergence bounds, offering a more practical understanding of algorithmic performance. Similar analyses have been extended to AC  algorithms, though predominantly in the unconstrained/regular MDP setting. For example, \cite{NEURIPS2020_5f7695de} derive finite-time bounds for a natural policy gradient algorithm applied to discounted-cost MDPs with constraints. \cite{wu2022finitetimeanalysistimescale} present a non-asymptotic analysis of two–timescale (unconstrained) AC under non-i.i.d. sampling, establishing a sample complexity of $\tilde{\mathcal{O}}(\epsilon^{-2.5})$ for convergence to an $\epsilon$-approximate stationary point of the performance objective. 
There have also been some attempts to establish finite-time sample complexity bounds for single–timescale AC algorithms. \cite{chen_zhao} and \cite{olshevsky} establish finite-time convergence results and a sample complexity of $\tilde{\mathcal{O}}(\epsilon^{-2})$ in the case of average-reward and discounted reward settings, respectively. 
\citep{suttleetal} examine the non-asymptotic convergence of the Multi-level Monte Carlo AC algorithm, while \cite{mondal_aggarwal} propose and analyze the convergence of the Accelerated Natural Policy Gradient algorithm. Additional studies have investigated NAC algorithms from a finite-time perspective, see, for instance, \cite{cayci2022finitetimeanalysisentropyregularizedneural}, \cite{xuetal}, \cite{khodadadian}, \cite{khodadadian2021finitesampleanalysisoffpolicynatural}, and \cite{chen2022finitesampleanalysisoffpolicynatural}.

In some of the early work on RL algorithms for constrained MDP problems, \cite{Borkar} introduced the first constrained AC (C-AC) algorithm in the long-run average cost setting and established its asymptotic convergence in the tabular case. 
Subsequently, \citep{Bhatnagar2012OnlineActorCritic} developed a C-AC algorithm for long-run average cost constrained MDPs with function approximation, employing a policy-gradient actor and temporal-difference critic. \cite{panda_and_bhatnagar} showed a finite-time analysis of the three-timescale C-AC and C-NAC algorithms obtaining a sample complexity of $\tilde{\mathcal{O}}(\epsilon^{-2.5})$.
A comparative summary of our results with selected related works, in terms of sample complexity, is provided in Table \ref{sample-table}.
\section{The Framework and Algorithms}

We consider a Markov decision process (MDP) with finite state and action spaces that is characterized by the tuple $(S,A,p,r)$, where 
$S$ denotes the state space, 
$A$ denotes the action space, $p(s,s',a)$ denotes the probability of transition from state $s$ to state $s'$ under action $a$, and  
$r$ is the single-stage reward that depends on the state $s$ and action $a$ at the given instant. We let $\vert r(s,a) \vert \leq U_r$, $\forall s \in S,\forall a \in A$ where $U_r > 0$ is a constant.
We consider stationary randomized policies $\pi_\theta(a|s)$, $a\in A, s\in S$ parameterised by $\theta$.
\subsection{The Unconstrained MDP Setting}
Our aim here is to maximize the long-run average reward (with $\mu_\theta$ being the stationary distribution):
\begin{equation}
    J(\pi) = \lim_{n \rightarrow \infty} \frac{1}{n} 
    \EE\!\left[ \sum_{m=0}^{n-1} r(m) \,\middle|\, \pi \right]. 
    \label{eq1}
\end{equation}
The differential value function $V^{\pi}(s), s\in S$ is defined as  ${ 
    V^{\pi}(s) = \EE\bigg[\sum\limits_{n=0}^{\infty}(r(n) - J(\pi)) |s_0 = s \bigg]}$,
with $s_0$ the starting state, $a_t \sim \pi_{\theta}(\cdot|s_t)$ and $s_{t+1} \sim p( s_t,\cdot, a_t)$).    
The differential action-value (Q-value) function 
is defined as
\begin{align*}
Q_{\pi}(s,a)&= \mathbb{E}[\sum\limits_{n=0}^{\infty}(r(n)-J(\pi))|s_0=s,a_0=a]\\
    &\overset{\text{(i)}}
    {=}r(s,a)-J(\pi)+\mathbb{E}[V^{\pi}(s')],
\end{align*}
where the expectation in (i) is taken over $s'\sim p(s,\cdot,a)$.


We use linear function approximators here for our theoretical results. Such approximators have been found to be theoretically viable for asymptotic analyses. See \citep{tsitsiklisroy2} for an asymptotic analysis of temporal difference learning and \citep{konda_onactorcritic} for analysis of AC algorithms. 
We approximate the state-value function here using a linear approximation architecture as
${\displaystyle
\widehat{V}^{\theta}(s;v)=f_s^\top v}$, 
where $f_s \in \mathbb{R}^{d_1}$ is a feature vector 
$f_s = (f_s(1), f_s(2), \ldots, f_s(d_1))^\top$ associated with state $s$ and $\theta$ is the policy parameter for the considered policy. 

We now recall the CA algorithm with linear function approximation from \cite{Panda_Bhatnagar_2025}.
At each time step $n$, the algorithm maintains the critic parameter $v_n$, the actor parameter $\theta_n$, the average reward estimate $L_n$.
Let $\Gamma : \mathbb{R}^{d_1} \to C$ denote the projection operator that maps any point in $\mathbb{R}^{d_1}$ to its nearest point in a prescribed compact and convex set $C\subset \mathbb{R}^{d_1}$.  
For any $x \in C$, note that $\| x \| \le U_v$ for some constant $U_v > 0$.
\begin{algorithm}[H]
\caption{Two-timescale critic-actor algorithm }
\label{algo_ca}
\begin{algorithmic}[1]

\STATE \textbf{Input}   $v_{0}$, $\theta_{0}$, $L_{0}$,   step-size $a(n)$ for actor , $b(n)$ for critic  and $d(n)$ for average reward estimate.
\STATE Draw $s_0$ from some initial distribution
\FOR{$n > 0 $ and $k = 1,2,\ldots,N$}
    \STATE Sample $a_n \sim \pi_{\theta_n}(\cdot |s_n)$, $s_{n+1}\sim p(s_n, \cdot, a_n)$ 
    \STATE Observe  the reward $r(n)$
    \STATE  $L_{n+1} = L_n + d(n)(r(n)  - L_n)$
    \STATE $\delta_{n} = r(n) - L_n + v_{n}^T(f_{s_{n+1}} - f_{s_n})$
    \STATE $v_{n+1} = \Gamma(v_{n} + b(n)\delta_{n}f_{s_n})$
    \STATE $\theta_{n+1} = \theta_{n} + a(n)\delta_{n}\Psi_{s_{n}a_{n}}$
\ENDFOR
\end{algorithmic}
\end{algorithm}


\subsection{The Constrained MDP Setting}
%
  %
 %
   %
%
Consider a similar MDP as before with 
$r(n)$, $h_1(n)$, $\ldots$, $h_N(n), \; n \geq 0$, denoting the set of reward and costs incurred when transitioning from state $s_n$ to $s_{n+1}$ under action $a_n \in A(s_n)$. 
These quantities depend only on the current state-action pair $(s_n, a_n)$ and given $(s_n,a_n)$, they are independent of all past states and actions $s_m, a_m, \; m < n$. 
For any $s \in S$ and $a \in A(s)$, we define
$
\bar{d}(s,a) = \mathbb{E}\!\left[ r(n) \mid s_n = s, a_n = a \right]$, 
$h_k(s,a) = \mathbb{E}\!\left[ h_k(n) \mid s_n = s, a_n = a \right]$, $k = 1, \ldots, N$. 
Note the slight abuse of notation for the costs $h_k(n)$ and their expectations $h_k(s,a)$.
We assume that all single-stage rewards and costs are real-valued, non-negative, and mutually independent. Further, each is uniformly bounded in absolute value by a common constant $U_c > 0$.

Our objective here is to maximize the reward functional $J(\pi)$ defined (as in the unconstrained case) as  
\begin{align}
    J(\pi) &= \lim_{n \rightarrow \infty} \frac{1}{n} 
    \EE\!\left[ \sum_{m=0}^{n-1} r(m) \,\middle|\, \pi \right] \notag\\
    \label{eq2}
    &= \sum_{s \in S} \mu_\pi(s) \sum_{a \in A(s)} \pi(s,a) \, \bar{d}(s,a),
\end{align}
but subject to the constraints
\begin{align}
    G_k(\pi) &= \lim_{n \rightarrow \infty} \frac{1}{n} 
    \EE\!\left[ \sum_{m=0}^{n-1} h_k(m) \,\middle|\, \pi \right] \notag\\
    \label{eq3_const}
    &= \sum_{s \in S} \mu_\pi(s) \sum_{a \in A(s)} \pi(s,a) \, h_k(s,a) \le \alpha_k,
\end{align}
$k = 1, \ldots, N$, where $\alpha_1, \ldots, \alpha_N$ are given positive threshold values.  
Let $\gamma = (\gamma_1, \ldots, \gamma_N)^T$ denote the vector of Lagrange multipliers, with each $\gamma_k \in \mathbb{R}^+ \cup \{0\}$.  
The Lagrangian $L(\pi, \gamma)$ is then given by
{\small
\begin{align*}
   &L(\pi, \gamma) = J(\pi) - \sum_{k=1}^{N} \gamma_k \big( G_k(\pi) - \alpha_k \big) \\
   &= \sum_{s} \mu_\pi(s) \sum_{a} \pi(s,a) 
      \left( \bar{d}(s,a) - \sum_{k=1}^{N} \gamma_k \big( h_k(s,a) - \alpha_k \big) \right).
\end{align*}}
This formulation converts the original constrained MDP into an unconstrained one, 
with the single-stage reward at time $n$ given by
$
    r(n) - \sum_{k=1}^{N} \gamma_k \big( h_k(n) - \alpha_k \big).
$
The differential action-value function in the relaxed control formulation is defined as
{\small
\begin{align*}
   &M^{\pi,\gamma}(s,a) 
   = \sum_{n=0}^{\infty} \EE\Bigg[ r(n) - \sum_{i=1}^{N} \gamma_i \big( h_i(n) - \alpha_i \big) \\
   &\quad - L(\pi,\gamma)) 
   \,\Big|\, s_0 = s, a_0 = a, \pi \Bigg].
\end{align*}
}
In the constrained setting, the policy gradient of the Lagrangian takes the form \cite{Bhatnagar2012OnlineActorCritic}:
\begin{equation}
\label{pgt}
\nabla_{\theta} L(\theta,\gamma) 
= \sum_{s \in S} \mu_\pi(s) \sum_{a \in A(s)} 
\nabla \pi(a|s) \, \textit{A}^{\pi,\gamma}(s,a),
\end{equation}
where the advantage function for the relaxed formulation is given by
$
\textit{A}^{\pi,\gamma}(s,a) 
= M^{\pi,\gamma}(s,a) - V^{\pi,\gamma}(s),
$ 
and $V^{\pi,\gamma}(s)$ denotes the differential value function for policy $\pi$ and Lagrange multipliers $\gamma$.
By an abuse of notation, we many times use $\theta$ in place of the policy $\pi_\theta$, for instance, $\nabla_\theta L(\theta,\gamma)$ in place of $\nabla_\theta L(\pi_\theta,\gamma)$.  

We employ linear function approximation for $M^{\pi,\gamma}(s,a)$, and let
$
\hat{M}_{w}^{\pi,\gamma}(s,a) \stackrel{\triangle}{=} w^{\pi,\gamma^\top} \Psi_{sa},
$
denote the approximator of $M^{\pi,\gamma}(s,a)$. Here
$w^{\pi,\gamma} \in \mathbb{R}^d$ is the parameter vector and 
$\Psi_{sa}= \nabla_\theta \log \pi_\theta(a|s)$ denotes the compatible feature vector for $(s,a)$. 
Similarly, we approximate the differential value function $V^{\pi,\gamma}(s)$ using
$
    \hat{V}_{v}^{\pi,\gamma}(s)\stackrel{\triangle}{=} v^{\pi,\gamma^\top} f_s,
$
where $v^{\pi,\gamma}, f_s \in \mathbb{R}^{d_1}$ denote the parameter and feature vector, respectively, with $v^{\pi,\gamma}_i, f_s(i)$ denoting their $i$th components.

We now present the C-NCA algorithm.  
At each time step $n$, the algorithm maintains $v_n$ (critic parameter), 
$\theta_n$ (actor parameter),  
$L_n$ (Lagrange estimate),  
$U_k(n)$ (average constraint cost estimate), $k = 1, 2, \ldots, N$,  
$\gamma(n) = (\gamma_1(n), \gamma_2(n), \ldots, \gamma_N(n))^\top$ (Lagrange multiplier estimates),  
and $G(n)$ (Fisher information matrix estimate), respectively. 
We  define $\hat{\Gamma} : \mathbb{R} \to [0, M]$ by  $\hat{\Gamma}(y) = \max\big(0, \min(y, M)\big)$, 
for any $y \in \mathbb{R}$, where $M < \infty$ is a large positive constant.  
This projection ensures that the Lagrange multiplier estimates remain non-negative and bounded.

We initialize $G(0) = p I$, where $I$ is the $d \times d$ identity matrix and $p > 0$ is a constant.  
From the update rule, it follows that $G(n)$ for $n \ge 1$ remains positive definite and symmetric, since each update takes the form  
$(1 - a(n)) G(n-1) + a(n) \Psi_{s_n a_n} \Psi_{s_n a_n}^\top$.  
Consequently, $G(n)^{-1}$ is also positive definite and symmetric for all $n \ge 1$.  
Let $\lambda_i > 0$ denote the smallest eigenvalue of $G(i)^{-1}$, and let $\lambda_G = \min_i \lambda_i > 0$.


\begin{algorithm}[H]
\caption{Three-timescale constrained natural critic-actor  algorithm}\label{algo_ncca}
\begin{algorithmic}[1]

\STATE \textbf{Input}   $v_{0}$, $\theta_{0}$, $L_{0}$, $U_k(0)$ for $1 \leq k \leq N$, $\gamma_k(0)$  for $1 \leq k \leq N$, $G(0)$, step-size $a(n)$ for actor , $b(n)$ for critic, $c(n)$ for Lagrange parameter and $d(n)$ for average cost estimate.
\STATE Draw $s_0$ from some initial distribution
\FOR{$n > 0 $ and $k = 1,2,\ldots,N$}
    \STATE Sample $a_n \sim \pi_{\theta_n}(\cdot |s_n)$, $s_{n+1}\sim p(s_n, \cdot, a_n)$ 
    \STATE Observe  the reward $r(n)$ and constraint costs $h_1(n),h_2(n),.....,h_N(n)$
    \STATE  $L_{n+1} = L_n + d(n)(r(n) - \sum_{k=1}^{N}\gamma_k(n)(h_k(n)-\alpha_k) - L_n)$
    \STATE $\delta_{n} = r(n) - \sum_{k=1}^{N}\gamma_k(n)(h_k(n)-\alpha_k) - L_n + v_{n}^T(f_{s_{n+1}} - f_{s_n})$
    \STATE $v_{n+1} = \Gamma(v_{n} + b(n)\delta_{n}f_{s_n})$
    \STATE $\theta_{n+1} = \theta_{n} + a(n)\delta_{n}G(n)^{-1}\Psi_{s_{n}a_{n}}$
    \STATE $U_{k}(n+1) = U_{k}(n) + a(n)(h_{k}(n) - U_{k}(n))$
    \STATE $\gamma_{k}(n+1) = \hat{\Gamma}(\gamma_k(n) + c(n)(U_k(n) - \alpha_k))$
    \STATE $G(n+1) = (1-a(n))G(n) + a(n)\Psi_{s_{n}a_{n}}\Psi_{s_{n}a_{n}}^T$
\ENDFOR
\end{algorithmic}
\end{algorithm}

\begin{figure*}
\vspace{.3in}
\centering\includegraphics[scale =0.11]{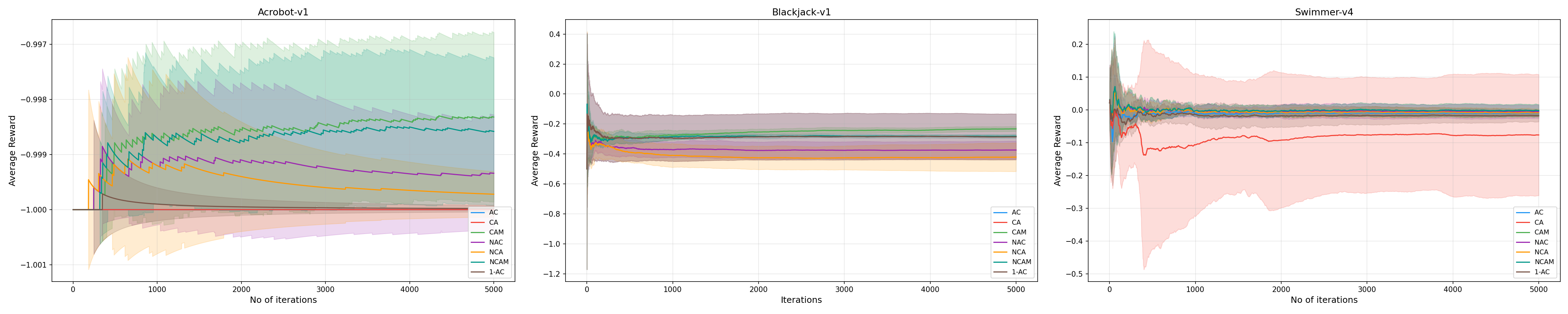}
\vspace{.3in}
\caption{Average Reward as a Function of Number of Iterations for AC, CA, CA-M, N-AC, N-CA and N-CAM Algorithms.}\label{fig:experiments_1}
\end{figure*}
\section{Finite-Time Analysis}

We present here the main theoretical results on the non-asymptotic analysis of both  algorithms, including convergence rates and sample complexity bounds.  
The complete proofs are in the appendix.


\begin{assumption} \label{assum:bounded_feature_norm}
    Each state feature vector is bounded in norm by $1$, i.e., $\Vert f_{i}\Vert \le 1$.
\end{assumption}

Let $v^{*}(\theta)$ denote the convergence point of the critic under the behavior policy $\pi_{\theta}$, given actor and Lagrange parameters $\theta$ and $\gamma$. Define $\bar{\Ab}$ and $\bar{\bbb}$ as
$
    \bar{\Ab} = \EE_{s_n,a_n,s_{n+1}} \big[ f_{s_{n}} \big( f_{s_{n+1}} - f_{s_{n}}\big)^{\top} \big]$, 
    $\bar{\bbb} = \EE_{s_n,a_n,s_{n+1}} \big[(r(s_n,a_n)- L(\theta))f_{s_{n}} \big]$,
where $s_n \sim \mu_{\theta}(\cdot), \; a_n \sim \pi_{\theta}(\cdot | s_n), \; s_{n+1}\sim p(s_n, \cdot, a_n)$. It follows that,
$
    \bar{\Ab} v^{*}(\theta) + \bar{\bbb} = \mathbf{0}.
$
The next assumption ensures the existence and uniqueness of $v^{*}(\theta)$.

\begin{assumption} \label{assum:negative-definite_CA}  
    The matrix $\bar{\Ab}$ above is negative definite, with its largest eigenvalue given by $- \hat{\lambda_e} < 0$, for all $\theta$.  
\end{assumption}

\begin{assumption}[Uniform ergodicity] \label{assum:ergodicity_CA}
    For any parameter $\theta$, the Markov chain, with $a_t \sim \pi_{\theta}(\cdot \mid s_t)$ and $s_{t+1} \sim p(s_t,\cdot,a_t)$, is uniformly ergodic, i.e.,  
    there exist constants $b > 0$ and $k \in (0,1)$ such that  
    \begin{align*}
        d_{TV}\big(p^\tau(x,y,a), \mu_{\theta}(y)\big) \le b \, k^{\tau}, 
        \quad \forall \tau \ge 0, \ \forall x,y \in \cS,
    \end{align*}
where $d_{TV}(\cdot,\cdot)$ indicates the total variation distance and $a \sim \pi_{\theta}(.|x)$.
\end{assumption}

Assumption~\ref{assum:ergodicity_CA} is commonly used to address the challenges arising from Markov sampling in TD learning, see for instance,  \cite{bhandari2018finitetimeanalysistemporal}. 
Let $\tau_{t}$ denote the mixing time of the Markov chain given by 
\begin{align}\label{eq:def_mixing_time_CCA}
    \tau_t & := 
    \min 
    \big \{
   m \ge 0 |
    bk^{m-1} \le
    \min \{ a(t), b(t), d(t) \}
    \big \},
\end{align}
where $b,k$ are defined as in Assumption \ref{assum:ergodicity_CA}.

\begin{assumption} \label{assum:policy-lipschitz-bounded}
 There exist constants $L,B$, $M_{m}$ such that $\forall \theta_1,\theta_2 ,\theta \in \RR^d$, we have
\begin{enumerate}
\item[(a)] $\big\|\nabla \log \pi_{\theta}(a|s) \big\| \le B$, $\forall s,\forall a$.
\item[(b)] $\big\|\nabla \log \pi_{\theta_1}(a|s) - \nabla \log \pi_{\theta_2}(a|s) \big\| \le M_{m} \Vert\theta_1 - \theta_2\Vert$, $\forall s,\forall a$.  
\item[(c)] $\big|\pi_{\theta_1}(a|s) - \pi_{\theta_2}(a|s) \big| \le L \norm{\theta_1 - \theta_2}$, $\forall s,\forall a$.
\item[(d)] There exist scalars $\check{K}, \hat{K}>0$ such that for any $x\not=0$ and all $s_n,a_n$,
\[
\check{K}\| x\|^2 \leq x^T \Psi_{s_na_n}\Psi_{s_na_n}^T x \leq \hat{K}\|x\|^2.
\]
\end{enumerate}
\end{assumption}

Assumption \ref{assum:policy-lipschitz-bounded} ensures smoothness of the parameterized policies and is satisfied by the softmax policies, see Lemma~\ref{lem:softmax-policy-properties} (Appendix) for a proof. 
It may be noted that Assumption~\ref{assum:policy-lipschitz-bounded}(b) is much weaker than Assumption~4(b) of \cite{Panda_Bhatnagar_2025}. The latter does not hold for softmax policies.

\subsection{ Finite-Time Convergence Results for CAM Algorithm}
The sample complexity obtained in \cite{Panda_Bhatnagar_2025} of $\tilde{\cal O}(\epsilon^{-(2+\delta)})$ is the best when $\delta=0$, a case that however corresponds to single-timescale actor-critic as in such a case $\sigma = \nu$, a condition that implies the same timescale used for both actor and critic recursions. As noted in \cite{Panda_Bhatnagar_2025}, asymptotic stability guarantees cannot be provided for such an algorithm as a step-size difference ensures that the algorithm mimics policy iteration. This motivates us to look for different step-sizes that respect the step-size difference while giving a sample complexity of $\tilde{\cal O}(\epsilon^{-2})$. A key observation here is that we need the degree of $a(n)$ and $b(n)$ to be the same while requiring $\lim_{n\rightarrow\infty}\frac{b(n)}{a(n)} = 0$. We thus introduce a function $f(n)$ such that $a(n) = \frac{c_\alpha f(n)}{(1+n)^{\nu}}$, where $f(n)$ is an increasing function of $n$ while growing asymptotically slower than $(1+n)^\nu.$  The latter is needed as the sample complexity depends on the degree of polynomial terms in the learning rates and so we can ignore $f(n)$ while focusing on the polynomial terms for sample complexity. Hence a choice is of $f(n)=(\ln (n + 1))^{y}$ where $y > 0$. We let $f(n) = (\ln (n + 1))^{1/2}$ as it simplifies the task of estimating upper bounds on some error terms. This choice of step-sizes does satisfy $b(n)=o(a(n))$ as desired. 

Our next task is to verify whether, under these learning rates, we can establish convergence of the actor, critic, and average cost estimates, since simply setting 
 ${\delta} = 0$ does not work.
We provide below the results of the finite-time analysis after incorporating the modified learning rates: $a(n) = \frac{c_a (\ln (n + 1))^{1/2}}{(1+n)^\nu}$, $b(n) = \frac{c_b}{(1+n)^\nu}$,  $d(n) = \frac{c_d (\ln (n+1))^{1/2}}{(1+n)^\nu}$, $n\geq 0$, where $0.5 \leq \nu \leq 1$. We let $c_a,c_b,c_d >0$ and  ${\displaystyle \frac{c_a}{c_d} <  \frac{1}{2BG }}$ where $G$ is some positive constant as follows: 
$G := 2B(U_r + U_v)$,  
$\vert r(s,a)\vert \leq U_r$, $\forall s \in S, a \in A$. Note here that the timescales obtained from the step-size sequences $\{a(n)\}$ and $\{d(n)\}$ are the same.

\begin{theorem}[Convergence of average cost estimate]\label{cost_convergence_1}

Under Assumptions \ref{assum:bounded_feature_norm}, \ref{assum:ergodicity_CA}, \ref{assum:policy-lipschitz-bounded}, the following holds:
{\small
\begin{align*}
  & \frac{1}{1+t-\tau_t} \sum\limits_{k=\tau_t}^{t} \mathbb{E}[y_k^2]
  = \mathcal{O}(\log^{-0.5} t \cdot  t^{\nu -1}) + \mathcal{O}(\log^{2.5} t \cdot t^{-\nu})\\&
    + \mathcal{O}\bigg(\frac{1}{(1 + t - \tau_t)}\sum\limits_{k=\tau_t}^{t}\mathbb{E}\Vert M(\theta_k,v_k)\Vert^2\bigg),
\end{align*}}
where $y_t = (L_t - L(\theta_t))$, $ \bar{M}(\theta_t,v_t)  = \mathbb{E}_{s_t \sim \mu_{\theta_t},a_t \sim \pi_{\theta_t},s_{t+1} \sim p}[( r(s_t,a_t)- L(\theta_t) + f_{s_{t+1}}^{\top} v_{t} - f_{s_t}^{\top} v_{t})\nabla \log\pi_{\theta_t}(a_t|s_t)]$ respectively.
\end{theorem}

\begin{proof}
See the supplementary material for the proof.
\end{proof}
\begin{theorem}[Convergence of actor]\label{actor_convergence_1}
Under Assumptions \ref{assum:bounded_feature_norm},  \ref{assum:ergodicity_CA}, \ref{assum:policy-lipschitz-bounded}, the following holds:
\begin{align*}
    &\frac{1}{1+t-\tau_t}\sum\limits_{k=\tau_t}^{t}\mathbb{E}\Vert M(\theta_k,v_k) \Vert^2\\
    &= \mathcal{O}((\log t)^{-0.5} \cdot t^{\nu - 1}) + \mathcal{O}(\log^{2.5} t \cdot t^{-\nu}) .
\end{align*}
\end{theorem}

\begin{theorem}[Convergence of critic]\label{critic_convergence_1}
Under Assumptions \ref{assum:bounded_feature_norm}, \ref{assum:negative-definite_CA},  \ref{assum:ergodicity_CA}, \ref{assum:policy-lipschitz-bounded}, the following holds:
\begin{align*}
    &\frac{1}{1+t-\tau_t}\sum\limits_{k=\tau_t}^{t}\mathbb{E}\Vert v_k - v^{*}(\theta_k) \Vert^2 =  \mathcal{O}( t^{\nu}) +  \mathcal{O}(\log^{3} t \cdot t^{1-\nu}).
\end{align*}
 \end{theorem}

By optimizing over the parameter $\nu$, we obtain 
$\nu = 0.5$ . Consequently, we arrive at
\[
\frac{1}{1+t-\tau_t}\sum_{k=\tau_t}^{t}\mathbb{E}\,\|z_k\|^2 
= \mathcal{O}\!\left(\log^2 t \cdot t^{\,- 0.5}\right),
\]
where $z_k =  v_k - v^{*}(\theta_k) .$
Thus, in order for the mean squared error of the critic to be upper bounded by $\epsilon$, namely,
\[
\frac{1}{1+t-\tau_t}\sum_{k=\tau_t}^{t}\mathbb{E}\,\|z_k\|^2 
= \mathcal{O}\!\left(\log^2 T \cdot T^{\,- 0.5}\right) \;\leq\; \epsilon,
\]
it suffices to take ${\displaystyle 
T \;=\; \tilde{\mathcal{O}}\!\left(\epsilon^{-2}\right)}$.

This sample complexity matches that of the single-timescale AC algorithm (see \cite{chen_zhao}) even though our algorithm incorporates two timescales as opposed to one required by the aforementioned reference. As pointed out in \cite{Panda_Bhatnagar_2025}, it is important to note that asymptotic stability of the iterates is not assured in \cite{chen_zhao} as the algorithm in this case does not mimic a policy iteration procedure.

\subsection{Finite-Time Convergence Results for C-NCAM Algorithm}

We apply TD(0) with function approximation for the critic recursion, which estimates the state-value function. Let $v^{*}(\theta,\gamma)$ denote the convergence point of the critic under the behavior policy $\pi_{\theta}$, given actor and Lagrange parameters $\theta$ and $\gamma$. Define 
$\bbb = \EE_{s_n,a_n,s_{n+1}}[(C(s_n,a_n,\gamma)- L(\theta,\gamma))f_{s_{n}}]$,
where $s_n \sim \mu_{\theta}(\cdot), \; a_n \sim \pi_{\theta}(\cdot | s_n), \; s_{n+1}\sim p(s_n, \cdot, a_n)$, and $C(s_n,a_n,\gamma) = \bar{d}(s_n,a_n) + \sum_{k=1}^{N}\gamma_k\big(h_k(s_n,a_n) - \alpha_k\big)$,
denotes the single-stage cost for the relaxed problem. Analogous to the unconstrained case (see \cite{Bhatnagar2012OnlineActorCritic}), it follows that
$
    \bar{\Ab} v^{*}(\theta,\gamma) + \bbb = \mathbf{0},
$
where $\bar{\Ab}$ is as in Assumption~\ref{assum:negative-definite_CA} that we assume holds here as well.
 
\begin{table*}
  \centering
  \begin{tabular}{|p{2 cm}|p{1.5 cm}|p{1.5 cm}|p{1.5 cm}|p{2 cm}|p{2 cm}|p{2 cm}|p{1.5 cm}|}
    \hline
    \textbf{Environment} & \textbf{AC} & \textbf{CA} & \textbf{CAM} & \textbf{NAC} & \textbf{NCA} & \textbf{NCAM} & \textbf{1AC}\\
    \hline
    Acrobot-v1 & $-1.0000 \pm 0.0000$ & $-1.0000 \pm 0.0000$ & $\textbf{-0.9983} \pm \textbf{0.0015}$ & $-0.9993 \pm 0.0011$ & $-0.9997 \pm 0.0004$ & $-0.9986 \pm 0.0013$ & $-1.0000 \pm 0.0001$\\ \hline
    Blackjack-v1 & $-0.285 \pm 0.153$ & $-0.285 \pm 0.153$ & $\textbf{-0.233} \pm \textbf{0.016}$ & $-0.374 \pm 0.059$ & $-0.421 \pm 0.096$ & $-0.282 \pm 0.013$ & $-0.285 \pm 0.153$\\ \hline
   Swimmer-v4 & $-0.0109 \pm 0.0079$ & $-0.0769 \pm 0.1847$ & $\textbf{-0.0024} \pm \textbf{0.0187}$ & $-0.0058 \pm 0.0204$ & $-0.0093 \pm 0.0123$ & $\textbf{-0.0024} \pm \textbf{0.0192}$ & $-0.0178 \pm 0.0185$\\
    \hline
  \end{tabular}
  \caption{Comparison of  Critic-Actor Modified (CAM) with different algorithms in terms of average reward $\pm$ standard error upon convergence.}
  \label{tab:experiment_1}
\end{table*}




Let $\tau_{t}$ be the mixing time of our ergodic Markov chain. For $\bar{b},k$ defined as in Assumption \ref{assum:ergodicity_CA}, we have
\begin{align}\label{eq:def_mixing_time}
    \tau_t & := 
    \min 
    \big \{
   m \ge 0 \mid
    \bar{b}k^{m-1} \le
    \min \{ a(t), b(t),c(t), d(t) \}
    \big \}.
\end{align}


The following are the learning rates that we incorporate in this algorithm: $a(n) = \frac{c_a (\ln (n + 1))^{1/2}}{(1+n)^\nu}$, $b(n) = \frac{c_b}{(1+n)^\nu}$, $c(n) = \frac{c_c}{(1+n)^\beta}$, $d(n) = \frac{c_d (\ln (n+1))^{1/2}}{(1+n)^\nu}$, $n\geq 0$, where $0.5 \leq \nu < \beta \leq 1$.  We  have  $c_a,c_b,c_c,c_d >0$ and  ${\displaystyle \frac{c_a}{c_d} <  \frac{\lambda_G}{2BU_G G}}$, where $G$ and $U_G$ are some positive constants as follows: $ G := 2B(U_r + U_v),\vert \bar{d}(s,a) + \sum_{k=1}^{N}\gamma_k(t)(h_{k}(s,a)-\alpha_{k})\vert \leq U_r, \forall s \in S, a \in A, \gamma \in \RR^{N},\sup\limits_{t}\Vert G(t)^{-1} \Vert \leq U_G.$

Effectively, $a(n)$ and $d(n)$ differ only in a constant term and constitute the same timescale. Recall that the average cost recursion $L_n,n\geq 0$ incorporates the step-size schedule $d(n), n\geq 0$, while the policy parameter $\theta_n$ (that is updated here on the faster timescale) incorporates $a(n), n\geq 0$ as the step-size schedule.
Moreover, the value function parameter $v_n$ updates involve the step-size $b(n)$ and the Lagrange parameter updates $\gamma_k(n)$ incorporate the step-size $c(n)$. 
For $\nu>0.5$, all these (modified) step-sizes satisfy the Robbins-Monro conditions for asymptotic convergence of stochastic approximation. 
The timescale difference in step-size schedules is satisfied as 
${\displaystyle
\lim_{n\rightarrow\infty}\frac{b(n)}{a(n)} =
\lim_{n\rightarrow\infty} \frac{c(n)}{b(n)} = 0.}$ 
This indicates in effect that the average reward and actor updates together proceed on the faster timescale, the critic update proceeds on a slower timescale, while the Lagrange parameter update proceeds on the slowest timescale. Such a structure of a constrained critic-actor algorithm had previously not been explored in the literature. 
We now examine the sample complexity of the C-NCAM algorithm (i.e., the constrained NCA with modified learning rates).

\begin{theorem}[Convergence of average cost estimate]\label{cost_convergence_2}
 Under Assumptions \ref{assum:bounded_feature_norm},   
 \ref{assum:ergodicity_CA}, 
 \ref{assum:policy-lipschitz-bounded}, the following holds:
{\small
\begin{eqnarray*}
&\frac{1}{(1 + t - \tau_t)}\sum\limits_{k=\tau_t}^{t} \mathbb{E}[y_k^2]
    \leq \mathcal{O}(\log^{-0.5} t \cdot  t^{\nu -1}) + \mathcal{O}(\log^{2.5} t \cdot t^{-\nu})\\ 
    &+ \mathcal{O}(t^{ \nu - \beta})
    + \mathcal{O}\bigg(\frac{1}{(1 + t - \tau_t)}\sum\limits_{k=\tau_t}^{t}\mathbb{E}\Vert \bar{M}(\theta_k,v_k,\gamma(k))\Vert^2\bigg),
\end{eqnarray*}}
where $y_t = (L_t - L(\theta_t,\gamma(t)))$, $ \bar{M}(\theta_t,v_t,\gamma(t))  = \mathbb{E}_{s_t \sim \mu_{\theta_t},a_t \sim \pi_{\theta_t},s_{t+1} \sim p}[( r(s_t,a_t,\gamma(t))- L(\theta_t,\gamma(t)) + f_{s_{t+1}}^{\top} v_{t} - f_{s_t}^{\top} v_{t})\nabla \log\pi_{\theta_t}(a_t|s_t)]$, and $r(s_t,a_t,\gamma(t)) = \bar{d}(s_t,a_t) + \sum_{k=1}^{N}\gamma_k(t)(h_{k}(s_t,a_t)-\alpha_{k})$, respectively.
\end{theorem}

\begin{theorem}[Convergence of actor]\label{actor_convergence_2}
Under Assumptions \ref{assum:bounded_feature_norm},  \ref{assum:ergodicity_CA}, \ref{assum:policy-lipschitz-bounded}, the following holds:
\begin{align*}
     &\frac{1}{(1 + t - \tau_t)}\sum\limits_{k=\tau_t}^{t}\mathbb{E}\Vert \bar{M}(\theta_k,v_k,\gamma(k))\Vert^2\\
     &= \mathcal{O}((\log t)^{-0.5} \cdot t^{\nu - 1}) + \mathcal{O}(\log^{2.5} t \cdot t^{-\nu}) + \mathcal{O}(t^{ \nu - \beta}).
 \end{align*}
\end{theorem}

\begin{theorem}[Convergence of critic]\label{critic_convergence_2}
Under Assumptions \ref{assum:bounded_feature_norm}, \ref{assum:negative-definite_CA}, \ref{assum:ergodicity_CA}, and \ref{assum:policy-lipschitz-bounded}, the following holds
with $z_k =  v_k - v^{*}(\theta_k,\gamma(k)).$
\begin{eqnarray}
\nonumber
\frac{1}{1+t-\tau_t}\sum\limits_{k=\tau_t}^{t}\mathbb{E}\Vert z_k \Vert^2
    &=& \mathcal{O}( t^{\nu - 1}) +  \mathcal{O}(\log^{3} t \cdot t^{-\nu})\\
    \label{eqx}
    &&+ \mathcal{O}(\log^{0.5} t \cdot t^{\nu - \beta}).  
\end{eqnarray}
\end{theorem}
It can be seen that $\nu = 0.5$ and $\beta = 1$ provide the optimal values. Hence, 
$
    \frac{1}{1+t-\tau_t}\sum_{k=\tau_t}^{t}\mathbb{E}\Vert z_k \Vert^2 = \mathcal{O}(\log^3 t \cdot t^{ - 0.5}).
$
Therefore, in order for the mean squared error of the critic (LHS of (\ref{eqx})) to be upper bounded by $\epsilon$, 
we need to set $T = \tilde{\mathcal{O}}(\epsilon^{-2})$.
This rate had previously only been obtained in the case of single-timescale actor-critic algorithms where asymptotic stability of the iterates is not assured. As argued in \cite{Panda_Bhatnagar_2025} that multi-timescale recursions alone can guarantee asymptotic stability. Our algorithm with these step-sizes thus obtains optimal rates of convergence while ensuring algorithmic stability.

\section{Experiments}
We show the results of experiments for both unconstrained and constrained MDP algorithms. For the unconstrained MDP setting, these experiments are conducted on 
three OpenAI Gym environments, namely (a) Acrobot-v1, (b) Blackjack-v1 and (c) Swimmer-v4 , while for the constrained setting, these are shown on 
SafetyGym environments, namely (a) SafetyAntCircle1-v0, (b) SafetyCarGoal1-v0, and (c) SafetyPointPush1-v0, respectively.  Each experiment is run with ten different initial seeds and the mean and standard error from these experiments are shown.
We evaluate the performance of the CAM algorithm
in comparison with NCAM, CA, AC \cite{wu2022finitetimeanalysistimescale}, NAC, and 1AC (single-timescale AC of \cite{chen_zhao}) algorithms with both neural network and linear function approximators. 
The performance comparisons for the unconstrained algorithms are provided in Figure~\ref{fig:experiments_1}
and Table~\ref{tab:experiment_1} for the settings with neural network approximators. Performance comparisons on the same settings with linear function approximators can be found in Figure~\ref{fig:experiments_CA_L} and Table~\ref{tab:linear_critic_reward} (see appendix in supplementary material).

For the constrained algorithms, we compare the performance of Constrained NCAM (C-NCAM) with Constrained CAM (C-CAM), Constrained NCA (C-NCA), Copnstrained AC (C-AC), Constrained NAC (C-NAC), and Constrained CA (C-CA).
The performance comparisons in this case are provided in Figure~\ref{fig:experiments} and Tables~\ref{tab:experiment}-\ref{tab:constraint_cost} in the appendix.
The variants with modified learning rates use the new learning rate schedules presented earlier.

All the experimental plots are generated by averaging over ten different initial seeds. For the policy neural network, we used a single hidden layer and performed hyper-parameter tuning by varying the number of hidden nodes between 16, 32, and 64. The same approach was applied to the value function network. Details of the hyper-parameters used in the various settings and algorithms can be seen in Tables 6--9 and Table 10 contains the server infrastructure details (see Appendix).

The performance of the various algorithms is compared by showing the average reward together with the corresponding standard errors. 
It can be seen that CAM achieves the best performance in all three environments under neural network approximators.  
Also, NCAM consistently outperforms NCA across all environments. On Swimmer-v4, NCAM is comparable with CAM in performance. In the case when linear function approximators are used in the same setting, we observe that CAM is the best performer in two of the three settings while NAC shows the best results in one setting.

 Plots in the top row in Figure \ref{fig:experiments} (see appendix in the supplementary material) are for the average reward performance while those in the bottom row are for the constraint costs. 
 In the lower row of the plots in Figure~\ref{fig:experiments}, the horizontal dotted red line with a value $0.2$ represents the constraint cost threshold. All algorithms are seen to asymptotically satisfy this threshold in two of the environments while simultaneously optimizing for the average reward performance. In the SafetyAntCircle1-v0 environment, all algorithms satisfy the threshold except the C-NAC algorithm. 
It can be seen that  C-NCAM outperforms the other algorithms on two of the three settings here while being competitively close on the third (SafetyAntCircle1-v0) environment. Thus, we also observe experimental benefits in using the proposed learning rates in CAM and C-NCAM algorithms. 


\section{Conclusions}
We proposed a new set of learning rates for the two-timescale Critic–Actor algorithm with linear function approximation and for the long-run average-reward setting and performed a finite-time analysis which resulted in a sample complexity of $\tilde{\mathcal{O}}(\epsilon^{-2})$.
 We then presented the first Constrained Natural Critic-Actor algorithm with the new learning rates under linear function approximation for the long-run average-reward setting. By performing a finite-time analysis with the proposed  learning rates, we again established a sample complexity of $\tilde{\mathcal{O}}(\epsilon^{-2})$. Notably, this sample complexity was previously achieved only by single-timescale AC methods. Unlike single-timescale AC, our two-timescale CA and three-timescale C-NCA algorithms ensure stability of iterates due to the difference in step-sizes. 
 
 We also report experimental results on three Gym environments (for the unconstrained algorithms) and three Safety-Gym environments (for the constrained algorithms). Our unconstrained algorithms with modified learning rates and neural network approximators showed the best performance on all three environments. The same algorithms with linear function approximators show the best results in two of the three settings while being competitive on the third, in comparison to several well-established algorithms. In the case of constrained algorithms, a similar trend is observed where C-NCAM shows the best results in two of the three environments. 
Both our algorithms (C-NCAM and C-CAM) are also seen to satisfy the constraints. The modified learning rates can be applied on other multiscale RL algorithms and their sample complexity and empirical performance tested.

\bibliography{aaai2027}

@book{puterman,
  title={Markov decision processes: discrete stochastic dynamic programming},
  author={Puterman, M.L.},
  year={2014},
  publisher={John Wiley \& Sons}
}

@article{bhatnagar2023actorcritic,
  title={Actor-Critic or Critic-Actor? A Tale of Two Time Scales},
  author={Bhatnagar, S. and Borkar, V.S. and Guin, S.},
  journal={IEEE Control Systems Letters},
  year={2023},
  volume={7},
pages = {2671-2676},
publisher={IEEE}
}

@article{konda_actor_critic_type,
author = {Konda, V.R. and Borkar, V.S.},
year = {1999},
month = {12},
pages = {94-123},
title = {Actor-Critic--Type Learning Algorithms for Markov Decision Processes},
volume = {38},
journal = {SIAM J. Control and Optimization},
doi = {10.1137/S036301299731669X}
}

@article{konda_onactorcritic,
author = {Konda, V.R. and Tsitsiklis, J.N.},
title = {OnActor-Critic Algorithms},
journal = {SIAM Journal on Control and Optimization},
volume = {42},
number = {4},
pages = {1143-1166},
year = {2003},
doi = {10.1137/S0363012901385691},

URL = { 
    
        https://doi.org/10.1137/S0363012901385691
    
    
},
eprint = { 
    
        https://doi.org/10.1137/S0363012901385691
    
    

}
}

@inproceedings{kakade_2001,
 author = {Kakade, S.M.},
 booktitle = {Advances in Neural Information Processing Systems},
 editor = {T. Dietterich and S. Becker and Z. Ghahramani},
 pages = {},
 publisher = {MIT Press},
 title = {A Natural Policy Gradient},
 url = {https://proceedings.neurips.cc/paper_files/paper/2001/file/4b86abe48d358ecf194c56c69108433e-Paper.pdf},
 volume = {14},
 year = {2001}
}

@misc{dicastro2009convergent,
      title={A Convergent Online Single Time Scale Actor Critic Algorithm}, 
      author={D. Di Castro and R. Meir},
      year={2009},
      eprint={0909.2934},
      archivePrefix={arXiv},
      primaryClass={cs.LG}
}

@misc{zhang2020provably,
      title={Provably Convergent Two-Timescale Off-Policy Actor-Critic with Function Approximation}, 
      author={Zhang, S. and Liu, B. and Yao, H. and Whiteson, S.},
      year={2020},
      eprint={1911.04384},
      archivePrefix={arXiv},
      primaryClass={cs.LG}
}

@article{BHATNAGAR20092471,
title = {Natural actor–critic algorithms},
journal = {Automatica},
volume = {45},
number = {11},
pages = {2471-2482},
year = {2009},
issn = {0005-1098},
doi = {https://doi.org/10.1016/j.automatica.2009.07.008},
url = {https://www.sciencedirect.com/science/article/pii/S0005109809003549},
author = {Bhatnagar, S. and Sutton, R.S. and Ghavamzadeh, M. and Lee, M.},
}

@InProceedings{pmlr-v247-zeng24a,
  title = 	 {Fast two-time-scale stochastic gradient method with applications in reinforcement learning},
  author =       {Zeng, Sihan and Doan, Thinh},
  booktitle = 	 {Proceedings of Thirty Seventh Conference on Learning Theory},
  pages = 	 {5166--5212},
  year = 	 {2024},
  editor = 	 {Agrawal, Shipra and Roth, Aaron},
  volume = 	 {247},
  series = 	 {Proceedings of Machine Learning Research},
  month = 	 {30 Jun--03 Jul},
  publisher =    {PMLR},
  url = 	 {https://proceedings.mlr.press/v247/zeng24a.html}
}

@inproceedings{NEURIPS2020_5f7695de,
 author = {Ding, Dongsheng and Zhang, Kaiqing and Basar, Tamer and Jovanovic, Mihailo},
 booktitle = {Advances in Neural Information Processing Systems},
 editor = {H. Larochelle and M. Ranzato and R. Hadsell and M.F. Balcan and H. Lin},
 pages = {8378--8390},
 publisher = {Curran Associates, Inc.},
 title = {Natural Policy Gradient Primal-Dual Method for Constrained Markov Decision Processes},
 url = {https://proceedings.neurips.cc/paper_files/paper/2020/file/5f7695debd8cde8db5abcb9f161b49ea-Paper.pdf},
 volume = {33},
 year = {2020}
}

@misc{wu2022finitetimeanalysistimescale,
      title={A Finite Time Analysis of Two Time-Scale Actor Critic Methods}, 
      author={Yue Wu and Weitong Zhang and Pan Xu and Quanquan Gu},
      year={2022},
      eprint={2005.01350},
      archivePrefix={arXiv},
      primaryClass={cs.LG},
      url={https://arxiv.org/abs/2005.01350}, 
}

@misc{chen_zhao,
      title={Finite-time analysis of single-timescale actor-critic}, 
      author={Xuyang Chen and Lin Zhao},
      year={2024},
      eprint={2210.09921},
      archivePrefix={arXiv},
      primaryClass={cs.LG},
      url={https://arxiv.org/abs/2210.09921}, 
}

@misc{suttleetal,
      title={Beyond Exponentially Fast Mixing in Average-Reward Reinforcement Learning via Multi-Level Monte Carlo Actor-Critic}, 
      author={Wesley A. Suttle and Amrit Singh Bedi and Bhrij Patel and Brian M. Sadler and Alec Koppel and Dinesh Manocha},
      year={2023},
      eprint={2301.12083},
      archivePrefix={arXiv},
      primaryClass={cs.LG},
      url={https://arxiv.org/abs/2301.12083}, 
}

@misc{mondal_aggarwal,
      title={Improved Sample Complexity Analysis of Natural Policy Gradient Algorithm with General Parameterization for Infinite Horizon Discounted Reward Markov Decision Processes}, 
      author={Washim Uddin Mondal and Vaneet Aggarwal},
      year={2024},
      eprint={2310.11677},
      archivePrefix={arXiv},
      primaryClass={cs.LG},
      url={https://arxiv.org/abs/2310.11677}, 
}

@misc{cayci2022finitetimeanalysisentropyregularizedneural,
      title={Finite-Time Analysis of Entropy-Regularized Neural Natural Actor-Critic Algorithm}, 
      author={Semih Cayci and Niao He and R. Srikant},
      year={2022},
      eprint={2206.00833},
      archivePrefix={arXiv},
      primaryClass={cs.LG},
      url={https://arxiv.org/abs/2206.00833}, 
}

@inproceedings{xuetal,
author = {Xu, Tengyu and Wang, Zhe and Liang, Yingbin},
title = {Improving sample complexity bounds for (natural) actor-critic algorithms},
year = {2020},
isbn = {9781713829546},
publisher = {Curran Associates Inc.},
address = {Red Hook, NY, USA},
booktitle = {Proceedings of the 34th International Conference on Neural Information Processing Systems},
articleno = {366},
numpages = {12},
location = {Vancouver, BC, Canada},
series = {NIPS '20}
}

@ARTICLE{khodadadian,
  author={Khodadadian, Sajad and Doan, Thinh T. and Romberg, Justin and Maguluri, Siva Theja},
  journal={IEEE Transactions on Automatic Control}, 
  title={Finite-Sample Analysis of Two-Time-Scale Natural Actor–Critic Algorithm}, 
  year={2023},
  volume={68},
  number={6},
  pages={3273-3284},
  doi={10.1109/TAC.2022.3190032}}

@misc{khodadadian2021finitesampleanalysisoffpolicynatural,
      title={Finite-Sample Analysis of Off-Policy Natural Actor-Critic Algorithm}, 
      author={Sajad Khodadadian and Zaiwei Chen and Siva Theja Maguluri},
      year={2021},
      eprint={2102.09318},
      archivePrefix={arXiv},
      primaryClass={cs.LG},
      url={https://arxiv.org/abs/2102.09318}, 
}

@misc{chen2022finitesampleanalysisoffpolicynatural,
      title={Finite-Sample Analysis of Off-Policy Natural Actor-Critic with Linear Function Approximation}, 
      author={Zaiwei Chen and Sajad Khodadadian and Siva Theja Maguluri},
      year={2022},
      eprint={2105.12540},
      archivePrefix={arXiv},
      primaryClass={cs.LG},
      url={https://arxiv.org/abs/2105.12540}, 
}

@article{Borkar,
title = {An actor-critic algorithm for constrained Markov decision processes},
journal = {Systems and Control Letters},
volume = {54},
number = {3},
pages = {207-213},
year = {2005},
issn = {0167-6911},
doi = {https://doi.org/10.1016/j.sysconle.2004.08.007},
url = {https://www.sciencedirect.com/science/article/pii/S0167691104001276},
author = {V.S. Borkar}
}

@article{Bhatnagar2012OnlineActorCritic,
  author    = {Bhatnagar, Shalabh and Lakshmanan, K.},
  title     = {An Online Actor-Critic Algorithm with Function Approximation for Constrained Markov Decision Processes},
  journal   = {Journal of Optimization Theory and Applications},
  volume    = {153},
  number    = {3},
  pages     = {688--708},
  year      = {2012},
  doi       = {10.1007/s10957-012-9989-5},
}

@misc{panda_and_bhatnagar,
      title={Finite-Time Analysis of Three-Timescale Constrained Actor-Critic and Constrained Natural Actor-Critic Algorithms}, 
      author={Prashansa Panda and Shalabh Bhatnagar},
      year={2024},
      eprint={2310.16363},
      archivePrefix={arXiv},
      primaryClass={cs.LG},
      url={https://arxiv.org/abs/2310.16363}, 
}

@article{Panda_Bhatnagar_2025, 
title={Two-Timescale Critic-Actor for Average Reward MDPs with Function Approximation}, 
volume={39}, 
url={https://ojs.aaai.org/index.php/AAAI/article/view/34182}, DOI={10.1609/aaai.v39i19.34182}, abstractNote={Several recent works have focused on carrying out non-asymptotic convergence analyses for AC algorithms. Recently, a two-timescale critic-actor algorithm has been presented for the discounted cost setting in the look-up table case where the timescales of the actor and the critic are reversed and only asymptotic convergence shown. In our work, we present the first two-timescale critic-actor algorithm with function approximation in the long-run average reward setting and present the first finite-time non-asymptotic as well as asymptotic convergence analysis for such a scheme. We obtain optimal learning rates and prove that our algorithm achieves a sample complexity that can be made arbitrarily close to that of single-timescale AC and clearly better than the one obtained for two-timescale AC in a similar setting.. A notable feature of our analysis is that we present the asymptotic convergence analysis of our scheme in addition to the finite-time bounds that we obtain and show the almost sure asymptotic convergence of the (slower) critic recursion to the attractor of an associated differential inclusion with actor parameters corresponding to local maxima of a perturbed average reward objective. We also show the results of numerical experiments on three benchmark settings and observe that our critic-actor algorithm performs the best amongst all algorithms.}, number={19}, 
journal={Proceedings of the AAAI Conference on Artificial Intelligence},
author={Panda, Prashansa and Bhatnagar, Shalabh},
year={2025},
month={Apr.},
pages={19813-19820} }

@article{olshevsky,
  title={A small gain analysis of single timescale actor critic},
  author={Olshevsky, A. and Gharesifard, B.},
  journal={SIAM Journal on Control and Optimization},
  volume={61},
  number={2},
  pages={980--1007},
  year={2023},
  publisher={SIAM}
}

@misc{bhandari2018finitetimeanalysistemporal,
      title={A Finite Time Analysis of Temporal Difference Learning With Linear Function Approximation}, 
      author={Jalaj Bhandari and Daniel Russo and Raghav Singal},
      year={2018},
      eprint={1806.02450},
      archivePrefix={arXiv},
      primaryClass={cs.LG},
      url={https://arxiv.org/abs/1806.02450}, 
}

@book{Bertsekas2013OptimalControlVol2,
  author    = {Bertsekas, Dimitri P.},
  title     = {Optimal Control and Dynamic Programming},
  volume    = {2},
  edition   = {4},
  year      = {2013},
  publisher = {Athena Scientific},
  address   = {Belmont, MA}
}

@article{tsitsiklisroy2,
  title={Average cost temporal-difference learning},
  author={Tsitsiklis, J.N. and Van Roy, B.},
  journal={Automatica},
  volume={35},
  number={11},
  pages={1799--1808},
  year={1999},
  publisher={Elsevier}
}


\onecolumn

\appendix

\section{Appendix}

\subsection{Natural Critic-Actor Algorithm}

We now present the (unconstrained) Natural Critic-Actor with Modified Learning Rates (NCAM) algorithm as we also show empirical performance comparisons with it. At each time step $n$, the algorithm maintains $v_n$ (critic parameter), 
$\theta_n$ (actor parameter),  
$L_n$ (average reward  estimate)  
and $G(n)$ (Fisher information matrix estimate), respectively.

We initialize $G(0) = p I$, where $I$ is the $d \times d$ identity matrix and $p > 0$ is a constant.  
From the update rule, it follows that $G(n)$ for $n \ge 1$ remains positive definite and symmetric, since each update takes the form  
$(1 - a(n)) G(n-1) + a(n) \Psi_{s_n a_n} \Psi_{s_n a_n}^\top$.  
Consequently, $G(n)^{-1}$ is also positive definite and symmetric for all $n \ge 1$.  
Let $\lambda_i > 0$ denote the smallest eigenvalue of $G(i)^{-1}$, and let $\lambda_G = \min_i \lambda_i > 0$. 

We incorporate the following modified learning rates in the NCAM algorithm: $a(n) = \frac{c_a (\ln (n + 1))^{1/2}}{(1+n)^\nu}$, $b(n) = \frac{c_b}{(1+n)^\nu}$,  $d(n) = \frac{c_d (\ln (n+1))^{1/2}}{(1+n)^\nu}$, $n\geq 0$, where $0.5 \leq \nu < \beta \leq 1$ and  $c_a,c_b,c_d >0$. This choice of step-sizes is the same as for the CA algorithm presented earlier. From their form, it can be seen that $\frac{a(n)}{d(n)}= \frac{c_a}{c_d}, \forall n\geq 0$. Thus, step-size sequences $\{a(n)\}$ and $\{d(n)\}$ constitute the same timescale schedule.

\begin{algorithm}[H]
\caption{Two Time-Scale Natural Critic-Actor  algorithm}\label{algo_nca}
\begin{algorithmic}[1]

\STATE \textbf{Input}   $v_{0}$, $\theta_{0}$, $L_{0}$, $G(0)$,  step-size $a(n)$ for actor , $b(n)$ for critic  and $d(n)$ for average reward estimate.
\STATE Draw $s_0$ from some initial distribution
\FOR{$n > 0 $ and $k = 1,2,\ldots,N$}
    \STATE Sample $a_n \sim \pi_{\theta_n}(\cdot |s_n)$, $s_{n+1}\sim p(s_n, \cdot, a_n)$ 
    \STATE Observe  the reward $r(s_n,a_n)$
    \STATE  $L_{n+1} = L_n + d(n)(r(s_n,a_n)  - L_n)$
    \STATE $\delta_{n} = r(s_n,a_n) - L_n + v_{n}^T(f_{s_{n+1}} - f_{s_n})$
    \STATE $v_{n+1} = \Gamma(v_{n} + b(n)\delta_{n}f_{s_n})$
    \STATE $\theta_{n+1} = \theta_{n} + a(n)\delta_{n}G(n)^{-1}\Psi_{s_{n}a_{n}}$
    \STATE $G(n+1) = (1-a(n))G(n) + a(n)\Psi_{s_{n}a_{n}}\Psi_{s_{n}a_{n}}^T$
\ENDFOR
\end{algorithmic}
\end{algorithm}

\subsection{Constrained Critic-Actor Algorithm}

We now present the Constrained Critic-Actor (C-CA) algorithm as we also show empirical performance comparisons with it. At each time step $n$, the algorithm maintains the critic parameter $v_n$, the actor parameter $\theta_n$, the Lagrangian estimate $L_n$, the average constraint cost estimates $U_k(n)$ for $k = 1, 2, \ldots, N$ and the  the vector of Lagrange multiplier estimates $\gamma(n) = (\gamma_1(n), \gamma_2(n), \ldots, \gamma_N(n))^\top$. 

Define also the projection operator $\hat{\Gamma} : \mathbb{R} \rightarrow [0, M]$ as before by
\[
\hat{\Gamma}(y) = \max\left\{0, \min(y, M)\right\}, \quad y \in \mathbb{R},
\]
where $M < \infty$ is a sufficiently large positive constant. This projection operator ensures that the Lagrange multiplier estimates remain non-negative while preventing them from exceeding the prescribed upper bound $M$. 

We incorporate the following modified learning rates for the C-CAM algorithm: $a(n) = \frac{c_a (\ln (n + 1))^{1/2}}{(1+n)^\nu}$, $b(n) = \frac{c_b}{(1+n)^\nu}$, $c(n) = \frac{c_c}{(1+n)^\beta}$, $d(n) = \frac{c_d (\ln (n+1))^{1/2}}{(1+n)^\nu}$, $n\geq 0$, where $0.5 \leq \nu < \beta \leq 1$ and  $c_a,c_b,c_c,c_d >0$. This choice of step-sizes is the same as for the C-NCAM algorithm presented earlier. 

\begin{algorithm}[H]
\caption{Three Time-Scale Critic-Actor  Algorithm for Constrained MDP}\label{algo_cca}
\begin{algorithmic}[1]

\STATE \textbf{Input}   $v_{0}$, $\theta_{0}$, $L_{0}$, $U_k(0)$ for $1 \leq k \leq N$  for $1 \leq k \leq N$, $G(0)$, step-size $a(n)$ for actor , $b(n)$ for critic, $c(n)$ for Lagrange parameter and $d(n)$ for average cost estimate.
\STATE Draw $s_0$ from some initial distribution
\FOR{$n > 0 $ and $k = 1,2,\ldots,N$}
    \STATE Sample $a_n \sim \pi_{\theta_n}(\cdot |s_n)$, $s_{n+1}\sim p(s_n, \cdot, a_n)$ 
    \STATE Observe  the costs $q(n),h_1(n),h_2(n),.....,h_N(n)$
    \STATE  $L_{n+1} = L_n + d(n)(q(n) - \sum_{k=1}^{N}\gamma_k(n)(h_k(n)-\alpha_k) - L_n)$
    \STATE $\delta_{n} = q(n) - \sum_{k=1}^{N}\gamma_k(n)(h_k(n)-\alpha_k) - L_n + v_{n}^T(f_{s_{n+1}} - f_{s_n})$
    \STATE $v_{n+1} = \Gamma(v_{n} + b(n)\delta_{n}f_{s_n})$
    \STATE $\theta_{n+1} = \theta_{n} + a(n)\delta_{n}\Psi_{s_{n}a_{n}}$
    \STATE $U_{k}(n+1) = U_{k}(n) + a(n)(h_{k}(n) - U_{k}(n))$
    \STATE $\gamma_{k}(n+1) = \hat{\Gamma}(\gamma_k(n) + c(n)(U_k(n) - \alpha_k))$
\ENDFOR
\end{algorithmic}
\end{algorithm}

\begin{lemma}[Softmax policy satisfies Assumption~\ref{assum:policy-lipschitz-bounded}]
\label{lem:softmax-policy-properties}
Consider the softmax policy
\[
\pi_\theta(a|s)
=
\frac{\exp\left(\theta^\top \phi(s,a)\right)}
{\sum_{b\in\mathcal A}\exp\left(\theta^\top \phi(s,b)\right)},
\]
where $\phi(s,a)\in\mathbb{R}^d$ is a bounded feature vector satisfying
\[
\|\phi(s,a)\|\le C_\phi,\qquad \forall s\in\mathcal S,\; a\in\mathcal A.
\]
Then Assumption~\ref{assum:policy-lipschitz-bounded} holds with
\[
B=2C_\phi,\qquad
M_m=2C_\phi^2,\qquad
L=2C_\phi.
\]
Specifically, for all $\theta,\theta_1,\theta_2\in\mathbb{R}^d$, $s\in\mathcal S$, and $a\in\mathcal A$,
\begin{align*}
\|\nabla_\theta \log \pi_\theta(a|s)\|
&\le 2C_\phi,\\
\|\nabla_\theta \log \pi_{\theta_1}(a|s)
-\nabla_\theta \log \pi_{\theta_2}(a|s)\|
&\le 2C_\phi^2
\|\theta_1-\theta_2\|,\\
|\pi_{\theta_1}(a|s)-\pi_{\theta_2}(a|s)|
&\le
2C_\phi
\|\theta_1-\theta_2\|.
\end{align*}
\end{lemma}

\begin{proof}
The log-policy is
\[
\log\pi_\theta(a|s)
=
\theta^\top\phi(s,a)
-
\log\sum_{b\in\mathcal A}
\exp(\theta^\top\phi(s,b)).
\]
Differentiating with respect to $\theta$ gives
\[
\nabla_\theta\log\pi_\theta(a|s)
=
\phi(s,a)-\bar{\phi}_\theta(s),
\]
where
\[
\bar{\phi}_\theta(s)
=
\sum_{b\in\mathcal A}
\pi_\theta(b|s)\phi(s,b).
\]
Since $\bar{\phi}_\theta(s)$ is a convex combination of bounded feature vectors,
\[
\|\bar{\phi}_\theta(s)\|
\le
\sum_b\pi_\theta(b|s)\|\phi(s,b)\|
\le C_\phi.
\]
Hence,
\[
\|\nabla_\theta\log\pi_\theta(a|s)\|
\le
\|\phi(s,a)\|
+\|\bar{\phi}_\theta(s)\|
\le
2C_\phi,
\]
proving part (a).

For part (b),
\[
\nabla_\theta\log\pi_\theta(a|s)
=
\phi(s,a)-\bar{\phi}_\theta(s),
\]
so
\[
\|\nabla\log\pi_{\theta_1}(a|s)
-
\nabla\log\pi_{\theta_2}(a|s)\|
=
\|\bar{\phi}_{\theta_1}(s)
-
\bar{\phi}_{\theta_2}(s)\|.
\]
Differentiating $\bar{\phi}_\theta(s)$ yields
\[
\nabla_\theta\bar{\phi}_\theta(s)
=
\sum_b
\pi_\theta(b|s)
(\phi(s,b)-\bar{\phi}_\theta(s))
\phi(s,b)^\top.
\]
Therefore,
\[
\|\nabla_\theta\bar{\phi}_\theta(s)\|
\le
\sum_b
\pi_\theta(b|s)
\|\phi(s,b)-\bar{\phi}_\theta(s)\|
\|\phi(s,b)\|
\le
2C_\phi^2.
\]
Applying the mean value theorem,
\[
\|\bar{\phi}_{\theta_1}(s)-\bar{\phi}_{\theta_2}(s)\|
\le
2C_\phi^2
\|\theta_1-\theta_2\|,
\]
which establishes part (b).

Finally,
\[
\nabla_\theta\pi_\theta(a|s)
=
\pi_\theta(a|s)
\left(\phi(s,a)-\bar{\phi}_\theta(s)\right),
\]
so that
\[
\|\nabla_\theta\pi_\theta(a|s)\|
\le
2C_\phi.
\]
Using the mean value theorem once more,
\[
|\pi_{\theta_1}(a|s)-\pi_{\theta_2}(a|s)|
\le
2C_\phi
\|\theta_1-\theta_2\|,
\]
which proves part (c).
\end{proof}
\begin{remark}
\label{rem1-1}
Note that since $\mathcal{S}$ and $\mathcal{A}$ are finite sets and $\phi(s,a)\in\mathbb{R}^d$ $\forall (s,a)$, we may simply let $C_\phi = \max_{s,a}\|\phi(s,a)\|$. Clearly $C_\phi$ exists, i.e., is finite by the finiteness of the state-action sets.     
\end{remark}

\subsection{Proof Sketch of the Algorithms}

We begin by first presenting the sketches of the proofs of finite-time analysis for the two algorithms namely (a) Critic-Actor with Modified Learning Rates and (b) Constrained Natural Critic-Actor with Modified Learning Rates, respectively. This will then be followed by the detailed analyses of the two algorithms.

\subsection{Proof Sketch for Critic-Actor Analysis with the Modified Learning Rates}

We provide here a proof sketch of the non-asymptotic (finite-time) analysis for all three recursions namely (a) average reward recursion, (b) the actor update and (c) the critic update, respectively.

\subsubsection*{Convergence of the Average Reward Estimate:}
Notations:-
\begin{align}
    \begin{split}
        O_t :&= (s_t , a_t , s_{t+1})\\
        y_t :&= (L_t - L(\theta_t)\\
        r(s_k,a_k) ;&= \EE[r(k) | s_k,a_k]\\
        r_k :&= r(s_k,a_k)\\
        M(\theta_t,v_t) :& = \mathbb{E}_{s_t \sim \mu_{\theta_t},a_t \sim \pi_{\theta_t},s_{t+1} \sim p}[( r(t)- L(\theta_t) + \phi(s_{t+1})^{\top} v_{t} 
        - \phi(s_t)^{\top} v_{t})\nabla \log\pi_{\theta_t}(a_t|s_t)]\\
        N(O_t ,\theta_t,v_t,L_t) :& = ( r(t)- L_t  + \phi(s_{t+1})^{\top} v_{t} - \phi(s_t)^{\top} v_{t})\nabla \log\pi_{\theta_t}(a_t|s_t)\\
        U_w :&=  2B(U_{v} +  \bar{U}_{v})\\
        G :&= 2B(U_r + U_v)
    \end{split}
\end{align}

We have , $\vert V^{\theta}(s) \vert \leq \bar{U}_{v} ,\forall \theta \in \RR^{d},\forall s \in S$ where $\bar{U}_{v}$ is some positive constant.

From the update rule of the reward estimation recursion in  Algorithm \ref{algo_ca}, we have
\begin{align*}
    L_{t+1}-L(\theta_{t+1})=L_t-L(\theta_t)+L(\theta_t)-L(\theta_{t+1})+d(t)(r(t)-L_t).
\end{align*}
We then have
\begin{align*}
    y_{t+1}^2=&\ (y_t+L(\theta_t)-L(\theta_{t+1})+d(t)(r(t)-L_t))^2\\
    \leq &\ y_t^2+2y_t(L(\theta_t)-L(\theta_{t+1}))+2d(t) y_t(r(t)-L_t)+2(L(\theta_t)-L(\theta_{t+1})^2\\
    &\qquad +2d(t)^2(r(t)-L_t)^2\\
    = &\  (1-2d(t))y_t^2+2d(t)y_t(r(t)-L(\theta_t))+2d(t)(L(\theta_t,\gamma(t))-L(\theta_{t+1})\\
    &\qquad+2(L(\theta_t)-L(\theta_{t+1}))^2
    +2d(t)^2(r(t)-L_t)^2.
\end{align*}

Taking expectations, rearranging  and summing from $\tau_t$ to $t$ we obtain,

\begin{align}
    \sum\limits_{k=\tau_t}^{t} \mathbb{E}[y_k^2]\leq &\  \underbrace{\sum\limits_{t=\tau_t}^t\frac{1}{2d(k)}\mathbb{E}(y_k^2-y^2_{k+1})}_{I_1}+\underbrace{\sum\limits_{k=\tau_t}^{t}\mathbb{E}[y_k(r(k)-L(\theta_k))]}_{I_2} \notag \\
    & +\underbrace{\sum\limits_{k=\tau_t}^{t}\frac{1}{d(k)}\mathbb{E}[y_k(L(\theta_k)-L(\theta_{k+1})]}_{I_3} \notag \\
    &  +\underbrace{\sum\limits_{k=\tau_t}^{t}\frac{1}{d(k)}\mathbb{E}[(L(\theta_k)-L(\theta_{k+1}))^2]}_{I_4}+\underbrace{\sum\limits_{k=\tau_t}^{t}d(k)\mathbb{E}[(r(k)-L_k)^2]}_{I_5}.
\end{align}

After analysing the  terms $I_1 - I_5$ we have,
\begin{align*}
    \sum\limits_{k=\tau_t}^{t} \mathbb{E}[y_k^2]&\leq \frac{2}{c_d \cdot \ln ^{0.5} (t + 1)}U_r^2 (1+t)^{\nu} + \mathcal{O}(\log^{2.5} t \cdot t^{1-\nu}) + \frac{2L_{J'}U_rG^2c_a^2 \ln^{0.5} (t+1)}{c_d}(1+t-\tau_t)^{1-\nu}\\
    &\qquad+ G\frac{c_a }{c_d}(\sum\limits_{k=\tau_t}^{t}\mathbb{E}y_t^2)^{\frac{1}{2}}(\sum\limits_{k=\tau_t}^{t}\mathbb{E}\Vert M(\theta_k,v_k)\Vert^2)^\frac{1}{2}\\
    &\qquad + \mathcal{O}\bigg( \bigg(\sum\limits_{k=\tau_t}^{t}\mathbb{E}y_t^2\bigg)^{\frac{1}{2}}\bigg(\sum\limits_{k=\tau_t}^{t}b^2\bigg)^\frac{1}{2}\bigg).
\end{align*}

\begin{align*}
    \Rightarrow &\sum\limits_{k=\tau_t}^{t} \mathbb{E}[y_k^2]
    \leq  \frac{2}{c_d \cdot \ln ^{0.5} (t + 1)}U_r^2 (1+t)^{\nu} + \mathcal{O}(\log^{2.5} t \cdot t^{1-\nu}) + \frac{2L_{J'}U_rG^2c_a^2 \ln^{0.5} (t+1)}{c_d}(1+t-\tau_t)^{1-\nu}\\
    &\qquad+ G\frac{c_a }{c_d}(\sum\limits_{k=\tau_t}^{t}\mathbb{E}y_t^2)^{\frac{1}{2}}(\sum\limits_{k=\tau_t}^{t}\mathbb{E}\Vert M(\theta_k,v_k)\Vert^2)^\frac{1}{2} + 
   \mathcal{O}\bigg( \bigg(\sum\limits_{k=\tau_t}^{t}\mathbb{E}y_t^2\bigg)^{\frac{1}{2}}\bigg(\sum\limits_{k=\tau_t}^{t}b^2\bigg)^\frac{1}{2}\bigg).
\end{align*}

Hence, we obtain:

\begin{align*}
    \sum\limits_{k=\tau_t}^{t} \mathbb{E}[y_k^2] &\leq \mathcal{O}(\log^{-0.5} t \cdot  t^{\nu}) + \mathcal{O}(\log^{2.5} t \cdot t^{1-\nu}) + G\frac{c_a }{c_d}(\sum\limits_{k=\tau_t}^{t}\mathbb{E}y_t^2)^{\frac{1}{2}}(\sum\limits_{k=\tau_t}^{t}\mathbb{E}\Vert M(\theta_k,v_k)\Vert^2)^\frac{1}{2}\\
    &\qquad + \mathcal{O}\bigg( \bigg(\sum\limits_{k=\tau_t}^{t}\mathbb{E}y_t^2\bigg)^{\frac{1}{2}}\bigg(\sum\limits_{k=\tau_t}^{t}b^2\bigg)^\frac{1}{2}\bigg).
\end{align*}

After applying the squaring technique (see page 23 of \citep{wu2022finitetimeanalysistimescale}), we have,

\begin{align}\label{ineq1}
    \sum\limits_{k=\tau_t}^{t} \mathbb{E}[y_k^2] &\leq \mathcal{O}(\log^{-0.5} t \cdot  t^{\nu}) + \mathcal{O}(\log^{2.5} t \cdot t^{1-\nu})  + 
     2G^2\frac{c_a^2}{c_d^2}\sum\limits_{k=\tau_t}^{t}\mathbb{E}\Vert M(\theta_k,v_k)\Vert^2 + \mathcal{O}\bigg(\sum_{k = \tau_t}^{t}b^2 \bigg) 
\end{align}
In the above equation, the $\mathcal{O}(b^2)$ term represents the mixing-time error arising from the ergodicity of the Markov chain. So after ignoring this mixing-time error we have,

\begin{align*}
    \frac{1}{t - \tau_t + 1}\sum\limits_{k=\tau_t}^{t} \mathbb{E}[y_k^2] &\leq \mathcal{O}(\log^{-0.5} t \cdot  t^{\nu - 1}) + \mathcal{O}(\log^{2.5} t \cdot t^{-\nu})  
    + 2G^2\frac{c_a^2}{c_d^2} \frac{1}{t - \tau_t + 1}\sum\limits_{k=\tau_t}^{t}\mathbb{E}\Vert M(\theta_k,v_k)\Vert^2 .
\end{align*}

\subsubsection*{Convergence of Actor:}

\begin{align}
    \begin{split}
        O_t :&= (s_t , a_t , s_{t+1})\\
        h(O_t,\theta_t,L_t,v_t) :&= ( r(t)- L_t  + \phi(s_{t+1})^{\top} v_{t} - \phi(s_t)^{\top} v_{t})\nabla \log\pi_{\theta_t}(a_t|s_t)\\
        I(O_t,L_t,\theta_t,v_t) :&= \langle \nabla L(\theta_t,\gamma(t)) ,h(O_t,\theta_t,L_t,v_t)\\
        &\qquad - \mathbb{E}_{s_t \sim \mu_{\theta_t},a_t \sim \pi_{\theta_t},s_{t+1} \sim p}[h(O_t,\theta_t,L_t,v_t)] \rangle  \\
        \bar{h}(O_t,\theta_t,v_t) :&= ( r(t)- L(\theta_t)  + \phi(s_{t+1})^{\top} v_{t}\\
        &\qquad - \phi(s_t)^{\top} v_{t})\nabla \log\pi_{\theta_t}(a_t|s_t)\\
        M(\theta_t,v_t) :& = \mathbb{E}_{s_t \sim \mu_{\theta_t},a_t \sim \pi_{\theta_t},s_{t+1} \sim p}[\bar{h}(O_t,\theta_t,v_t)]
    \end{split}
\end{align}

After applying Lemma 1 of \cite{Panda_Bhatnagar_2025} to the update rule of the actor, we have,

\begin{align*}
    L(\theta_{t+1}) &\geq L(\theta_t) + a(t) \langle \nabla L(\theta_t) ,\delta_{t}\nabla \log\pi_{\theta_t}(a_t|s_t) \rangle- M_{L}a(t)^2\Vert\delta_{t}\nabla \log\pi_{\theta_t}(a_t|s_t)\Vert^2.
\end{align*}

For the term $\langle \nabla L(\theta_t) ,\delta_{t}\nabla \log\pi_{\theta_t}(a_t|s_t) \rangle $, we have,
\begin{align*}
    &\langle \nabla L(\theta_t) ,\delta_{t}\nabla \log\pi_{\theta_t}(a_t|s_t) \rangle \\
    &= \langle \nabla L(\theta_t) ,( r(t) - L_t + \phi(s_{t+1})^{\top} v_{t} - \phi(s_t)^{\top} v_{t})\nabla \log\pi_{\theta_t}(a_t|s_t) \rangle\\
    & = I(O_t,\theta_t,L_t,v_t)+ \langle \nabla L(\theta_t) , E_{s_t \sim \mu_{\theta_t},a_t \sim \pi_{\theta_t},s_{t+1} \sim p}[h(O_t,\theta_t,L_t,v_t)] \rangle.
\end{align*}

Hence,
\begin{align}\label{inequality_L(theta)}
    & L(\theta_{t+1})\\
    &\geq L(\theta_t) + a(t)I(O_t,\theta_t,L_t,v_t) + a(t)\langle \nabla L(\theta_t) , M(\theta _t,v_t) \rangle +a(t)\langle \nabla L(\theta_t) ,  \mathbb{E}_{\theta_t}[(L(\theta_t) - L_t)\nabla \log\pi_{\theta_t}(a_t|s_t)] \rangle \notag\\
&\qquad -M_{L}a(t)^2\Vert\delta_{t}\nabla \log\pi_{\theta_t}(a_t|s_t)\Vert^2\notag\\
& = L(\theta_t) + a(t) I(O_t,\theta_t,L_t,v_t) + a(t)\langle M(\theta_t,v_t) , M(\theta_t,v_t)\rangle \notag \\
&\qquad+a(t)\langle   \mathbb{E}_{\theta_t}[(V^{\theta_t}(s_{t+1}) -  \phi(s_{t+1})^Tv_t - V^{\theta_t}(s_t) + \phi(s_{t})^Tv_t)\nabla \log\pi_{\theta_t}(a_t|s_t)] ,   M(\theta_t,v_t) \rangle\notag\\
&\qquad + a(t)\langle \nabla L(\theta_t),  \mathbb{E}_{\theta_t}[(L(\theta_t) - L_t)\nabla \log\pi_{\theta_t}(a_t|s_t)] \rangle - M_{L}a(t)^2\Vert\delta_{t}\nabla \log\pi_{\theta_t}(a_t|s_t)\Vert^2.\notag 
\end{align}

where $\EE_{\theta_t}[.]$ denotes $\EE_{s_t \sim \mu_{\theta_t} , a_t \sim \pi_{\theta_t}(.|s_t), s_{t+1} \sim p(s_t,.,a_t)}[.]$.

After carrying out some analysis and rearranging the terms we have,

\begin{align*}
    &\sum\limits_{k=\tau_t}^{t}\mathbb{E}\Vert M(\theta_k,v_k) \Vert^2 \\
&\leq \underbrace{\sum\limits_{k=\tau_t}^{t}\frac{1}{a(k)}\mathbb{E}[(L(\theta_{k+1}) -L(\theta_{k})  )]}_{I_1} - \underbrace{\sum\limits_{k=\tau_t}^{t}\mathbb{E}[I(O_k,\theta_k,L_k,v_k)]}_{I_2} + \underbrace{\sum\limits_{k=\tau_t}^{t} \EE\langle \EE_{\theta_k}[H(O_k, \theta_k,v_k)], M(\theta_k,v_k)\rangle]}_{I_3} \\
&\qquad \underbrace{-\sum\limits_{k=\tau_t}^{t}\mathbb{E}[\langle \nabla L(\theta_k) , \mathbb{E}_{\theta_k}[(L(\theta_k) - L_k)\nabla \log\pi_{\theta_k}(a_k|s_k)] \rangle]}_{I_4} \\
&\qquad +\underbrace{M_{L}\sum\limits_{k=\tau_t}^{t}a(k)\mathbb{E}[\Vert\delta_{k}\nabla \log\pi_{\theta_k}(a_k|s_k)\Vert^2]}_{I_5}.
\end{align*}

After analysing  the terms $I_1 - I_5$ we have,

\begin{align*}
    &\sum\limits_{k=\tau_t}^{t}\mathbb{E}\Vert M(\theta_k,v_k) \Vert^2\\
    &\leq \mathcal{O}((\log t)^{-0.5}\cdot t^{\nu}) +  \mathcal{O}(\log ^{2.5} t \cdot t^{1-\nu}) + \mathcal{O}\bigg( \sum\limits_{l=\tau_t}^{t} bk\bigg)\\
    &\qquad + B\sqrt{\sum\limits_{k=\tau_t}^{t}\mathbb{E}\Vert M(\theta_k,v_k)\Vert^2}\sqrt{\sum\limits_{k=\tau_t}^{t}\mathbb{E}\vert L_k - L(\theta_k) \vert^2} +  \sqrt{\sum\limits_{k=\tau_t}^{t}\EE \Vert \EE_{\theta_k}[H(O_k, \theta_k,v_k)] \Vert^2}\sqrt{\sum\limits_{k=\tau_t}^{t}\EE \Vert M(\theta_k,v_k)\Vert^2}
\end{align*}

After applying the squaring technique we have,

\begin{align*}
    \sum\limits_{k=\tau_t}^{t}\mathbb{E}\Vert M(\theta_k,v_k) \Vert^2& \leq \mathcal{O}((\log t)^{-0.5}\cdot t^{\nu}) +  \mathcal{O}(\log ^{2.5} t \cdot t^{1-\nu}) + 2B^2\sum\limits_{k=\tau_t}^{t}\mathbb{E}\vert L_k - L(\theta_k) \vert^2\\
    &\qquad + 2 \sum\limits_{k=\tau_t}^{t}\EE \Vert \EE_{\theta_k}[H(O_k, \theta_k,v_k)] \Vert^2 + \mathcal{O}\bigg( \sum\limits_{l=\tau_t}^{t} bk\bigg)\\
    &\leq \mathcal{O}((\log t)^{-0.5}\cdot t^{\nu}) +  \mathcal{O}(\log ^{2.5} t \cdot t^{1-\nu}) +\mathcal{O}\bigg( \sum\limits_{l=\tau_t}^{t} b^2\bigg)  \\
    &\qquad + 4B^2G^{2}\frac{c_a^2}{c_d^2}\sum\limits_{k=\tau_t}^{t}\mathbb{E}\Vert M(\theta_k,v_k)\Vert^2
\end{align*}

Now if we select the values for $c_a$ and $c_d$ such that $4B^2G^{2}\frac{c_a^2}{c_d^2} < 1$, we shall obtain,
 \begin{align*}
     \sum\limits_{k=\tau_t}^{t}\mathbb{E}\Vert M(\theta_k,v_k) \Vert^2 &= \mathcal{O}((\log t)^{-0.5}\cdot t^{\nu}) +  \mathcal{O}(\log ^{2.5} t \cdot t^{1-\nu}) + \mathcal{O}\bigg( \sum\limits_{l=\tau_t}^{t} b^2\bigg) .
 \end{align*}
 Dividing by $(1 + t - \tau_t)$  and assuming $t \geq 2\tau_t + 1$, we have,
 \begin{align*}
     \frac{1}{(1 + t - \tau_t)}\sum\limits_{k=\tau_t}^{t}\mathbb{E}\Vert M(\theta_k,v_k)\Vert^2 &= \mathcal{O}((\log t)^{-0.5} \cdot t^{\nu - 1}) + \mathcal{O}(\log^{2.5} t \cdot t^{-\nu}) + \mathcal{O}(  b^2) .
 \end{align*}
 In the above equation, the $\mathcal{O}(b^2)$ term represents the mixing-time error arising from the ergodicity of the Markov chain. So after ignoring this mixing-time error we have,
  \begin{align*}
     \frac{1}{(1 + t - \tau_t)}\sum\limits_{k=\tau_t}^{t}\mathbb{E}\Vert M(\theta_k,v_k)\Vert^2 &= \mathcal{O}((\log t)^{-0.5} \cdot t^{\nu - 1}) + \mathcal{O}(\log^{2.5} t \cdot t^{-\nu})  .
 \end{align*}

\subsubsection{Convergence of Critic:}

Notations:
\begin{align}
    \begin{split}
        z_t &:= v_t - v^{*}(\theta_t).
    \end{split}
\end{align}

From the critic update rule, we have,
\begin{align*}
    \Vert z_{t+1} \Vert^2 &= \Vert  v_{t+1} - v^{*}(\theta_{t+1})\Vert^2\\
    &=\Vert \Gamma(v_t +b(t)\delta_t\phi(s_t)) - v^{*}(\theta_{t+1})  \Vert^2\\
    &\leq \Vert v_t + b(t)\delta_t\phi(s_t) - v^{*}(\theta_{t+1}) \Vert^2\\
    &= \Vert z_t + b(t)\delta_t\phi(s_t)  + v^{*}(\theta_t)- v^{*}(\theta_{t+1}) \Vert^2\\
    &\leq  \Vert z_t \Vert^2 + 2b(t)\langle z_t ,  \delta_t\phi(s_t) \rangle + 2\langle z_t , v^{*}(\theta_t) - v^{*}(\theta_{t+1}) \rangle + 2b(t)^2\delta_t^2\Vert \phi(s_t)\Vert^2 \\
    &\qquad + 2\Vert v^{*}(\theta_t) - v^{*}(\theta_{t+1})  \Vert^2\\
    &=\Vert z_t \Vert^2 + 2b(t)\langle z_t ,  \delta_t\phi(s_t) -  \mathbb{E}_{\theta_t}[\delta_t\phi(s_t)]\rangle +
    2b(t)\langle z_t ,    \mathbb{E}_{\theta_t}[\delta_t\phi(s_t)]\rangle\\
    &\qquad +
    2\langle z_t , v^{*}(\theta_t) - v^{*}(\theta_{t+1}) \rangle
    + 2b(t)^2\delta_t^2\Vert \phi(s_t)\Vert^2 + 2\Vert v^{*}(\theta_t) - v^{*}(\theta_{t+1})  \Vert^2\\
    &\leq \Vert z_t \Vert^2 + 2b(t)\langle z_t ,  \delta_t\phi(s_t) -  \mathbb{E}_{\theta_t}[\delta_t\phi(s_t)]\rangle -
    2b(t)\hat{\lambda_e}\Vert z_t\Vert^2 +
    2\langle z_t , v^{*}(\theta_t) - v^{*}(\theta_{t+1}) \rangle\\
    &\qquad+ 2b(t)^2\delta_t^2\Vert \phi(s_t)\Vert^2 + 2\Vert v^{*}(\theta_t) - v^{*}(\theta_{t+1})  \Vert^2.
\end{align*}

The final inequality follows from Assumption 2.

After rearranging the terms and taking summation of terms from indices $\tau_t$ to $t$ we have,
\begin{align}
    \hat{\lambda_e}\sum\limits_{k=\tau_t}^{t}\mathbb{E}\Vert z_k \Vert^2 &\leq \underbrace{\sum\limits_{k=\tau_t}^{t}\frac{1}{2b(k)}(\mathbb{E}\Vert z_k \Vert^2 - \mathbb{E}\Vert z_{k+1} \Vert^2)}_{I_1} + \underbrace{\sum\limits_{k=\tau_t}^{t}\mathbb{E}[\langle z_k ,  \delta_k\phi(s_k) - \mathbb{E}_{\theta_k}[\delta_k\phi(s_k)]\rangle]}_{I_2}\notag\\
    &\qquad + \underbrace{\sum\limits_{k=\tau_t}^{t}\frac{1}{b(k)}\mathbb{E}[\langle z_k , v^{*}(\theta_k) - v^{*}(\theta_{k+1}) 
    + (\nabla v_k^*)^T(\theta_{k+1} - \theta_k)\rangle]}_{I_3}\notag\\
    &\qquad+ \underbrace{\sum\limits_{k=\tau_t}^{t}\frac{1}{b(k)}\mathbb{E}[\langle z_k , (\nabla v_k^*)^T(\theta_{k} - \theta_{k+1}) \rangle]}_{I_4} +\underbrace{\sum\limits_{k=\tau_t}^{t}b(k)\mathbb{E}[\delta_k^2\Vert \phi(s_k)\Vert^2]}_{I_5}\notag \\
    &\qquad + \underbrace{\sum\limits_{k=\tau_t}^{t}\frac{1}{b(k)}\mathbb{E}\Vert v^{*}(\theta_k) - v^{*}(\theta_{k+1})  \Vert^2}_{I_6}.
\end{align}

After analysing the  terms $I_1 - I_6$,  we have,

\begin{align*}
     \hat{\lambda_e}\sum\limits_{k=\tau_t}^{t}\mathbb{E}\Vert z_k \Vert^2 &\leq \mathcal{O}( t^{\nu}) +  \mathcal{O}(\log^{2.5} t \cdot t^{1-\nu})  \\
     &\qquad +  L_{*}\frac{c_a}{c_b}\log^{0.5} t\sqrt{\sum\limits_{k=\tau_t}^{t}\mathbb{E}\Vert z_k\Vert^2}\sqrt{\sum\limits_{k=\tau_t}^{t}\mathbb{E}[\Vert M(\theta_k,v_k)\Vert^2]}\\
    &\qquad + L_{*}B\frac{c_a}{c_b}\log^{0.5} t\sqrt{\sum\limits_{k=\tau_t}^{t}\mathbb{E}\Vert z_k\Vert^2}\sqrt{\sum\limits_{k=\tau_t}^{t}\mathbb{E}[(L(\theta_k)- L_k)^2]}
\end{align*}

After applying the square technique we have,

\begin{align*}
    \sum\limits_{k=\tau_t}^{t}\mathbb{E}\Vert z_k \Vert^2  &= \mathcal{O}( t^{\nu}) +  \mathcal{O}(\log^{2.5} t \cdot t^{1-\nu}) \\
    &\qquad + \mathcal{O}(\log^{0.5} t \cdot \sum\limits_{k=\tau_t}^{t}\mathbb{E}[\Vert M(\theta_k,v_k)\Vert^2]) + \mathcal{O}(\log^{0.5} t \cdot \sum\limits_{k=\tau_t}^{t}\mathbb{E}[(L(\theta_k)- L_k)^2])\\
    &= \mathcal{O}( t^{\nu}) +  \mathcal{O}(\log^{3} t \cdot t^{1-\nu}).
\end{align*}

Assuming $t \geq 2\tau_{t} - 1$, we have,

\begin{align*}
    \frac{1}{1+t-\tau_t}\sum\limits_{k=\tau_t}^{t}\mathbb{E}\Vert z_k \Vert^2 = \mathcal{O}( t^{\nu - 1}) +  \mathcal{O}(\log^{3} t \cdot t^{-\nu})  
\end{align*}

Optimising over the values of $\nu$, we have $\nu = 0.5$ . Hence we have  the following :-
\begin{align*}
    \frac{1}{1+t-\tau_t}\sum_{k=\tau_t}^{t}\mathbb{E}\Vert z_k \Vert^2 &= \mathcal{O}(\log^3 t \cdot t^{ - 0.5})
\end{align*}

Therefore in order for the mean squared error of the critic to be upper bounded by $\epsilon$, namely,

\begin{align*}
     \frac{1}{1+t-\tau_t}\sum_{k=\tau_t}^{t}\mathbb{E}\Vert z_k \Vert^2  =  \mathcal{O}(\log^3 T \cdot T^{- 0.5}) \leq \epsilon,
\end{align*}
we need to set $T = \tilde{\mathcal{O}}(\epsilon^{-2})$.\\

\subsection{Proof Sketch for Constrained Natural Critic-Actor Analysis with the Modified Learning Rates}

We provide here a proof sketch of the non-asymptotic (finite-time) analysis for the three recursions namely (a) average reward recursion, (b) the actor update and (c) the critic update, respectively.

\subsubsection*{Convergence of the Average Reward Estimate:}

Notations:-
\begin{align}
    \begin{split}
        O_t :&= (s_t , a_t , s_{t+1})\\
        y_t :&= (L_t - L(\theta_t,\gamma(t)))\\
        r_t :&= r(t) - \sum_{k=1}^{N}\gamma_k(t)(h_{k}(t)-\alpha_{k})\\
        M(\theta_t,v_t,\gamma(t)) :& = \mathbb{E}_{s_t \sim \mu_{\theta_t},a_t \sim \pi_{\theta_t},s_{t+1} \sim p}[( r_t- L(\theta_t,\gamma(t)) + \phi(s_{t+1})^{\top} v_{t} \\
        &\qquad- \phi(s_t)^{\top} v_{t})\nabla \log\pi_{\theta_t}(a_t|s_t)]\\
        N(O_t ,\theta_t,v_t,L_t,\gamma(t)) :& = ( r_t- L_t  + \phi(s_{t+1})^{\top} v_{t} - \phi(s_t)^{\top} v_{t})\nabla \log\pi_{\theta_t}(a_t|s_t)\\
        U_w :&=  2B(U_{v} +  \bar{U}_{v})\\
        G :&= 2B(U_r + U_v)
    \end{split}
\end{align}

We have , $\vert V^{\theta,\gamma}(s) \vert \leq \bar{U}_{v} ,\forall \theta \in \RR^{d},\forall s \in S$ and $\forall$ $\gamma \in \RR^N$, with $0 \leq \gamma_{i} \leq M$, where $i \in \{1,2,...,N\}$.

From the update rule of the reward estimation recursion in  Algorithm \ref{algo_ncca}, we have
\begin{align*}
    L_{t+1}-L(\theta_{t+1},\gamma(t+1))=L_t-L(\theta_t,\gamma(t))+L(\theta_t,\gamma(t))-L(\theta_{t+1},\gamma(t+1))+d(t)(r_t-L_t).
\end{align*}
We then have
\begin{align*}
    y_{t+1}^2=&\ (y_t+L(\theta_t,\gamma(t))-L(\theta_{t+1},\gamma(t+1))+d(t)(r_t-L_t))^2\\
    \leq &\ y_t^2+2y_t(L(\theta_t,\gamma(t))-L(\theta_{t+1},\gamma(t+1)))+2d(t) y_t(r_t-L_t)+2(L(\theta_t,\gamma(t))-L(\theta_{t+1},\gamma(t+1))^2\\
    &\qquad +2d(t)^2(r_t-L_t)^2\\
    = &\  (1-2d(t))y_t^2+2d(t)y_t(r_t-L(\theta_t))+2d(t)(L(\theta_t,\gamma(t))-L(\theta_{t+1},\gamma(t+1))\\
    &\qquad+2(L(\theta_t,\gamma(t))-L(\theta_{t+1},\gamma(t+1))^2
    +2d(t)^2(r_t-L_t)^2.
\end{align*}

Taking expectations, rearranging  and summing from $\tau_t$ to $t$ we obtain,

\begin{align}
    \sum\limits_{k=\tau_t}^{t} \mathbb{E}[y_k^2]\leq &\  \underbrace{\sum\limits_{t=\tau_t}^t\frac{1}{2d(k)}\mathbb{E}(y_k^2-y^2_{k+1})}_{I_1}+\underbrace{\sum\limits_{k=\tau_t}^{t}\mathbb{E}[y_k(r_k-L(\theta_k,\gamma(k)))]}_{I_2} \notag \\
    & +\underbrace{\sum\limits_{k=\tau_t}^{t}\frac{1}{d(k)}\mathbb{E}[y_k(L(\theta_k,\gamma(k))-L(\theta_{k+1},\gamma(k+1))]}_{I_3} \notag \\
    &  +\underbrace{\sum\limits_{k=\tau_t}^{t}\frac{1}{d(k)}\mathbb{E}[(L(\theta_k,\gamma(k))-L(\theta_{k+1},\gamma(k+1)))^2]}_{I_4}+\underbrace{\sum\limits_{k=\tau_t}^{t}d(k)\mathbb{E}[(r_k-L_k)^2]}_{I_5}.
\end{align}

After analysing terms $I_1 - I_5$ we have,

\begin{align*}
    \sum\limits_{k=\tau_t}^{t} \mathbb{E}[y_k^2]&\leq \frac{2}{c_d \cdot \ln ^{0.5} (t + 1)}U_r^2 (1+t)^{\nu} + \mathcal{O}(\log^{2.5} t \cdot t^{1-\nu}) + \frac{2L_{J'}U_rG^2c_a^2 \ln^{0.5} (t+1)}{c_d}(1+t-\tau_t)^{1-\nu}\\
    &\qquad+ G\frac{c_a }{c_d}(\sum\limits_{k=\tau_t}^{t}\mathbb{E}y_t^2)^{\frac{1}{2}}(\sum\limits_{k=\tau_t}^{t}\mathbb{E}\Vert M(\theta_k,v_k,\gamma(k))\Vert^2)^\frac{1}{2} + \mathcal{O}(t^{1 + \nu - \beta})\\
    &\qquad + \mathcal{O}\bigg( \bigg(\sum\limits_{k=\tau_t}^{t}\mathbb{E}y_t^2\bigg)^{\frac{1}{2}}\bigg(\sum\limits_{k=\tau_t}^{t}b^2\bigg)^\frac{1}{2}\bigg)\\
    &\qquad +\mathcal{O}(\ln^{0.5} t \cdot t^{1-\nu})
\end{align*}

Hence, we obtain:

\begin{align*}
    \sum\limits_{k=\tau_t}^{t} \mathbb{E}[y_k^2] &\leq \mathcal{O}(\log^{-0.5} t \cdot  t^{\nu}) + \mathcal{O}(\log^{2.5} t \cdot t^{1-\nu}) + G\frac{c_a }{c_d}(\sum\limits_{k=\tau_t}^{t}\mathbb{E}y_t^2)^{\frac{1}{2}}(\sum\limits_{k=\tau_t}^{t}\mathbb{E}\Vert M(\theta_k,v_k,\gamma(k))\Vert^2)^\frac{1}{2}\\
    &\qquad + \mathcal{O}(t^{1 + \nu - \beta}) + \mathcal{O}\bigg( \bigg(\sum\limits_{k=\tau_t}^{t}\mathbb{E}y_t^2\bigg)^{\frac{1}{2}}\bigg(\sum\limits_{k=\tau_t}^{t}b^2\bigg)^\frac{1}{2}\bigg).
\end{align*}

After applying the squaring technique (see page 23 of \citep{wu2022finitetimeanalysistimescale}), we have,

\begin{align}\label{ineq1_CCA}
    \sum\limits_{k=\tau_t}^{t} \mathbb{E}[y_k^2] &\leq \mathcal{O}(\log^{-0.5} t \cdot  t^{\nu}) + \mathcal{O}(\log^{2.5} t \cdot t^{1-\nu})  + \mathcal{O}(t^{1 + \nu - \beta}) \notag\\
    &\qquad + 2G^{2}\frac{c_a^2}{c_d^2}\sum\limits_{k=\tau_t}^{t}\mathbb{E}\Vert M(\theta_k,v_k,\gamma(k))\Vert^2 + \mathcal{O}\bigg(\sum_{k = \tau_t}^{t}b^2 \bigg) .
\end{align}

In the above equation, the $\mathcal{O}(b^2)$ term represents the mixing-time error arising from the ergodicity of the Markov chain. So after ignoring this mixing-time error we have,

\begin{align*}
    \frac{1}{t - \tau_t + 1}\sum\limits_{k=\tau_t}^{t} \mathbb{E}[y_k^2] &\leq \mathcal{O}(\log^{-0.5} t \cdot  t^{\nu - 1}) + \mathcal{O}(\log^{2.5} t \cdot t^{-\nu})   
    + 2G^2\frac{c_a^2}{c_d^2} \frac{1}{t - \tau_t + 1}\sum\limits_{k=\tau_t}^{t}\mathbb{E}\Vert M(\theta_k,v_k)\Vert^2 \\
    &\qquad + \mathcal{O}(t^{\nu - \beta})
\end{align*}

\subsubsection*{Convergence of Actor:}

Notations used here:

\begin{align}
    \begin{split}
        O_t :&= (s_t , a_t , s_{t+1})\\
        r_t :&= r(t) - \sum_{k=1}^{N}\gamma_k(t)(h_{k}(t)-\alpha_{k})\\
        h(O_t,\theta_t,L_t,v_t,\gamma(t),G(t)) :&= ( r_t- L_t  + \phi(s_{t+1})^{\top} v_{t} - \phi(s_t)^{\top} v_{t})G(t)^{-1}\nabla \log\pi_{\theta_t}(a_t|s_t)\\
        I(O_t,L_t,\theta_t,v_t,\gamma(t),G(t)) :&= \langle \nabla L(\theta_t,\gamma(t)) ,h(O_t,\theta_t,L_t,v_t,\gamma(t),G(t))\\
        &\qquad - \mathbb{E}_{s_t \sim \mu_{\theta_t},a_t \sim \pi_{\theta_t},s_{t+1} \sim p}[h(O_t,\theta_t,L_t,v_t,\gamma(t),G(t))] \rangle  \\
        \bar{h}(O_t,\theta_t,v_t,\gamma(t),G(t)) :&= ( r_t- L(\theta_t,\gamma(t))  + \phi(s_{t+1})^{\top} v_{t}\\
        &\qquad - \phi(s_t)^{\top} v_{t})G(t)^{-1}\nabla \log\pi_{\theta_t}(a_t|s_t)\\
        \hat{h}(O_t,\theta_t,v_t,\gamma(t)) :&= ( r_t- L(\theta_t,\gamma(t))  + \phi(s_{t+1})^{\top} v_{t}\\
        &\qquad - \phi(s_t)^{\top} v_{t})\nabla \log\pi_{\theta_t}(a_t|s_t)\\
        M(\theta_t,v_t,\gamma(t),G(t)) :& = \mathbb{E}_{s_t \sim \mu_{\theta_t},a_t \sim \pi_{\theta_t},s_{t+1} \sim p}[\bar{h}(O_t,\theta_t,v_t,\gamma(t),G(t))]\\
        \bar{M}(\theta_t,v_t,\gamma(t)) :& = \mathbb{E}_{s_t \sim \mu_{\theta_t},a_t \sim \pi_{\theta_t},s_{t+1} \sim p}[\hat{h}(O_t,\theta_t,v_t,\gamma(t))].
    \end{split}
\end{align}

After applying Lemma 1 of \cite{panda_and_bhatnagar} to the update rule of the actor, we have,

\begin{align*}
    L(\theta_{t+1},\gamma(t)) &\geq L(\theta_t,\gamma(t)) + a(t) \langle \nabla L(\theta_t,\gamma(t)) ,\delta_{t}G(t)^{-1}\nabla \log\pi_{\theta_t}(a_t|s_t) \rangle\\ 
&\qquad - M_{L}a(t)^2\Vert\delta_{t}G(t)^{-1}\nabla \log\pi_{\theta_t}(a_t|s_t)\Vert^2.
\end{align*}

For the term $\langle \nabla L(\theta_t,\gamma(t)) ,\delta_{t}G(t)^{-1}\nabla \log\pi_{\theta_t}(a_t|s_t) \rangle $, we have,
\begin{align*}
    &\langle \nabla L(\theta_t,\gamma(t)) ,\delta_{t}G(t)^{-1}\nabla \log\pi_{\theta_t}(a_t|s_t) \rangle \\
    &= \langle \nabla L(\theta_t,\gamma(t)) ,( r_t - L_t + \phi(s_{t+1})^{\top} v_{t} - \phi(s_t)^{\top} v_{t})G(t)^{-1}\nabla \log\pi_{\theta_t}(a_t|s_t) \rangle\\
    & = I(O_t,\theta_t,L_t,v_t,\gamma(t),G(t))\\
    &\qquad + \langle \nabla L(\theta_t,\gamma(t)) , E_{s_t \sim \mu_{\theta_t},a_t \sim \pi_{\theta_t},s_{t+1} \sim p}[h(O_t,\theta_t,L_t,v_t,\gamma(t),G(t))] \rangle.
\end{align*}

Hence,
\begin{align}\label{inequality_L(theta)_2}
    & L(\theta_{t+1},\gamma(t))\\
    &\geq L(\theta_t,\gamma(t)) + a(t)I(O_t,\theta_t,L_t,v_t,\gamma(t),G(t)) + a(t)\langle \nabla L(\theta_t,\gamma(t)) , M(\theta_t,v_t,\gamma(t),G(t)) \rangle\notag\\
&\qquad +a(t)\langle \nabla L(\theta_t,\gamma(t)) ,  \mathbb{E}_{\theta_t}[(L(\theta_t) - L_t)G(t)^{-1}\nabla \log\pi_{\theta_t}(a_t|s_t)] \rangle \notag\\
&\qquad -M_{L}a(t)^2\Vert\delta_{t}G(t)^{-1}\nabla \log\pi_{\theta_t}(a_t|s_t)\Vert^2\notag\\
& = L(\theta_t,\gamma(t)) + a(t) I(O_t,\theta_t,L_t,v_t,\gamma(t),G(t)) + a(t)\langle \bar{M}(\theta_t,v_t,\gamma(t)) , M(\theta_t,v_t,\gamma(t),G(t))\rangle \notag \\
&\qquad+a(t)\langle   \mathbb{E}_{\theta_t}[(V^{\theta_t,\gamma(t)}(s_{t+1}) -  \phi(s_{t+1})^Tv_t - V^{\theta_t,\gamma(t)}(s_t) + \phi(s_{t})^Tv_t)\nabla \log\pi_{\theta_t}(a_t|s_t)] ,  M(\theta_t,v_t,\gamma(t),G(t))] \rangle\notag\\
&\qquad + a(t)\langle \nabla L(\theta_t,\gamma(t)),  \mathbb{E}_{\theta_t}[(L(\theta_t,\gamma(t)) - L_t)G(t)^{-1}\nabla \log\pi_{\theta_t}(a_t|s_t)] \rangle \notag \\
&\qquad- M_{L}a(t)^2\Vert\delta_{t}G(t)^{-1}\nabla \log\pi_{\theta_t}(a_t|s_t)\Vert^2.\notag
\end{align}

After analysing and rearranging the terms of the inequality above we have,

\begin{align*}
    &\lambda_G\sum\limits_{k=\tau_t}^{t}\mathbb{E}\Vert \bar{M}(\theta_k,v_k,\gamma(k)) \Vert^2 \\
&\leq \underbrace{\sum\limits_{k=\tau_t}^{t}\frac{1}{a(k)}\mathbb{E}[(L(\theta_{k+1},\gamma(k)) -L(\theta_{k},\gamma(k))  )]}_{I_1} - \underbrace{\sum\limits_{k=\tau_t}^{t}\mathbb{E}[I(O_k,\theta_k,L_k,v_k,\gamma(k),G(k))]}_{I_2} \\
&\qquad + \underbrace{\sum\limits_{k=\tau_t}^{t} \EE\langle \EE_{\theta_k}[H(O_k, \theta_k,v_k,\gamma(k))], M(\theta_k,v_k,\gamma(k), G(k))\rangle]}_{I_3} \\
&\qquad \underbrace{-\sum\limits_{k=\tau_t}^{t}\mathbb{E}[\langle \nabla L(\theta_k,\gamma(k)) , \mathbb{E}_{\theta_k}[(L(\theta_k,\gamma(k)) - L_k)G(k)^{-1}\nabla \log\pi_{\theta_k}(a_k|s_k)] \rangle]}_{I_4} \\
&\qquad +\underbrace{M_{L}\sum\limits_{k=\tau_t}^{t}a(k)\mathbb{E}[\Vert\delta_{k}G(k)^{-1}\nabla \log\pi_{\theta_k}(a_k|s_k)\Vert^2]}_{I_5}.
\end{align*}

After analysing terms $I_1 - I_5$ we have,
\begin{align*}
    \sum\limits_{k=\tau_t}^{t}\mathbb{E}\Vert \bar{M}(\theta_k,v_k,\gamma(k)) \Vert^2 &\leq \mathcal{O}((\log t)^{-0.5}\cdot t^{\nu}) + \mathcal{O}((\log t)^{-0.5}\cdot t^{1 +\nu-\beta}) +  \mathcal{O}(\log ^{2.5} t \cdot t^{1-\nu})\\
    &\qquad + \frac{U_{G}}{\lambda_G} \sum\limits_{k=\tau_t}^{t} \sqrt{\EE \Vert \EE_{\theta_k}[H(O_k, \theta_k,v_k,\gamma(k))] \Vert^2}\sqrt{\EE \Vert \bar{M}(\theta_k,v_k,\gamma(k))\Vert^2}\\
    &\qquad + \frac{BU_{G}}{\lambda_G}\sqrt{\sum\limits_{k=\tau_t}^{t}\mathbb{E}\Vert \bar{M}(\theta_k,v_k,\gamma(k))\Vert^2}\sqrt{\sum\limits_{k=\tau_t}^{t}\mathbb{E}\vert L_k - L(\theta_k,\gamma(k)) \vert^2}\\
    &\qquad + \mathcal{O}\bigg( \sum\limits_{l=\tau_t}^{t} bk\bigg).
\end{align*}

After applying the squaring technique we have,

\begin{align*}
    \sum\limits_{k=\tau_t}^{t}\mathbb{E}\Vert \bar{M}(\theta_k,v_k,\gamma(k)) \Vert^2& \leq \mathcal{O}((\log t)^{-0.5}\cdot t^{\nu}) +  \mathcal{O}(\log ^{2.5} t \cdot t^{1-\nu}) + \mathcal{O}((\log t)^{-0.5}\cdot t^{1 +\nu-\beta}) \\
    &\qquad + 2\frac{B^2U_{G}^2}{\lambda_G^2}\sum\limits_{k=\tau_t}^{t}\mathbb{E}\vert L_k - L(\theta_k,\gamma(k)) \vert^2 + 2 \sum\limits_{k=\tau_t}^{t}\EE \Vert \EE_{\theta_k}[H(O_k, \theta_k,v_k,\gamma(k))] \Vert^2\\
    &\qquad + \mathcal{O}\bigg( \sum\limits_{l=\tau_t}^{t} bk\bigg)\\
    &\leq \mathcal{O}((\log t)^{-0.5}\cdot t^{\nu}) +  \mathcal{O}(\log ^{2.5} t \cdot t^{1-\nu}) + \mathcal{O}( t^{1 +\nu-\beta}) + \mathcal{O}\bigg( \sum\limits_{l=\tau_t}^{t} b^2\bigg)\\
    &\qquad + 4\frac{B^2U_{G}^2}{\lambda_G^2}G^{2}\frac{c_a^2}{c_d^2}\sum\limits_{k=\tau_t}^{t}\mathbb{E}\Vert \bar{M}(\theta_k,v_k,\gamma(k))\Vert^2
\end{align*}

Now if we select the values for $c_a$ and $c_d$ such that $4\frac{B^2U_{G}^2}{\lambda_G^2}G^{2}\frac{c_a^2}{c_d^2}< 1$, we shall obtain,
 \begin{align*}
     \sum\limits_{k=\tau_t}^{t}\mathbb{E}\Vert \bar{M}(\theta_k,v_k,\gamma(k)) \Vert^2 &= \mathcal{O}((\log t)^{-0.5}\cdot t^{\nu}) +  \mathcal{O}(\log ^{2.5} t \cdot t^{1-\nu}) + \mathcal{O}(t^{1 + \nu - \beta}) + \mathcal{O}\bigg( \sum\limits_{l=\tau_t}^{t} b^2\bigg).
 \end{align*}
 Dividing by $(1 + t - \tau_t)$  and assuming $t \geq 2\tau_t + 1$, we have,
 \begin{align}\label{ineq2_CCA}
     \frac{1}{(1 + t - \tau_t)}\sum\limits_{k=\tau_t}^{t}\mathbb{E}\Vert \bar{M}(\theta_k,v_k,\gamma(k))\Vert^2 &= \mathcal{O}((\log t)^{-0.5} \cdot t^{\nu - 1}) + \mathcal{O}(\log^{2.5} t \cdot t^{-\nu}) + \mathcal{O}(t^{ \nu - \beta}) + \mathcal{O}(b^2).
 \end{align}

 In the above equation, the $\mathcal{O}(b^2)$ term represents the mixing-time error arising from the ergodicity of the Markov chain. So after ignoring this mixing-time error we have,
 
\begin{align*}
     \frac{1}{(1 + t - \tau_t)}\sum\limits_{k=\tau_t}^{t}\mathbb{E}\Vert \bar{M}(\theta_k,v_k,\gamma(k))\Vert^2 &= \mathcal{O}((\log t)^{-0.5} \cdot t^{\nu - 1}) + \mathcal{O}(\log^{2.5} t \cdot t^{-\nu}) + \mathcal{O}(t^{ \nu - \beta}) .
 \end{align*}

\subsubsection*{Convergence of Critic:}

Recall that we have the following update rule for the critic:
\begin{align*}
    v_{n+1} = \Gamma(v_{n} + b(n)\delta_{n}f_{s_n}).
\end{align*}

Notations:
\begin{align}
    \begin{split}
        z_t &:= v_t - v^{*}(\theta_t,\gamma(t))
    \end{split}
\end{align}

From the critic update rule, we have,
\begin{align*}
    \Vert z_{t+1} \Vert^2 &= \Vert  v_{t+1} - v^{*}(\theta_{t+1},\gamma(t+1))\Vert^2\\
    &=\Vert \Gamma(v_t +b(t)\delta_t\phi(s_t)) - v^{*}(\theta_{t+1},\gamma(t+1))  \Vert^2\\
    &\leq \Vert v_t + b(t)\delta_t\phi(s_t) - v^{*}(\theta_{t+1},\gamma(t+1)) \Vert^2\\
    &= \Vert z_t + b(t)\delta_t\phi(s_t)  + v^{*}(\theta_t,\gamma(t))- v^{*}(\theta_{t+1},\gamma(t+1)) \Vert^2\\
    &\leq  \Vert z_t \Vert^2 + 2b(t)\langle z_t ,  \delta_t\phi(s_t) \rangle + 2\langle z_t , v^{*}(\theta_t,\gamma(t)) - v^{*}(\theta_{t+1},\gamma(t+1)) \rangle + 2b(t)^2\delta_t^2\Vert \phi(s_t)\Vert^2 \\
    &\qquad + 2\Vert v^{*}(\theta_t,\gamma(t)) - v^{*}(\theta_{t+1},\gamma(t+1))  \Vert^2\\
    &=\Vert z_t \Vert^2 + 2b(t)\langle z_t ,  \delta_t\phi(s_t) -  \mathbb{E}_{\theta_t}[\delta_t\phi(s_t)]\rangle +
    2b(t)\langle z_t ,    \mathbb{E}_{\theta_t}[\delta_t\phi(s_t)]\rangle\\
    &\qquad +
    2\langle z_t , v^{*}(\theta_t,\gamma(t)) - v^{*}(\theta_{t+1},\gamma(t+1)) \rangle
    + 2b(t)^2\delta_t^2\Vert \phi(s_t)\Vert^2 + 2\Vert v^{*}(\theta_t,\gamma(t)) - v^{*}(\theta_{t+1},\gamma(t+1))  \Vert^2\\
    &\leq \Vert z_t \Vert^2 + 2b(t)\langle z_t ,  \delta_t\phi(s_t) -  \mathbb{E}_{\theta_t}[\delta_t\phi(s_t)]\rangle -
    2b(t)\hat{\lambda_e}\Vert z_t\Vert^2 +
    2\langle z_t , v^{*}(\theta_t,\gamma(t)) - v^{*}(\theta_{t+1},\gamma(t+1)) \rangle\\
    &\qquad+ 2b(t)^2\delta_t^2\Vert \phi(s_t)\Vert^2 + 2\Vert v^{*}(\theta_t,\gamma(t)) - v^{*}(\theta_{t+1},\gamma(t+1))  \Vert^2.
\end{align*}

The final inequality follows from Assumption 2.

After rearranging the terms and taking summation of terms from indices $\tau_t$ to $t$ we have,
\begin{align}
    \hat{\lambda_e}\sum\limits_{k=\tau_t}^{t}\mathbb{E}\Vert z_k \Vert^2 &\leq \underbrace{\sum\limits_{k=\tau_t}^{t}\frac{1}{2b(k)}(\mathbb{E}\Vert z_k \Vert^2 - \mathbb{E}\Vert z_{k+1} \Vert^2)}_{I_1} + \underbrace{\sum\limits_{k=\tau_t}^{t}\mathbb{E}[\langle z_k ,  \delta_k\phi(s_k) - \mathbb{E}_{\theta_k}[\delta_k\phi(s_k)]\rangle]}_{I_2}\notag\\
    &\qquad + \underbrace{\sum\limits_{k=\tau_t}^{t}\frac{1}{b(k)}\mathbb{E}[\langle z_k , v^{*}(\theta_k,\gamma(k)) - v^{*}(\theta_{k+1},\gamma(k+1)) 
    + (\nabla v_k^*)^T(\theta_{k+1} - \theta_k)\rangle]}_{I_3}\notag\\
    &\qquad+ \underbrace{\sum\limits_{k=\tau_t}^{t}\frac{1}{b(k)}\mathbb{E}[\langle z_k , (\nabla v_k^*)^T(\theta_{k} - \theta_{k+1}) \rangle]}_{I_4} +\underbrace{\sum\limits_{k=\tau_t}^{t}b(k)\mathbb{E}[\delta_k^2\Vert \phi(s_k)\Vert^2]}_{I_5}\notag \\
    &\qquad + \underbrace{\sum\limits_{k=\tau_t}^{t}\frac{1}{b(k)}\mathbb{E}\Vert v^{*}(\theta_k,\gamma(k)) - v^{*}(\theta_{k+1},\gamma(k+1))  \Vert^2}_{I_6}.
\end{align}

After analysing terms $I_1 - I_6$, we have,

\begin{align*}
     \hat{\lambda_e}\sum\limits_{k=\tau_t}^{t}\mathbb{E}\Vert z_k \Vert^2 &\leq \mathcal{O}( t^{\nu}) +  \mathcal{O}(\log^{2.5} t \cdot t^{1-\nu})  + \mathcal{O}(t^{\nu - \beta + 1}) \\
     &\qquad +  L_{*}U_{G}\frac{c_a}{c_b}\log^{0.5} t\sqrt{\sum\limits_{k=\tau_t}^{t}\mathbb{E}\Vert z_k\Vert^2}\sqrt{\sum\limits_{k=\tau_t}^{t}\mathbb{E}[\Vert \bar{M}(\theta_k,v_k,\gamma(k))\Vert^2]}\\
    &\qquad + L_{*}BU_{G}\frac{c_a}{c_b}\log^{0.5} t\sqrt{\sum\limits_{k=\tau_t}^{t}\mathbb{E}\Vert z_k\Vert^2}\sqrt{\sum\limits_{k=\tau_t}^{t}\mathbb{E}[(L(\theta_k,\gamma(k))- L_k)^2]}
\end{align*}

After applying the square technique we have,

\begin{align*}
    \sum\limits_{k=\tau_t}^{t}\mathbb{E}\Vert z_k \Vert^2  &= \mathcal{O}( t^{\nu}) +  \mathcal{O}(\log^{2.5} t \cdot t^{1-\nu})  + \mathcal{O}(t^{\nu - \beta + 1})\\
    &\qquad + \mathcal{O}(\log^{0.5} t \cdot \sum\limits_{k=\tau_t}^{t}\mathbb{E}[\Vert \bar{M}(\theta_k,v_k,\gamma(k))\Vert^2]) + \mathcal{O}(\log^{0.5} t \cdot \sum\limits_{k=\tau_t}^{t}\mathbb{E}[(L(\theta_k,\gamma(k))- L_k)^2])\\
    &= \mathcal{O}( t^{\nu}) +  \mathcal{O}(\log^{3} t \cdot t^{1-\nu})  + \mathcal{O}(\log^{0.5} t \cdot t^{\nu - \beta + 1}) 
\end{align*}

Assuming $t \geq 2\tau_{t} - 1$, we have,

\begin{align*}
    \frac{1}{1+t-\tau_t}\sum\limits_{k=\tau_t}^{t}\mathbb{E}\Vert z_k \Vert^2 = \mathcal{O}( t^{\nu - 1}) +  \mathcal{O}(\log^{3} t \cdot t^{-\nu})  + \mathcal{O}(\log^{0.5} t \cdot t^{\nu - \beta}) 
\end{align*}

Optimising over the values of $\nu$ and $\beta$ we have $\nu = 0.5$ and $\beta = 1$. Hence we have  the following :-
\begin{align*}
    \frac{1}{1+t-\tau_t}\sum_{k=\tau_t}^{t}\mathbb{E}\Vert z_k \Vert^2 &= \mathcal{O}(\log^3 t \cdot t^{ - 0.5})
\end{align*}

Therefore in order for the mean squared error of the critic to be upper bounded by $\epsilon$, namely,

\begin{align*}
     \frac{1}{1+t-\tau_t}\sum_{k=\tau_t}^{t}\mathbb{E}\Vert z_k \Vert^2  =  \mathcal{O}(\log^3 T \cdot T^{- 0.5}) \leq \epsilon,
\end{align*}
we need to set $T = \tilde{\mathcal{O}}(\epsilon^{-2})$.\\

\subsection{Finite-Time Analysis of the Critic-Actor Algorithm with Modified Learning Rates}

We now provide the complete detailed finite-time analysis of the CAM algorithm. 
Please note that from this point onward we denote by $\phi(s) \in \mathbb{R}^{d_1}$ the feature vector associated with state $s$.

\subsubsection*{Convergence of Average Reward Estimate:}\label{average_reward_convergence_modified}

Notations:-
\begin{align}
    \begin{split}
        O_t :&= (s_t , a_t , s_{t+1})\\
        y_t :&= (L_t - L(\theta_t)\\
        r(s_k,a_k) ;&= \EE[r(k) | s_k,a_k]\\
        r_k :&= r(s_k,a_k)\\
        M(\theta_t,v_t) :& = \mathbb{E}_{s_t \sim \mu_{\theta_t},a_t \sim \pi_{\theta_t},s_{t+1} \sim p}[( r(t)- L(\theta_t) + \phi(s_{t+1})^{\top} v_{t} 
        - \phi(s_t)^{\top} v_{t})\nabla \log\pi_{\theta_t}(a_t|s_t)]\\
        N(O_t ,\theta_t,v_t,L_t) :& = ( r(t)- L_t  + \phi(s_{t+1})^{\top} v_{t} - \phi(s_t)^{\top} v_{t})\nabla \log\pi_{\theta_t}(a_t|s_t)\\
        U_w :&=  2B(U_{v} +  \bar{U}_{v})\\
        G :&= 2B(U_r + U_v)
    \end{split}
\end{align}

We have , $\vert V^{\theta}(s) \vert \leq \bar{U}_{v} ,\forall \theta \in \RR^{d},\forall s \in S$ where $\bar{U}_{v}$ is some positive constant.

\subsubsection*{Proof of 
Theorem \ref{cost_convergence_1}:}

From the update rule of the reward estimation recursion in  Algorithm \ref{algo_ca}, we have
\begin{align*}
    L_{t+1}-L(\theta_{t+1})=L_t-L(\theta_t)+L(\theta_t)-L(\theta_{t+1})+d(t)(r(t)-L_t).
\end{align*}
We then have
\begin{align*}
    y_{t+1}^2=&\ (y_t+L(\theta_t)-L(\theta_{t+1})+d(t)(r(t)-L_t))^2\\
    \leq &\ y_t^2+2y_t(L(\theta_t)-L(\theta_{t+1}))+2d(t) y_t(r(t)-L_t)+2(L(\theta_t)-L(\theta_{t+1})^2\\
    &\qquad +2d(t)^2(r(t)-L_t)^2\\
    = &\  (1-2d(t))y_t^2+2d(t)y_t(r(t)-L(\theta_t))+2d(t)(L(\theta_t,\gamma(t))-L(\theta_{t+1})\\
    &\qquad+2(L(\theta_t)-L(\theta_{t+1}))^2
    +2d(t)^2(r(t)-L_t)^2.
\end{align*}

Taking expectations, rearranging  and summing from $\tau_t$ to $t$ we obtain,

\begin{align}
    \sum\limits_{k=\tau_t}^{t} \mathbb{E}[y_k^2]\leq &\  \underbrace{\sum\limits_{t=\tau_t}^t\frac{1}{2d(k)}\mathbb{E}(y_k^2-y^2_{k+1})}_{I_1}+\underbrace{\sum\limits_{k=\tau_t}^{t}\mathbb{E}[y_k(r(k)-L(\theta_k))]}_{I_2} \notag \\
    & +\underbrace{\sum\limits_{k=\tau_t}^{t}\frac{1}{d(k)}\mathbb{E}[y_k(L(\theta_k)-L(\theta_{k+1})]}_{I_3} \notag \\
    &  +\underbrace{\sum\limits_{k=\tau_t}^{t}\frac{1}{d(k)}\mathbb{E}[(L(\theta_k)-L(\theta_{k+1}))^2]}_{I_4}+\underbrace{\sum\limits_{k=\tau_t}^{t}d(k)\mathbb{E}[(r(k)-L_k)^2]}_{I_5}.
\end{align}

For term $I_1$ we have,from Abel summation by parts, we have

\begin{align*}
    I_1=& \sum\limits_{k=\tau_t}^{t}\frac{1}{2d(k)}(y_k^2-y_{k+1}^2)\\
    =&\  \sum\limits_{k=\tau_t+1}^{t}y_k^2(\frac{1}{2d(k)}-\frac{1}{2d(k-1)})+\frac{1}{2d(\tau_t)}y_{\tau_t}^2-\frac{1}{d(t)}y_{t+1}^2\\
    \leq &\ \frac{2U_r^2}{d(t)}\\
    = &\  \frac{2}{c_d \cdot \ln ^{0.5} (t + 1)}U_r^2 (1+t)^{\nu}
\end{align*}

For detailed analysis of term $I_1$ kindly refer \cite{wu2022finitetimeanalysistimescale}.

We are assuming $\tau_t \geq 4$.
Now for term $I_2$ we can have the analysis similar to lemma C.5 in \cite{wu2022finitetimeanalysistimescale} and get,
\begin{align*}
    & \mathbb{E}[y_t(r(t)-L(\theta_t))]\\
    &\mathbb{E}[y_t(r_t-L(\theta_t))] \\
    &= \mathcal{O}(E\Vert \theta_t - \theta_{t-\tau} \Vert) +\mathcal{O}(E\vert L_t - L_{t-\tau} \vert)+\mathcal{O}(\sum_{i=t-\tau}^{t} E\Vert\theta_i - \theta_{t-\tau} \Vert)+\mathcal{O}(bk^{\tau - 1})
\end{align*}

Hence we have,
\begin{align*}
   I_2 =  &\sum\limits_{k=\tau_t}^{t}\mathbb{E}[y_k(r(k)-L(\theta_k))]\\
    &= \mathcal{O}(\tau_t^2 \sum\limits_{k=\tau_t}^{t} a(k-\tau_t))\\
    &= \mathcal{O}(\tau_t^2 \ln^{0.5}(t+1)\sum\limits_{k=\tau_t}^{t} \frac{1}{(1 + k)^{\nu}})\\
    & = \mathcal{O}(\log^{2.5} t \cdot t^{1-\nu})
\end{align*}

For term $I_3$ we have,

\begin{align*}
    I_3=& \sum\limits_{k=\tau_t}^{t}\frac{1}{d(k)}\mathbb{E}[y_k(L(\theta_k)-L(\theta_{k+1}))]\\
    \leq &\ \sum\limits_{k=\tau_t}^{t}\frac{1}{d(k)}\mathbb{E}[L_{J'}U_r\Vert \theta_k-\theta_{k+1}\Vert^2+|y_k|\Vert \theta_k-\theta_{k+1}\Vert \Vert M(\theta_k,v_k)\Vert]\\
     &\qquad + \sum\limits_{k=\tau_t}^{t}\frac{1}{d(k)}\mathbb{E}[y_k\langle \mathbb{E}_{\theta_k}[(V^{\theta_{k}}(s_{k+1}) - v(k)^T\phi(s_{k+1}) - V^{\theta_{k}}(s_{k}) \\
     &\qquad+ v(k)^T\phi(s_{k}))\nabla \log\pi_{\theta_k}(a_k|s_k)] , \theta_k - \theta_{k+1} \rangle]\\
    \leq &\ \sum\limits_{k=\tau_t}^{t}\mathbb{E}[L_{J'}U_rG^2\frac{a(k)^2}{d(k)} + G\frac{c_a {\log t}^{0.5}}{c_d}|y_k|\Vert M(\theta_k,v_k)\Vert]\\
     &\qquad + \sum\limits_{k=\tau_t}^{t}\frac{1}{d(k)}\mathbb{E}[y_k\langle \mathbb{E}_{\theta_k}[(V^{\theta_{k}}(s_{k+1}) - v(k)^T\phi(s_{k+1}) - V^{\theta_{k}}(s_{k})\\
     &\qquad+ v(k)^T\phi(s_{k}))\nabla \log\pi_{\theta_k}(a_k|s_k)] , \theta_k - \theta_{k+1} \rangle]\\
    \leq &\ \frac{2L_{J'}U_rG^2c_a^2 \ln^{0.5} (t+1)}{c_d}(1+t-\tau_t)^{1-\nu}+ G\frac{c_a }{c_d}(\sum\limits_{k=\tau_t}^{t}\mathbb{E}y_t^2)^{\frac{1}{2}}(\sum\limits_{k=\tau_t}^{t}\mathbb{E}\Vert M(\theta_k,v_k)\Vert^2)^\frac{1}{2}\\
    &\qquad + \sum\limits_{k=\tau_t}^{t}\frac{1}{d(k)}\mathbb{E}[y_k\langle \mathbb{E}_{\theta_k}[(V^{\theta_{k}}(s_{k+1}) - v(k)^T\phi(s_{k+1}) - V^{\theta_{k}}(s_{k})\\
     &\qquad+ v(k)^T\phi(s_{k}))\nabla \log\pi_{\theta_k}(a_k|s_k)] , \theta_k - \theta_{k+1} \rangle]\\
\end{align*}

Now,

\begin{align*}
    &\mathbb{E}_{\theta_t}[(V^{\theta_t}(s_{t+1}) -  \phi(s_{t+1})^Tv_t - V^{\theta_t}(s_t) +\phi(s_{t})^Tv_t)\nabla \log\pi_{\theta_t}(a_t|s_t)]\\
    =& \mathbb{E}_{\theta_t}[(V^{\theta_t}(s_{t+1}) -  \phi(s_{t+1})^Tv_t )\nabla \log\pi_{\theta_t}(a_t|s_t)] + \mathbb{E}_{\theta_t}[(-V^{\theta_t}(s_{t}) + \phi(s_{t})^Tv_t )\nabla \log\pi_{\theta_t}(a_t|s_t)]\\
    =& \mathbb{E}_{\theta_t}[(V^{\theta_t}(s_{t+1}) - \phi(s_{t+1})^Tv_t )\nabla \log\pi_{\theta_t}(a_t|s_t)] + \mathbb{E}_{s_t}[(-V^{\theta_t}(s_{t}) + \phi(s_{t})^Tv_t )\sum\limits_{a}\pi_{\theta_t}(a|s_t)\nabla \log\pi_{\theta_t}(a|s_t)]\\
    =& \mathbb{E}_{\theta_t}[(V^{\theta_t}(s_{t+1}) -  \phi(s_{t+1})^Tv_t )\nabla \log\pi_{\theta_t}(a_t|s_t)] + 0\\
\end{align*}

Let $ H(O_t, \theta_t,v_t) = (V^{\theta_t}(s_{t+1}) - \phi(s_{t+1})^Tv_t )\nabla \log\pi_{\theta_t}(a_t|s_t) $.
Now, 

\begin{align*}
    &\Vert \EE_{\theta_t}[H(O_t, \theta_t,v_t)]\Vert\\
    \leq & \Vert \EE_{\theta_t}[H(O_t, \theta_t,v_t)] - \EE_{s \sim \mu_{\theta_t},  a \sim \pi_{\theta_t}, s' \sim \mu_{\theta_t}}[H(O,\theta_t,v_t)]\Vert + \Vert \EE_{s \sim \mu_{\theta_t},  a \sim \pi_{\theta_t}, s' \sim \mu_{\theta_t}}[H(O,\theta_t,v_t)]\Vert\\
    =& \mathcal{O}(bk)
\end{align*}
Hence collecting all the terms, we have,
\begin{align*}
    I_3 = &\frac{2L_{J'}U_rG^2c_a^2 \ln^{0.5} (t+1)}{c_d}(1+t-\tau_t)^{1-\nu}+ G\frac{c_a }{c_d}(\sum\limits_{k=\tau_t}^{t}\mathbb{E}y_t^2)^{\frac{1}{2}}(\sum\limits_{k=\tau_t}^{t}\mathbb{E}\Vert M(\theta_k,v_k)\Vert^2)^\frac{1}{2}\\
    &\qquad + \mathcal{O}\bigg( \bigg(\sum\limits_{k=\tau_t}^{t}\mathbb{E}y_t^2\bigg)^{\frac{1}{2}}\bigg(\sum\limits_{k=\tau_t}^{t}b^2\bigg)^\frac{1}{2}\bigg)
\end{align*}

For term $I_4$, we have
\begin{align*}
    I_4=& \sum\limits_{k=\tau_t}^{t}\frac{1}{d(k)}\mathbb{E}[(L(\theta_k)-L(\theta_{k+1}))^2]\\
    =& \mathcal{O} \bigg(\sum\limits_{k=\tau_t}^{t} \frac{a(k)^2}{d(k)}\bigg)\\
    = & \mathcal{O}(\ln^{0.5} t \cdot t^{1-\nu}).
\end{align*}

For term $I_5$, we have
\begin{align*}
    I_5=& \sum\limits_{k=\tau_t}^{t}d(k)\mathbb{E}[(r(k)-L_k)^2]\\
    = & \mathcal{O}(\ln^{0.5} t \cdot t^{1-\nu}).
\end{align*}
For detailed analysis of terms $I_4$ and  $I_5$ refer to the analysis of term $I_2$ in page number 22 of \cite{wu2022finitetimeanalysistimescale}.\\
Hence putting together terms $I_1 - I_5$ we have,
\begin{align*}
    \sum\limits_{k=\tau_t}^{t} \mathbb{E}[y_k^2]&\leq \frac{2}{c_d \cdot \ln ^{0.5} (t + 1)}U_r^2 (1+t)^{\nu} + \mathcal{O}(\log^{2.5} t \cdot t^{1-\nu}) + \frac{2L_{J'}U_rG^2c_a^2 \ln^{0.5} (t+1)}{c_d}(1+t-\tau_t)^{1-\nu}\\
    &\qquad+ G\frac{c_a }{c_d}(\sum\limits_{k=\tau_t}^{t}\mathbb{E}y_t^2)^{\frac{1}{2}}(\sum\limits_{k=\tau_t}^{t}\mathbb{E}\Vert M(\theta_k,v_k)\Vert^2)^\frac{1}{2}\\
    &\qquad + \mathcal{O}\bigg( \bigg(\sum\limits_{k=\tau_t}^{t}\mathbb{E}y_t^2\bigg)^{\frac{1}{2}}\bigg(\sum\limits_{k=\tau_t}^{t}b^2\bigg)^\frac{1}{2}\bigg).
\end{align*}

\begin{align*}
    \Rightarrow &\sum\limits_{k=\tau_t}^{t} \mathbb{E}[y_k^2]
    \leq  \frac{2}{c_d \cdot \ln ^{0.5} (t + 1)}U_r^2 (1+t)^{\nu} + \mathcal{O}(\log^{2.5} t \cdot t^{1-\nu}) + \frac{2L_{J'}U_rG^2c_a^2 \ln^{0.5} (t+1)}{c_d}(1+t-\tau_t)^{1-\nu}\\
    &\qquad+ G\frac{c_a }{c_d}(\sum\limits_{k=\tau_t}^{t}\mathbb{E}y_t^2)^{\frac{1}{2}}(\sum\limits_{k=\tau_t}^{t}\mathbb{E}\Vert M(\theta_k,v_k)\Vert^2)^\frac{1}{2} + 
   \mathcal{O}\bigg( \bigg(\sum\limits_{k=\tau_t}^{t}\mathbb{E}y_t^2\bigg)^{\frac{1}{2}}\bigg(\sum\limits_{k=\tau_t}^{t}b^2\bigg)^\frac{1}{2}\bigg).
\end{align*}

Hence, we obtain:

\begin{align*}
    \sum\limits_{k=\tau_t}^{t} \mathbb{E}[y_k^2] &\leq \mathcal{O}(\log^{-0.5} t \cdot  t^{\nu}) + \mathcal{O}(\log^{2.5} t \cdot t^{1-\nu}) + G\frac{c_a }{c_d}(\sum\limits_{k=\tau_t}^{t}\mathbb{E}y_t^2)^{\frac{1}{2}}(\sum\limits_{k=\tau_t}^{t}\mathbb{E}\Vert M(\theta_k,v_k)\Vert^2)^\frac{1}{2}\\
    &\qquad + \mathcal{O}\bigg( \bigg(\sum\limits_{k=\tau_t}^{t}\mathbb{E}y_t^2\bigg)^{\frac{1}{2}}\bigg(\sum\limits_{k=\tau_t}^{t}b^2\bigg)^\frac{1}{2}\bigg).
\end{align*}

After applying the squaring technique (see page 23 of \citep{wu2022finitetimeanalysistimescale}), we have,

\begin{align}\label{ineq1}
    \sum\limits_{k=\tau_t}^{t} \mathbb{E}[y_k^2] &\leq \mathcal{O}(\log^{-0.5} t \cdot  t^{\nu}) + \mathcal{O}(\log^{2.5} t \cdot t^{1-\nu})  + 
     2G^2\frac{c_a^2}{c_d^2}\sum\limits_{k=\tau_t}^{t}\mathbb{E}\Vert M(\theta_k,v_k)\Vert^2 + \mathcal{O}\bigg(\sum_{k = \tau_t}^{t}b^2 \bigg) 
\end{align}
In the above equation, the $\mathcal{O}(b^2)$ term represents the mixing-time error arising from the ergodicity of the Markov chain. So after ignoring this mixing-time error we have,

\begin{align*}
    \frac{1}{t - \tau_t + 1}\sum\limits_{k=\tau_t}^{t} \mathbb{E}[y_k^2] &\leq \mathcal{O}(\log^{-0.5} t \cdot  t^{\nu - 1}) + \mathcal{O}(\log^{2.5} t \cdot t^{-\nu})  
    + 2G^2\frac{c_a^2}{c_d^2} \frac{1}{t - \tau_t + 1}\sum\limits_{k=\tau_t}^{t}\mathbb{E}\Vert M(\theta_k,v_k)\Vert^2 .
\end{align*}

\subsubsection*{Convergence of Actor:}\label{actor_convergence_modified}
Notations used here:

\begin{align}
    \begin{split}
        O_t :&= (s_t , a_t , s_{t+1})\\
        h(O_t,\theta_t,L_t,v_t) :&= ( r(t)- L_t  + \phi(s_{t+1})^{\top} v_{t} - \phi(s_t)^{\top} v_{t})\nabla \log\pi_{\theta_t}(a_t|s_t)\\
        I(O_t,L_t,\theta_t,v_t) :&= \langle \nabla L(\theta_t,\gamma(t)) ,h(O_t,\theta_t,L_t,v_t)\\
        &\qquad - \mathbb{E}_{s_t \sim \mu_{\theta_t},a_t \sim \pi_{\theta_t},s_{t+1} \sim p}[h(O_t,\theta_t,L_t,v_t)] \rangle  \\
        \bar{h}(O_t,\theta_t,v_t) :&= ( r(t)- L(\theta_t)  + \phi(s_{t+1})^{\top} v_{t}\\
        &\qquad - \phi(s_t)^{\top} v_{t})\nabla \log\pi_{\theta_t}(a_t|s_t)\\
        M(\theta_t,v_t) :& = \mathbb{E}_{s_t \sim \mu_{\theta_t},a_t \sim \pi_{\theta_t},s_{t+1} \sim p}[\bar{h}(O_t,\theta_t,v_t)]
    \end{split}
\end{align}

The following lemmas are needed for deriving finite time bounds while proving convergence of the actor recursion. 

\begin{lemma}\label{V_lipschitz_theta}
$\exists L_{v} > 0$ such that for any $s \in S$, 
    \begin{align*}
        \Vert V^{\theta_1}(s) - V^{\theta_2}(s) \Vert \leq L_{v}\Vert \theta_1 - \theta_2 \Vert , \forall \theta_1,\theta_2 \in \RR^{d}.
    \end{align*}
\end{lemma}

\begin{proof}
For finite state and action MDPs, under bounded costs and with  policy continuously differentiable and which selects each action with a positive probability, one can show that the stationary distribution and so the average reward is continuously differentiable and bounded, making it Lipschitz (see Theorem A.1, \cite{Panda_Bhatnagar_2025}) . The differential value function can be seen to then be differentiable from the fact that 
\[
\sum_{k=0}^K P_\pi^k(H^\pi - Q(\pi)) = V^\pi-P_\pi^{K+1}V^\pi,
\]
where $H^\pi = E[r |X_0=i, \pi]$ and $Q(\pi) = J(\pi)$, 
see p.284,\cite{Bertsekas2013OptimalControlVol2}, upon noting that $\lim_{K\rightarrow\infty} P_\pi^{K+1}V^\pi=0$.
\end{proof}

\subsubsection*{Proof of 
Theorem \ref{actor_convergence_1}:}

After applying Lemma 1 of \cite{Panda_Bhatnagar_2025} to the update rule of the actor, we have,

\begin{align*}
    L(\theta_{t+1}) &\geq L(\theta_t) + a(t) \langle \nabla L(\theta_t) ,\delta_{t}\nabla \log\pi_{\theta_t}(a_t|s_t) \rangle- M_{L}a(t)^2\Vert\delta_{t}\nabla \log\pi_{\theta_t}(a_t|s_t)\Vert^2.
\end{align*}

For the term $\langle \nabla L(\theta_t) ,\delta_{t}\nabla \log\pi_{\theta_t}(a_t|s_t) \rangle $, we have,
\begin{align*}
    &\langle \nabla L(\theta_t) ,\delta_{t}\nabla \log\pi_{\theta_t}(a_t|s_t) \rangle \\
    &= \langle \nabla L(\theta_t) ,( r(t) - L_t + \phi(s_{t+1})^{\top} v_{t} - \phi(s_t)^{\top} v_{t})\nabla \log\pi_{\theta_t}(a_t|s_t) \rangle\\
    & = I(O_t,\theta_t,L_t,v_t)+ \langle \nabla L(\theta_t) , E_{s_t \sim \mu_{\theta_t},a_t \sim \pi_{\theta_t},s_{t+1} \sim p}[h(O_t,\theta_t,L_t,v_t)] \rangle.
\end{align*}

Hence,
\begin{align}\label{inequality_L(theta)}
    & L(\theta_{t+1})\\
    &\geq L(\theta_t) + a(t)I(O_t,\theta_t,L_t,v_t) + a(t)\langle \nabla L(\theta_t) , M(\theta _t,v_t) \rangle +a(t)\langle \nabla L(\theta_t) ,  \mathbb{E}_{\theta_t}[(L(\theta_t) - L_t)\nabla \log\pi_{\theta_t}(a_t|s_t)] \rangle \notag\\
&\qquad -M_{L}a(t)^2\Vert\delta_{t}\nabla \log\pi_{\theta_t}(a_t|s_t)\Vert^2\notag\\
& = L(\theta_t) + a(t) I(O_t,\theta_t,L_t,v_t) + a(t)\langle M(\theta_t,v_t) , M(\theta_t,v_t)\rangle \notag \\
&\qquad+a(t)\langle   \mathbb{E}_{\theta_t}[(V^{\theta_t}(s_{t+1}) -  \phi(s_{t+1})^Tv_t - V^{\theta_t}(s_t) + \phi(s_{t})^Tv_t)\nabla \log\pi_{\theta_t}(a_t|s_t)] ,   M(\theta_t,v_t) \rangle\notag\\
&\qquad + a(t)\langle \nabla L(\theta_t),  \mathbb{E}_{\theta_t}[(L(\theta_t) - L_t)\nabla \log\pi_{\theta_t}(a_t|s_t)] \rangle - M_{L}a(t)^2\Vert\delta_{t}\nabla \log\pi_{\theta_t}(a_t|s_t)\Vert^2.\notag 
\end{align}

where $\EE_{\theta_t}[.]$ denotes $\EE_{s_t \sim \mu_{\theta_t} , a_t \sim \pi_{\theta_t}(.|s_t), s_{t+1} \sim p(s_t,.,a_t)}[.]$.

Now,

\begin{align*}
    &\mathbb{E}_{\theta_t}[(-V^{\theta_t}(s_{t+1}) +  \phi(s_{t+1})^Tv_t + V^{\theta_t}(s_t) - \phi(s_{t})^Tv_t)\nabla \log\pi_{\theta_t}(a_t|s_t)]\\
    =& \mathbb{E}_{\theta_t}[(-V^{\theta_t}(s_{t+1}) +  \phi(s_{t+1})^Tv_t )\nabla \log\pi_{\theta_t}(a_t|s_t)] + \mathbb{E}_{\theta_t}[(V^{\theta_t}(s_{t}) - \phi(s_{t})^Tv_t )\nabla \log\pi_{\theta_t}(a_t|s_t)]\\
    =& \mathbb{E}_{\theta_t}[(-V^{\theta_t}(s_{t+1}) +  \phi(s_{t+1})^Tv_t )\nabla \log\pi_{\theta_t}(a_t|s_t)] + \mathbb{E}_{s_t}[(V^{\theta_t}(s_{t}) - \phi(s_{t})^Tv_t )\sum\limits_{a}\pi_{\theta_t}(a|s_t)\nabla \log\pi_{\theta_t}(a|s_t)]\\
    =& \mathbb{E}_{\theta_t}[(-V^{\theta_t}(s_{t+1}) +  \phi(s_{t+1})^Tv_t )\nabla \log\pi_{\theta_t}(a_t|s_t)] + 0\\
\end{align*}

Let $ H(O_t, \theta_t,v_t) = (-V^{\theta_t}(s_{t+1}) + \phi(s_{t+1})^Tv_t )\nabla \log\pi_{\theta_t}(a_t|s_t) $.
Now, 

\begin{align} \label{ergodicity_error}
    &\Vert \EE_{\theta_t}[H(O_t, \theta_t,v_t)]\Vert \notag\\
    \leq & \Vert \EE_{\theta_t}[H(O_t, \theta_t,v_t)] - \EE_{s \sim \mu_{\theta_t},  a \sim \pi_{\theta_t}, s' \sim \mu_{\theta_t}}[H(O,\theta_t,v_t)]\Vert + \Vert \EE_{s \sim \mu_{\theta_t},  a \sim \pi_{\theta_t}, s' \sim \mu_{\theta_t}}[H(O,\theta_t,v_t)]\Vert\notag\\
    =& \mathcal{O}(bk)
\end{align}

Looking back at inequality (\ref{inequality_L(theta)}), we have the following:

\begin{align*}
    &\sum\limits_{k=\tau_t}^{t}\mathbb{E}\Vert M(\theta_k,v_k) \Vert^2 \\
&\leq \underbrace{\sum\limits_{k=\tau_t}^{t}\frac{1}{a(k)}\mathbb{E}[(L(\theta_{k+1}) -L(\theta_{k})  )]}_{I_1} - \underbrace{\sum\limits_{k=\tau_t}^{t}\mathbb{E}[I(O_k,\theta_k,L_k,v_k)]}_{I_2} + \underbrace{\sum\limits_{k=\tau_t}^{t} \EE\langle \EE_{\theta_k}[H(O_k, \theta_k,v_k)], M(\theta_k,v_k)\rangle]}_{I_3} \\
&\qquad \underbrace{-\sum\limits_{k=\tau_t}^{t}\mathbb{E}[\langle \nabla L(\theta_k) , \mathbb{E}_{\theta_k}[(L(\theta_k) - L_k)\nabla \log\pi_{\theta_k}(a_k|s_k)] \rangle]}_{I_4} \\
&\qquad +\underbrace{M_{L}\sum\limits_{k=\tau_t}^{t}a(k)\mathbb{E}[\Vert\delta_{k}\nabla \log\pi_{\theta_k}(a_k|s_k)\Vert^2]}_{I_5}.
\end{align*}

Now, for term $I_1$ we have,
\begin{align*}
    &\sum\limits_{k=\tau_t}^{t}\frac{1}{a(k)}\\ &=\mathbb{E}[(L(\theta_{k+1}) -L(\theta_{k})  )]\\
    &= \mathcal{O}(1/a(t))\\
    &= \mathcal{O}((\log t)^{-0.5}\cdot t^{\nu})
\end{align*}

The analysis of term $I_1$ is similar to the analysis of term $I_1$ in the proof of average reward estimate convergence above.

We are assuming $\tau_t \geq 4$.

For term $I_2$, we have,
\begin{align*}
    I_2 = \mathcal{O}((\log t)^{2.5}\cdot t^{1-\nu})
\end{align*}

The analysis of term $I_2$ is similar to the analysis in lemma C.3 of \cite{wu2022finitetimeanalysistimescale}.

For term $I_3$ we have,

\begin{align*}
    &\sum\limits_{k=\tau_t}^{t} \langle \EE[\EE_{\theta_k}[H(O_k, \theta_k,v_k)], M(\theta_k,v_k)\rangle]\\
    \leq & \sum\limits_{k=\tau_t}^{t} \sqrt{\EE \Vert \EE_{\theta_k}[H(O_k, \theta_k,v_k)] \Vert^2}\sqrt{\EE \Vert M(\theta_k,v_k)\Vert^2}
\end{align*}

For term $I_4$ we have,

\begin{align*}
    -&\sum\limits_{k=\tau_t}^{t}\mathbb{E}[\langle \nabla L(\theta_k) ,\mathbb{E}_{\theta_k}[(L(\theta_k) - L_k)\nabla \log\pi_{\theta_k}(a_k|s_k)] \rangle]\\
    =&\ \sum\limits_{k=\tau_t}^{t}\mathbb{E}[\langle \mathbb{E}_{\theta_k}[(r(s,a) - L(\theta_k) + V^{\theta_k}(s^{'}) - V^{\theta_k}(s))\nabla \log \pi_{\theta_k}(a|s)] \\
    &\qquad \qquad \qquad, ( L_k - L(\theta_k,\gamma(k)))\mathbb{E}_{\theta_k}[\nabla \log\pi_{\theta_k}(a_k|s_k)]] \rangle]\notag\\
    = &\ \sum\limits_{k=\tau_t}^{t}E[\langle \mathbb{E}_{\theta_k}[(r(s,a) - L(\theta_k) + (\phi(s^{'}) - \phi(s))^Tv(k))\nabla \log \pi_{\theta_k}(a|s)]\\
    &\qquad, ( L_k - L(\theta_k))\mathbb{E}_{\theta_k}[\nabla \log\pi_{\theta_k}(a_k|s_k)]] \rangle]\notag\\
    &+ \sum\limits_{k=\tau_t}^{t}\mathbb{E}[\langle\mathbb{E}_{\theta_k}[( V^{\theta_k}(s^{'})- \phi(s^{'})^Tv_{k} + \phi(s)^Tv_{k} - V^{\theta_k}(s))\nabla \log \pi_{\theta_k}(a|s)] \\
    &\qquad \qquad, ( L_k - L(\theta_k))\mathbb{E}_{\theta_k}[\nabla \log\pi_{\theta_k}(a_k|s_k)]] \rangle]\notag\\
    \leq  &\ B\sqrt{\sum\limits_{k=\tau_t}^{t}\mathbb{E}\Vert M(\theta_k,v_k)\Vert^2}\sqrt{\sum\limits_{k=\tau_t}^{t}\mathbb{E}\vert L_k - L(\theta_k) \vert^2} + I_{4a}.
\end{align*}

where 
\begin{align*}
    I_{4a} = &\sum\limits_{k=\tau_t}^{t}\mathbb{E}[\langle \mathbb{E}_{\theta_k}[( V^{\theta_k}(s^{'})- \phi(s^{'})^Tv_{k} + \phi(s)^Tv_{k} - V^{\theta_k}(s))\nabla \log \pi_{\theta_k}(a|s)] \\
    &\qquad \qquad, ( L_k - L(\theta_k))\mathbb{E}_{\theta_k}[\nabla \log\pi_{\theta_k}(a_k|s_k)]] \rangle]\\
    =& \mathcal{O}\bigg( \sum\limits_{l=\tau_t}^{t} bk\bigg)
\end{align*}

The analysis in inequality (\ref{ergodicity_error}) is used in the analysis of term $I_{4a}$ .\\
Hence gathering all the terms for $I_4$, we have,

\begin{align*}
    I_4 \leq B\sqrt{\sum\limits_{k=\tau_t}^{t}\mathbb{E}\Vert M(\theta_k,v_k)\Vert^2}\sqrt{\sum\limits_{k=\tau_t}^{t}\mathbb{E}\vert L_k - L(\theta_k) \vert^2} + \mathcal{O}\bigg( \sum\limits_{l=\tau_t}^{t} bk\bigg).
\end{align*}

For term $I_5$, we have,

\begin{align*}
  &M_{L}\sum\limits_{k=\tau_t}^{t}a(k)\mathbb{E}[\Vert\delta_{k}\nabla \log\pi_{\theta_k}(a_k|s_k)\Vert^2] \\
  &= \mathcal{O}\bigg( \sum\limits_{k=\tau_t}^{t}a(k)\bigg)\\
    & = \mathcal{O}(\log ^{0.5} t \cdot t^{1-\nu}).
\end{align*}

For detailed analysis of term $I_5$ refer to the analysis of term $I_2$ in page number 22 of \cite{wu2022finitetimeanalysistimescale}.\\
Now, gathering all the terms we have,

\begin{align*}
    &\sum\limits_{k=\tau_t}^{t}\mathbb{E}\Vert M(\theta_k,v_k) \Vert^2\\
    &\leq \mathcal{O}((\log t)^{-0.5}\cdot t^{\nu}) +  \mathcal{O}(\log ^{2.5} t \cdot t^{1-\nu}) + \mathcal{O}\bigg( \sum\limits_{l=\tau_t}^{t} bk\bigg)\\
    &\qquad + B\sqrt{\sum\limits_{k=\tau_t}^{t}\mathbb{E}\Vert M(\theta_k,v_k)\Vert^2}\sqrt{\sum\limits_{k=\tau_t}^{t}\mathbb{E}\vert L_k - L(\theta_k) \vert^2} +  \sqrt{\sum\limits_{k=\tau_t}^{t}\EE \Vert \EE_{\theta_k}[H(O_k, \theta_k,v_k)] \Vert^2}\sqrt{\sum\limits_{k=\tau_t}^{t}\EE \Vert M(\theta_k,v_k)\Vert^2}
\end{align*}

After applying the squaring technique we have,

\begin{align*}
    \sum\limits_{k=\tau_t}^{t}\mathbb{E}\Vert M(\theta_k,v_k) \Vert^2& \leq \mathcal{O}((\log t)^{-0.5}\cdot t^{\nu}) +  \mathcal{O}(\log ^{2.5} t \cdot t^{1-\nu}) + 2B^2\sum\limits_{k=\tau_t}^{t}\mathbb{E}\vert L_k - L(\theta_k) \vert^2\\
    &\qquad + 2 \sum\limits_{k=\tau_t}^{t}\EE \Vert \EE_{\theta_k}[H(O_k, \theta_k,v_k)] \Vert^2 + \mathcal{O}\bigg( \sum\limits_{l=\tau_t}^{t} bk\bigg)\\
    &\leq \mathcal{O}((\log t)^{-0.5}\cdot t^{\nu}) +  \mathcal{O}(\log ^{2.5} t \cdot t^{1-\nu}) +\mathcal{O}\bigg( \sum\limits_{l=\tau_t}^{t} b^2\bigg)  \\
    &\qquad + 4B^2G^{2}\frac{c_a^2}{c_d^2}\sum\limits_{k=\tau_t}^{t}\mathbb{E}\Vert M(\theta_k,v_k)\Vert^2
\end{align*}

The last inequality follows from \ref{ineq1}.

Now if we select the values for $c_a$ and $c_d$ such that $4B^2G^{2}\frac{c_a^2}{c_d^2} < 1$, we shall obtain,
 \begin{align*}
     \sum\limits_{k=\tau_t}^{t}\mathbb{E}\Vert M(\theta_k,v_k) \Vert^2 &= \mathcal{O}((\log t)^{-0.5}\cdot t^{\nu}) +  \mathcal{O}(\log ^{2.5} t \cdot t^{1-\nu}) + \mathcal{O}\bigg( \sum\limits_{l=\tau_t}^{t} b^2\bigg) .
 \end{align*}
 Dividing by $(1 + t - \tau_t)$  and assuming $t \geq 2\tau_t + 1$, we have,
 \begin{align}\label{ineq2}
     \frac{1}{(1 + t - \tau_t)}\sum\limits_{k=\tau_t}^{t}\mathbb{E}\Vert M(\theta_k,v_k)\Vert^2 &= \mathcal{O}((\log t)^{-0.5} \cdot t^{\nu - 1}) + \mathcal{O}(\log^{2.5} t \cdot t^{-\nu}) + \mathcal{O}(  b^2) .
 \end{align}
 In the above equation, the $\mathcal{O}(b^2)$ term represents the mixing-time error arising from the ergodicity of the Markov chain. So after ignoring this mixing-time error we have,
  \begin{align*}
     \frac{1}{(1 + t - \tau_t)}\sum\limits_{k=\tau_t}^{t}\mathbb{E}\Vert M(\theta_k,v_k)\Vert^2 &= \mathcal{O}((\log t)^{-0.5} \cdot t^{\nu - 1}) + \mathcal{O}(\log^{2.5} t \cdot t^{-\nu})  .
 \end{align*}

\subsubsection*{Convergence of the Critic:}

Recall that we have the following update rule for the critic:
\begin{align*}
    v_{n+1} = \Gamma(v_{n} + b(n)\delta_{n}f_{s_n}).
\end{align*}

Notations:
\begin{align}
    \begin{split}
         O_t :&= (s_t , a_t , s_{t+1})\\
        z_t &:= v_t - v^{*}(\theta_t)\\
        g(O_t,v_t,\theta_t) &:= (r(t) - L(\theta_t) + \phi(s_{t+1})^{\top} v_{t} - \phi(s_t)^{\top} v_{t})\phi(s_t)\\
        \bar{g}(v_t,\theta_t) &:= \mathbb{E}_{s_t \sim \mu_{\theta_t},a_t \sim \pi_{\theta_t},s_{t+1} \sim p(.|s,a)}[(r(t) - L(\theta_t) + \phi(s_{t+1})^{\top} v_{t} - \phi(s_t)^{\top} v_{t})\phi(s_t)]\\
        \bar{Q}(O_t,v_t,\theta_t) &:= \langle z_t , g(O_t,v_t,\theta_t) - \bar{g}(v_t,\theta_t)\rangle \\
        \bar{U}(O_t,v_t,\theta_t) &:= (\nabla v_t^*)^T (r(t) - L(\theta_t) + \phi(s_{t+1})^{\top} v_{t} - \phi(s_t)^{\top} v_{t}) \nabla_{\theta} \log \pi_{\theta_{t}}(a_t|s_t)\\
        \Psi(O_t,v_t,\theta_t) &:= \langle z_t ,   \mathbb{E}_{\theta_t}[\bar{U}(O_t,v_t,\theta_t)] - \bar{U}(O_t,v_t,\theta_t)\rangle.
    \end{split}
\end{align}

\subsubsection*{Proof of 
Theorem \ref{critic_convergence_1}:}
From the critic update rule, we have,
\begin{align*}
    \Vert z_{t+1} \Vert^2 &= \Vert  v_{t+1} - v^{*}(\theta_{t+1})\Vert^2\\
    &=\Vert \Gamma(v_t +b(t)\delta_t\phi(s_t)) - v^{*}(\theta_{t+1})  \Vert^2\\
    &\leq \Vert v_t + b(t)\delta_t\phi(s_t) - v^{*}(\theta_{t+1}) \Vert^2\\
    &= \Vert z_t + b(t)\delta_t\phi(s_t)  + v^{*}(\theta_t)- v^{*}(\theta_{t+1}) \Vert^2\\
    &\leq  \Vert z_t \Vert^2 + 2b(t)\langle z_t ,  \delta_t\phi(s_t) \rangle + 2\langle z_t , v^{*}(\theta_t) - v^{*}(\theta_{t+1}) \rangle + 2b(t)^2\delta_t^2\Vert \phi(s_t)\Vert^2 \\
    &\qquad + 2\Vert v^{*}(\theta_t) - v^{*}(\theta_{t+1})  \Vert^2\\
    &=\Vert z_t \Vert^2 + 2b(t)\langle z_t ,  \delta_t\phi(s_t) -  \mathbb{E}_{\theta_t}[\delta_t\phi(s_t)]\rangle +
    2b(t)\langle z_t ,    \mathbb{E}_{\theta_t}[\delta_t\phi(s_t)]\rangle\\
    &\qquad +
    2\langle z_t , v^{*}(\theta_t) - v^{*}(\theta_{t+1}) \rangle
    + 2b(t)^2\delta_t^2\Vert \phi(s_t)\Vert^2 + 2\Vert v^{*}(\theta_t) - v^{*}(\theta_{t+1})  \Vert^2\\
    &\leq \Vert z_t \Vert^2 + 2b(t)\langle z_t ,  \delta_t\phi(s_t) -  \mathbb{E}_{\theta_t}[\delta_t\phi(s_t)]\rangle -
    2b(t)\hat{\lambda_e}\Vert z_t\Vert^2 +
    2\langle z_t , v^{*}(\theta_t) - v^{*}(\theta_{t+1}) \rangle\\
    &\qquad+ 2b(t)^2\delta_t^2\Vert \phi(s_t)\Vert^2 + 2\Vert v^{*}(\theta_t) - v^{*}(\theta_{t+1})  \Vert^2.
\end{align*}

The final inequality follows from Assumption 2.

After rearranging the terms we obtain,
\begin{align*}
    \hat{\lambda_e}\Vert z_t \Vert^2 &\leq \frac{1}{2b(t)}(\Vert z_t \Vert^2 - \Vert z_{t+1} \Vert^2) + \langle z_t ,  \delta_t\phi(s_t) -  \mathbb{E}_{\theta_t}[\delta_t\phi(s_t)]\rangle + \frac{1}{b(t)}\langle z_t , v^{*}(\theta_t) - v^{*}(\theta_{t+1}) \\
    &\qquad + (\nabla v_t^*)^T(\theta_{t+1} - \theta_t)\rangle
    + \frac{1}{b(t)}\langle z_t , (\nabla v_t^*)^T(\theta_{t} - \theta_{t+1}) \rangle +b(t)\delta_t^2\Vert \phi(s_t)\Vert^2 \\
    &\qquad + \frac{1}{b(t)}\Vert v^{*}(\theta_t) - v^{*}(\theta_{t+1})  \Vert^2.
\end{align*}
Taking summation of terms from indices $\tau_t$ to $t$ we have,
\begin{align}
    \hat{\lambda_e}\sum\limits_{k=\tau_t}^{t}\mathbb{E}\Vert z_k \Vert^2 &\leq \underbrace{\sum\limits_{k=\tau_t}^{t}\frac{1}{2b(k)}(\mathbb{E}\Vert z_k \Vert^2 - \mathbb{E}\Vert z_{k+1} \Vert^2)}_{I_1} + \underbrace{\sum\limits_{k=\tau_t}^{t}\mathbb{E}[\langle z_k ,  \delta_k\phi(s_k) - \mathbb{E}_{\theta_k}[\delta_k\phi(s_k)]\rangle]}_{I_2}\notag\\
    &\qquad + \underbrace{\sum\limits_{k=\tau_t}^{t}\frac{1}{b(k)}\mathbb{E}[\langle z_k , v^{*}(\theta_k) - v^{*}(\theta_{k+1}) 
    + (\nabla v_k^*)^T(\theta_{k+1} - \theta_k)\rangle]}_{I_3}\notag\\
    &\qquad+ \underbrace{\sum\limits_{k=\tau_t}^{t}\frac{1}{b(k)}\mathbb{E}[\langle z_k , (\nabla v_k^*)^T(\theta_{k} - \theta_{k+1}) \rangle]}_{I_4} +\underbrace{\sum\limits_{k=\tau_t}^{t}b(k)\mathbb{E}[\delta_k^2\Vert \phi(s_k)\Vert^2]}_{I_5}\notag \\
    &\qquad + \underbrace{\sum\limits_{k=\tau_t}^{t}\frac{1}{b(k)}\mathbb{E}\Vert v^{*}(\theta_k) - v^{*}(\theta_{k+1})  \Vert^2}_{I_6}.
\end{align}

For term $I_1$ we have,
\begin{align*}
&\sum\limits_{k=\tau_t}^{t}\frac{1}{2b(k)}(\mathbb{E}\Vert z_k \Vert^2 - E\Vert z_{k+1} \Vert^2) \\
& = \sum_{k=\tau_{t}}^{t} \bigg(\frac{1}{b(k)} 
    -
    \frac{1}{b(k-1)} 
    \bigg)
    \EE \Vert z_k \Vert^2 
    +
    \frac{1}{b(\tau_{t}-1)} 
    \EE \norm{z_{\tau_{t}}}^2 
    -
    \frac{1}{b(t)} 
    \EE \norm{z_{t+1}}^2 
    \\ 
    & \le 
     \sum_{k=\tau_{t}}^{t} \bigg(\frac{1}{b(k)} 
    -
    \frac{1}{b(k-1)} 
    \bigg)
    \EE \norm{z_k}^2 
    +
    \frac{1}{b(\tau_{t}-1)} 
    \EE \norm{z_{\tau_{t}}}^2 
    \\
    &= \mathcal{O}(1/b(t))\\
    &= \mathcal{O}( t^{\nu})
    \end{align*}

For term $I_2$ we have,
\begin{align*}
    I_2 = \mathcal{O}(\log^{2.5} t \cdot t^{1-\nu})
\end{align*}

The analysis of term $I_2$ is similar to analysis in lemma C.7  of \cite{wu2022finitetimeanalysistimescale}.

For term $I_3$ above, we have,
\begin{align*}
   &\sum\limits_{k=\tau_t}^{t}\frac{1}{b(k)}\mathbb{E}[\langle z_k , v^{*}(\theta_k) - v^{*}(\theta_{k+1}) 
    + (\nabla v_k^*)^T(\theta_{k+1} - \theta_k)\rangle]\\
    &\leq \frac{L_{m}}{2}\sum_{k=\tau_t}^{t}\frac{1}{b(k)}\mathbb{E}\Vert z_k\Vert \Vert \theta_{k+1} - \theta_k \Vert^2 \\
    &=\mathcal{O}(\sum_{k=\tau_t}^{t}\frac{a(k)^2}{b(k)}) \\
    &= \mathcal{O}(\log t \cdot t^{1-\nu})
\end{align*}

For term $I_4$ we have,

\begin{align*}
    &\sum\limits_{k=\tau_t}^{t}\frac{1}{b(k)}\mathbb{E}[\langle z_k , (\nabla v_k^*)^T(\theta_{k} - \theta_{k+1}) \rangle]\\
    &= -\sum_{k=\tau_t}^{t} \frac{1}{b(k)}\mathbb{E}\langle z_k , (\nabla v_k^*)^T a(k) \delta_k \nabla_{\theta} \log \pi_{\theta_{k}}(a_k|s_k) \rangle\\
    &= -\sum_{k=\tau_t}^{t} \frac{1}{b(k)}\mathbb{E}\langle z_k , (\nabla v_k^*)^T a(k) (r(s_k,a_k) - L_k + \phi(s_{k+1})^{\top} v_{k} - \phi(s_k)^{\top} v_{k}) \nabla_{\theta} \log \pi_{\theta_{k}}(a_k|s_k) \rangle\\
    &= -\sum_{k=\tau_t}^{t} \frac{1}{b(k)}\mathbb{E}\langle z_k , (\nabla v_k^*)^T a(k) (r(s_k,a_k) - L(\theta_k) + \phi(s_{k+1})^{\top} v_{k} - \phi(s_k)^{\top} v_{k}) \nabla_{\theta} \log \pi_{\theta_{k}}(a_k|s_k) \rangle\\
    &\qquad - \sum_{k=\tau_t}^{t} \frac{1}{b(k)}\mathbb{E}\langle z_k , (\nabla v_k^*)^T a(k) (L(\theta_k) - L_k) \nabla_{\theta} \log \pi_{\theta_{k}}(a_k|s_k) \rangle\\
    & = -\sum_{k=\tau_t}^{t} \frac{a(k)}{b(k)}\mathbb{E}\langle z_k , (\nabla v_k^*)^T (r(s_k,a_k) - L(\theta_k) + \phi(s_{k+1})^{\top} v_{k} - \phi(s_k)^{\top} v_{k}) \nabla_{\theta} \log \pi_{\theta_{k}}(a_k|s_k) \rangle\\
    &\qquad + \sum_{k=\tau_t}^{t} \frac{a(k)}{b(k)}\mathbb{E}\langle z_k , (\nabla v_k^*)^T \mathbb{E}_{\theta_k}[(r(s_k,a_k) - L(\theta_k) + \phi(s_{k+1})^{\top} v_{k} - \phi(s_k)^{\top} v_{k}) \nabla_{\theta} \log \pi_{\theta_{k}}(a_k|s_k) ]\rangle\\
    &\qquad -  \sum_{k=\tau_t}^{t} \frac{a(k)}{b(k)}\mathbb{E}\langle z_k , (\nabla v_k^*)^T \mathbb{E}_{\theta_k}[(r(s_k,a_k) - L(\theta_k) + \phi(s_{k+1})^{\top} v_{k} - \phi(s_k)^{\top} v_{k}) \nabla_{\theta} \log \pi_{\theta_{k}}(a_k|s_k) ]\rangle\\
    &\qquad - \sum_{k=\tau_t}^{t} \frac{a(k)}{b(k)}\mathbb{E}\langle z_k , (\nabla v_k^*)^T  (L(\theta_k) - L_k) \nabla_{\theta} \log \pi_{\theta_{k}}(a_k|s_k) \rangle\\
    &= \sum_{k=\tau_t}^{t}\mathbb{E}[\frac{a(k)}{b(k)}\Psi(O_k,v_k,\theta_k)]\\
    &\qquad-  \sum_{k=\tau_t}^{t} \frac{a(k)}{b(k)}\mathbb{E}\langle z_k , (\nabla v_k^*)^T \mathbb{E}_{\theta_k}[(r(s_k,a_k) - L(\theta_k) + \phi(s_{k+1})^{\top} v_{k} - \phi(s_k)^{\top} v_{k}) \nabla_{\theta} \log \pi_{\theta_{k}}(a_k|s_k) ]\rangle\\
    &\qquad - \sum_{k=\tau_t}^{t} \frac{a(k)}{b(k)}\mathbb{E}\langle z_k , (\nabla v_k^*)^T  (L(\theta_k) - L_k) \nabla_{\theta} \log \pi_{\theta_{k}}(a_k|s_k) \rangle\\
    &\leq \frac{c_a}{c_b} \log^{0.5} t \sum_{k=\tau_t}^{t}\vert \mathbb{E}[\Psi(O_k,v_k,\theta_k)]\vert + L_{*}\frac{c_a}{c_b}\log^{0.5} t\sqrt{\sum\limits_{k=\tau_t}^{t}\mathbb{E}\Vert z_k\Vert^2}\sqrt{\sum\limits_{k=\tau_t}^{t}\mathbb{E}[\Vert M(\theta_k,v_k)\Vert^2]}\\
    &\qquad + L_{*}B\frac{c_a}{c_b}\log^{0.5} t\sqrt{\sum\limits_{k=\tau_t}^{t}\mathbb{E}\Vert z_k\Vert^2}\sqrt{\sum\limits_{k=\tau_t}^{t}\mathbb{E}[(L(\theta_k)- L_k)^2]} \\
    &= \mathcal{O}(\log^{2.5} t \cdot t^{- \nu + 1}) + L_{*}\frac{c_a}{c_b}\log^{0.5} t\sqrt{\sum\limits_{k=\tau_t}^{t}\mathbb{E}\Vert z_k\Vert^2}\sqrt{\sum\limits_{k=\tau_t}^{t}\mathbb{E}[\Vert M(\theta_k,v_k)\Vert^2]}\\
    &\qquad + L_{*}B\frac{c_a}{c_b}\log^{0.5} t\sqrt{\sum\limits_{k=\tau_t}^{t}\mathbb{E}\Vert z_k\Vert^2}\sqrt{\sum\limits_{k=\tau_t}^{t}\mathbb{E}[(L(\theta_k)- L_k)^2]}.
    \end{align*}

    For the term $I_5$, we have,
\begin{align*}
    \sum\limits_{k=\tau_t}^{t}b(k)\mathbb{E}[\delta_k^2\Vert \phi(s_k)\Vert^2] = \mathcal{O}(t^{1-\nu}).
\end{align*}

Next, for the term $I_6$, we have,

\begin{align*}
   &\sum\limits_{k=\tau_t}^{t}\frac{1}{b(k)}\mathbb{E}\Vert v^{*}(\theta_k) - v^{*}(\theta_{k+1})  \Vert^2\\
   &= \mathcal{O}(\sum\limits_{k=\tau_t}^{t}\frac{a(k)^2}{b(k)}) \\
   &= \mathcal{O}(\log t \cdot t^{1 - \nu }). 
\end{align*}

For detailed analysis of terms $I_5$ and $I_6$ refer to the analysis of term $I_2$ in page number 22 of \cite{wu2022finitetimeanalysistimescale}.\\

After gathering all the terms we have,

\begin{align*}
     \hat{\lambda_e}\sum\limits_{k=\tau_t}^{t}\mathbb{E}\Vert z_k \Vert^2 &\leq \mathcal{O}( t^{\nu}) +  \mathcal{O}(\log^{2.5} t \cdot t^{1-\nu})  \\
     &\qquad +  L_{*}\frac{c_a}{c_b}\log^{0.5} t\sqrt{\sum\limits_{k=\tau_t}^{t}\mathbb{E}\Vert z_k\Vert^2}\sqrt{\sum\limits_{k=\tau_t}^{t}\mathbb{E}[\Vert M(\theta_k,v_k)\Vert^2]}\\
    &\qquad + L_{*}B\frac{c_a}{c_b}\log^{0.5} t\sqrt{\sum\limits_{k=\tau_t}^{t}\mathbb{E}\Vert z_k\Vert^2}\sqrt{\sum\limits_{k=\tau_t}^{t}\mathbb{E}[(L(\theta_k)- L_k)^2]}
\end{align*}

After applying the square technique we have,

\begin{align*}
    \sum\limits_{k=\tau_t}^{t}\mathbb{E}\Vert z_k \Vert^2  &= \mathcal{O}( t^{\nu}) +  \mathcal{O}(\log^{2.5} t \cdot t^{1-\nu}) \\
    &\qquad + \mathcal{O}(\log^{0.5} t \cdot \sum\limits_{k=\tau_t}^{t}\mathbb{E}[\Vert M(\theta_k,v_k)\Vert^2]) + \mathcal{O}(\log^{0.5} t \cdot \sum\limits_{k=\tau_t}^{t}\mathbb{E}[(L(\theta_k)- L_k)^2])\\
    &= \mathcal{O}( t^{\nu}) +  \mathcal{O}(\log^{3} t \cdot t^{1-\nu}).
\end{align*}

Assuming $t \geq 2\tau_{t} - 1$, we have,

\begin{align*}
    \frac{1}{1+t-\tau_t}\sum\limits_{k=\tau_t}^{t}\mathbb{E}\Vert z_k \Vert^2 = \mathcal{O}( t^{\nu - 1}) +  \mathcal{O}(\log^{3} t \cdot t^{-\nu})  
\end{align*}

Optimising over the values of $\nu$, we have $\nu = 0.5$ . Hence we have  the following :-
\begin{align*}
    \frac{1}{1+t-\tau_t}\sum_{k=\tau_t}^{t}\mathbb{E}\Vert z_k \Vert^2 &= \mathcal{O}(\log^3 t \cdot t^{ - 0.5})
\end{align*}

Therefore in order for the mean squared error of the critic to be upper bounded by $\epsilon$, namely,

\begin{align*}
     \frac{1}{1+t-\tau_t}\sum_{k=\tau_t}^{t}\mathbb{E}\Vert z_k \Vert^2  =  \mathcal{O}(\log^3 T \cdot T^{- 0.5}) \leq \epsilon,
\end{align*}
we need to set $T = \tilde{\mathcal{O}}(\epsilon^{-2})$.\\

\subsection{Finite-Time Analysis of Constrained Natural Critic-Actor with Modified Learning rates}

We now provide the complete detailed finite-time analysis of the C-NCAM algorithm. 
Please note that from this point onward we denote by $\phi(s) \in \mathbb{R}^{d_1}$ the feature vector associated with state $s$.

\subsubsection*{Convergence of Average Reward Estimate}\label{average_reward_convergence_modified_CCA}

Notations:-
\begin{align}
    \begin{split}
        O_t :&= (s_t , a_t , s_{t+1})\\
        y_t :&= (L_t - L(\theta_t,\gamma(t)))\\
        r_t :&= r(t) - \sum_{k=1}^{N}\gamma_k(t)(h_{k}(t)-\alpha_{k})\\
        M(\theta_t,v_t,\gamma(t)) :& = \mathbb{E}_{s_t \sim \mu_{\theta_t},a_t \sim \pi_{\theta_t},s_{t+1} \sim p}[( r_t- L(\theta_t,\gamma(t)) + \phi(s_{t+1})^{\top} v_{t} \\
        &\qquad- \phi(s_t)^{\top} v_{t})\nabla \log\pi_{\theta_t}(a_t|s_t)]\\
        N(O_t ,\theta_t,v_t,L_t,\gamma(t)) :& = ( r_t- L_t  + \phi(s_{t+1})^{\top} v_{t} - \phi(s_t)^{\top} v_{t})\nabla \log\pi_{\theta_t}(a_t|s_t)\\
        U_w :&=  2B(U_{v} +  \bar{U}_{v})\\
        G :&= 2B(U_r + U_v)
    \end{split}
\end{align}

We have , $\vert V^{\theta,\gamma}(s) \vert \leq \bar{U}_{v} ,\forall \theta \in \RR^{d},\forall s \in S$ and $\forall$ $\gamma \in \RR^N$, with $0 \leq \gamma_{i} \leq M$, where $i \in \{1,2,...,N\}$.

\subsubsection*{Proof of 
Theorem \ref{cost_convergence_2}:}

From the update rule of the reward estimation recursion in  Algorithm \ref{algo_ncca}, we have
\begin{align*}
    L_{t+1}-L(\theta_{t+1},\gamma(t+1))=L_t-L(\theta_t,\gamma(t))+L(\theta_t,\gamma(t))-L(\theta_{t+1},\gamma(t+1))+d(t)(r_t-L_t).
\end{align*}
We then have
\begin{align*}
    y_{t+1}^2=&\ (y_t+L(\theta_t,\gamma(t))-L(\theta_{t+1},\gamma(t+1))+d(t)(r_t-L_t))^2\\
    \leq &\ y_t^2+2y_t(L(\theta_t,\gamma(t))-L(\theta_{t+1},\gamma(t+1)))+2d(t) y_t(r_t-L_t)+2(L(\theta_t,\gamma(t))-L(\theta_{t+1},\gamma(t+1))^2\\
    &\qquad +2d(t)^2(r_t-L_t)^2\\
    = &\  (1-2d(t))y_t^2+2d(t)y_t(r_t-L(\theta_t))+2d(t)(L(\theta_t,\gamma(t))-L(\theta_{t+1},\gamma(t+1))\\
    &\qquad+2(L(\theta_t,\gamma(t))-L(\theta_{t+1},\gamma(t+1))^2
    +2d(t)^2(r_t-L_t)^2.
\end{align*}

Taking expectations, rearranging  and summing from $\tau_t$ to $t$ we obtain,

\begin{align}
    \sum\limits_{k=\tau_t}^{t} \mathbb{E}[y_k^2]\leq &\  \underbrace{\sum\limits_{t=\tau_t}^t\frac{1}{2d(k)}\mathbb{E}(y_k^2-y^2_{k+1})}_{I_1}+\underbrace{\sum\limits_{k=\tau_t}^{t}\mathbb{E}[y_k(r_k-L(\theta_k,\gamma(k)))]}_{I_2} \notag \\
    & +\underbrace{\sum\limits_{k=\tau_t}^{t}\frac{1}{d(k)}\mathbb{E}[y_k(L(\theta_k,\gamma(k))-L(\theta_{k+1},\gamma(k+1))]}_{I_3} \notag \\
    &  +\underbrace{\sum\limits_{k=\tau_t}^{t}\frac{1}{d(k)}\mathbb{E}[(L(\theta_k,\gamma(k))-L(\theta_{k+1},\gamma(k+1)))^2]}_{I_4}+\underbrace{\sum\limits_{k=\tau_t}^{t}d(k)\mathbb{E}[(r_k-L_k)^2]}_{I_5}.
\end{align}

For term $I_1$ we have,from Abel summation by parts, we have

\begin{align*}
    I_1=& \sum\limits_{k=\tau_t}^{t}\frac{1}{2d(k)}(y_k^2-y_{k+1}^2)\\
    =&\  \sum\limits_{k=\tau_t+1}^{t}y_k^2(\frac{1}{2d(k)}-\frac{1}{2d(k-1)})+\frac{1}{2d(\tau_t)}y_{\tau_t}^2-\frac{1}{d(t)}y_{t+1}^2\\
    \leq &\ \frac{2U_r^2}{d(t)}\\
    = &\  \frac{2}{c_d \cdot \ln ^{0.5} (t + 1)}U_r^2 (1+t)^{\nu}
\end{align*}

For detailed analysis of term $I_1$ kindly refer \cite{wu2022finitetimeanalysistimescale}.

We are assuming $\tau_t \geq 4$.
Now for term $I_2$ we can have the analysis similar to lemma 6 in \cite{panda_and_bhatnagar} and get,
\begin{align*}
    &\mathbb{E}[y_t(r_t-L(\theta_t,\gamma(t)))] \\
    &= \mathcal{O}(E\vert \gamma_p(t) - \gamma_p(t-\tau) \vert) +\mathcal{O}(E\Vert \theta_t - \theta_{t-\tau} \Vert) +\mathcal{O}(E\vert L_t - L_{t-\tau} \vert)\\
     &\qquad+\mathcal{O}(\sum_{i=t-\tau}^{t} E\Vert\theta_i - \theta_{t-\tau} \Vert)+\mathcal{O}(bk^{\tau - 1})
\end{align*}
where 
\begin{align*}
    \vert \gamma_p(t) - \gamma_p(t-\tau) \vert &= \max\limits_{i=1,2,...,N}\vert \gamma_i(t) - \gamma_i(t-\tau) \vert,\\
    t &\geq \tau \geq 0.
\end{align*}

Hence we   have,
\begin{align*}
   I_2 =  &\sum\limits_{k=\tau_t}^{t}\mathbb{E}[y_k(r_k-L(\theta_k,\gamma(k)))]\\
    &= \mathcal{O}(\tau_t^2 \sum\limits_{k=\tau_t}^{t} a(k-\tau_t))\\
    &= \mathcal{O}(\tau_t^2 \ln^{0.5}(t+1)\sum\limits_{k=\tau_t}^{t} \frac{1}{(1 + k)^{\nu}})\\
    & = \mathcal{O}(\log^{2.5} t \cdot t^{1-\nu})
\end{align*}

\begin{align*}
    I_3=& \sum\limits_{k=\tau_t}^{t}\frac{1}{d(k)}\mathbb{E}[y_k(L(\theta_k,\gamma(k))-L(\theta_{k+1},\gamma(k+1)))]\\
    \leq &\ \sum\limits_{k=\tau_t}^{t}\frac{1}{d(k)}\mathbb{E}[L_{J'}U_r\Vert \theta_k-\theta_{k+1}\Vert^2+|y_k|\Vert \theta_k-\theta_{k+1}\Vert \Vert M(\theta_k,v_k,\gamma(k))\Vert]\\
     &\qquad + \sum\limits_{k=\tau_t}^{t}\frac{1}{d(k)}\mathbb{E}[y_k\langle \mathbb{E}_{\theta_k}[(V^{\theta_{k},\gamma(k)}(s_{k+1}) - v(k)^T\phi(s_{k+1}) - V^{\theta_{k},\gamma(k)}(s_{k}) \\
     &\qquad+ v(k)^T\phi(s_{k}))\nabla \log\pi_{\theta_k}(a_k|s_k)] , \theta_k - \theta_{k+1} \rangle]\\
     &\qquad + \sum\limits_{k=\tau_t}^{t}\frac{1}{d(k)} \mathbb{E}[y_k(L(\theta_{k+1},\gamma(k)) - L(\theta_{k+1},\gamma(k+1)))]\\
    \leq &\ \sum\limits_{k=\tau_t}^{t}\mathbb{E}[L_{J'}U_rG^2\frac{a(k)^2}{d(k)} + G\frac{c_a {\log t}^{0.5}}{c_d}|y_k|\Vert M(\theta_k,v_k,\gamma(k))\Vert]\\
     &\qquad + \sum\limits_{k=\tau_t}^{t}\frac{1}{d(k)}\mathbb{E}[y_k\langle \mathbb{E}_{\theta_k}[(V^{\theta_{k},\gamma(k)}(s_{k+1}) - v(k)^T\phi(s_{k+1}) - V^{\theta_{k},\gamma(k)}(s_{k})\\
     &\qquad+ v(k)^T\phi(s_{k}))\nabla \log\pi_{\theta_k}(a_k|s_k)] , \theta_k - \theta_{k+1} \rangle]\\
     &\qquad + \sum\limits_{k=\tau_t}^{t}\frac{1}{d(k)} \mathbb{E}[y_k(L(\theta_{k+1},\gamma(k)) - L(\theta_{k+1},\gamma(k+1)))]\\
    \leq &\ \frac{2L_{J'}U_rG^2c_a^2 \ln^{0.5} (t+1)}{c_d}(1+t-\tau_t)^{1-\nu}+ G\frac{c_a }{c_d}(\sum\limits_{k=\tau_t}^{t}\mathbb{E}y_t^2)^{\frac{1}{2}}(\sum\limits_{k=\tau_t}^{t}\mathbb{E}\Vert M(\theta_k,v_k,\gamma(k))\Vert^2)^\frac{1}{2}\\
    &\qquad + \sum\limits_{k=\tau_t}^{t}\frac{1}{d(k)}\mathbb{E}[y_k\langle \mathbb{E}_{\theta_k}[(V^{\theta_{k},\gamma(k)}(s_{k+1}) - v(k)^T\phi(s_{k+1}) - V^{\theta_{k},\gamma(k)}(s_{k})\\
     &\qquad+ v(k)^T\phi(s_{k}))\nabla \log\pi_{\theta_k}(a_k|s_k)] , \theta_k - \theta_{k+1} \rangle]\\
     &\qquad + \sum\limits_{k=\tau_t}^{t}\frac{1}{d(k)} \mathbb{E}[y_k(L(\theta_{k+1},\gamma(k)) - L(\theta_{k+1},\gamma(k+1)))]\\
\end{align*}

Now,

\begin{align*}
    &\mathbb{E}_{\theta_t}[(V^{\theta_t,\gamma(t)}(s_{t+1}) -  \phi(s_{t+1})^Tv_t - V^{\theta_t,\gamma(t)}(s_t) +\phi(s_{t})^Tv_t)\nabla \log\pi_{\theta_t}(a_t|s_t)]\\
    =& \mathbb{E}_{\theta_t}[(V^{\theta_t,\gamma(t)}(s_{t+1}) -  \phi(s_{t+1})^Tv_t )\nabla \log\pi_{\theta_t}(a_t|s_t)] + \mathbb{E}_{\theta_t}[(-V^{\theta_t,\gamma(t)}(s_{t}) + \phi(s_{t})^Tv_t )\nabla \log\pi_{\theta_t}(a_t|s_t)]\\
    =& \mathbb{E}_{\theta_t}[(V^{\theta_t,\gamma(t)}(s_{t+1}) - \phi(s_{t+1})^Tv_t )\nabla \log\pi_{\theta_t}(a_t|s_t)] + \mathbb{E}_{s_t}[(-V^{\theta_t,\gamma(t)}(s_{t}) + \phi(s_{t})^Tv_t )\sum\limits_{a}\pi_{\theta_t}(a|s_t)\nabla \log\pi_{\theta_t}(a|s_t)]\\
    =& \mathbb{E}_{\theta_t}[(V^{\theta_t,\gamma(t)}(s_{t+1}) -  \phi(s_{t+1})^Tv_t )\nabla \log\pi_{\theta_t}(a_t|s_t)] + 0\\
\end{align*}

Let $ H(O_t, \theta_t,v_t,\gamma(t)) = (V^{\theta_t,\gamma(t)}(s_{t+1}) - \phi(s_{t+1})^Tv_t )\nabla \log\pi_{\theta_t}(a_t|s_t) $.
Now, 

\begin{align*}
    &\Vert \EE_{\theta_t}[H(O_t, \theta_t,v_t,\gamma(t))]\Vert\\
    \leq & \Vert \EE_{\theta_t}[H(O_t, \theta_t,v_t,\gamma(t))] - \EE_{s \sim \mu_{\theta_t},  a \sim \pi_{\theta_t}, s' \sim \mu_{\theta_t}}[H(O,\theta_t,v_t,\gamma(t))]\Vert + \Vert \EE_{s \sim \mu_{\theta_t},  a \sim \pi_{\theta_t}, s' \sim \mu_{\theta_t}}[H(O,\theta_t,v_t,\gamma(t))]\Vert\\
    =& \mathcal{O}(bk)
\end{align*}

Hence collecting all the terms, we have,
\begin{align*}
    I_3 = &\frac{2L_{J'}U_rG^2c_a^2 \ln^{0.5} (t+1)}{c_d}(1+t-\tau_t)^{1-\nu}+ G\frac{c_a }{c_d}(\sum\limits_{k=\tau_t}^{t}\mathbb{E}y_t^2)^{\frac{1}{2}}(\sum\limits_{k=\tau_t}^{t}\mathbb{E}\Vert M(\theta_k,v_k,\gamma(k))\Vert^2)^\frac{1}{2}\\
    &\qquad +  \mathcal{O}(t^{1 + \nu - \beta})\\
    &\qquad + \mathcal{O}\bigg( \bigg(\sum\limits_{k=\tau_t}^{t}\mathbb{E}y_t^2\bigg)^{\frac{1}{2}}\bigg(\sum\limits_{k=\tau_t}^{t}b^2\bigg)^\frac{1}{2}\bigg).
\end{align*}

For term $I_4$, we have
\begin{align*}
    I_4=& \sum\limits_{k=\tau_t}^{t}\frac{1}{d(k)}\mathbb{E}[(L(\theta_k,\gamma(k))-L(\theta_{k+1},\gamma(k+1)))^2]\\
    =& \mathcal{O}\bigg( \sum\limits_{k = \tau_t}^{t} \frac{a(k)^2}{d(k)}\bigg) + \mathcal{O}\bigg( \sum\limits_{k = \tau_t}^{t} \frac{c(k)^2}{d(k)}\bigg)\\
    = & \mathcal{O}(\ln^{0.5} t \cdot t^{1-\nu}).
\end{align*}

For term $I_5$, we have
\begin{align*}
    I_5=& \sum\limits_{k=\tau_t}^{t}d(k)\mathbb{E}[(r_k-L_k)^2]\\
    =& \mathcal{O}\bigg( \sum\limits_{k=\tau_t}^{t}d(k)\bigg)\\
    = & \mathcal{O}(\ln^{0.5} t \cdot t^{1-\nu}).
\end{align*}

For detailed analysis of terms $I_4$ and  $I_5$ refer to the analysis of term $I_2$ in page number 22 of \cite{wu2022finitetimeanalysistimescale}.\\

Hence putting together terms $I_1 - I_5$ we have,
\begin{align*}
    \sum\limits_{k=\tau_t}^{t} \mathbb{E}[y_k^2]&\leq \frac{2}{c_d \cdot \ln ^{0.5} (t + 1)}U_r^2 (1+t)^{\nu} + \mathcal{O}(\log^{2.5} t \cdot t^{1-\nu}) + \frac{2L_{J'}U_rG^2c_a^2 \ln^{0.5} (t+1)}{c_d}(1+t-\tau_t)^{1-\nu}\\
    &\qquad+ G\frac{c_a }{c_d}(\sum\limits_{k=\tau_t}^{t}\mathbb{E}y_t^2)^{\frac{1}{2}}(\sum\limits_{k=\tau_t}^{t}\mathbb{E}\Vert M(\theta_k,v_k,\gamma(k))\Vert^2)^\frac{1}{2} + \mathcal{O}(t^{1 + \nu - \beta})\\
    &\qquad + \mathcal{O}\bigg( \bigg(\sum\limits_{k=\tau_t}^{t}\mathbb{E}y_t^2\bigg)^{\frac{1}{2}}\bigg(\sum\limits_{k=\tau_t}^{t}b^2\bigg)^\frac{1}{2}\bigg)\\
    &\qquad +\mathcal{O}(\ln^{0.5} t \cdot t^{1-\nu})
\end{align*}

Hence, we obtain:

\begin{align*}
    \sum\limits_{k=\tau_t}^{t} \mathbb{E}[y_k^2] &\leq \mathcal{O}(\log^{-0.5} t \cdot  t^{\nu}) + \mathcal{O}(\log^{2.5} t \cdot t^{1-\nu}) + G\frac{c_a }{c_d}(\sum\limits_{k=\tau_t}^{t}\mathbb{E}y_t^2)^{\frac{1}{2}}(\sum\limits_{k=\tau_t}^{t}\mathbb{E}\Vert M(\theta_k,v_k,\gamma(k))\Vert^2)^\frac{1}{2}\\
    &\qquad + \mathcal{O}(t^{1 + \nu - \beta}) + \mathcal{O}\bigg( \bigg(\sum\limits_{k=\tau_t}^{t}\mathbb{E}y_t^2\bigg)^{\frac{1}{2}}\bigg(\sum\limits_{k=\tau_t}^{t}b^2\bigg)^\frac{1}{2}\bigg).
\end{align*}

After applying the squaring technique (see page 23 of \citep{wu2022finitetimeanalysistimescale}), we have,

\begin{align}\label{ineq1_CCA}
    \sum\limits_{k=\tau_t}^{t} \mathbb{E}[y_k^2] &\leq \mathcal{O}(\log^{-0.5} t \cdot  t^{\nu}) + \mathcal{O}(\log^{2.5} t \cdot t^{1-\nu})  + \mathcal{O}(t^{1 + \nu - \beta}) \notag\\
    &\qquad + 2G^{2}\frac{c_a^2}{c_d^2}\sum\limits_{k=\tau_t}^{t}\mathbb{E}\Vert M(\theta_k,v_k,\gamma(k))\Vert^2 + \mathcal{O}\bigg(\sum_{k = \tau_t}^{t}b^2 \bigg) .
\end{align}

In the above equation, the $\mathcal{O}(b^2)$ term represents the mixing-time error arising from the ergodicity of the Markov chain. So after ignoring this mixing-time error we have,

\begin{align*}
    \frac{1}{t - \tau_t + 1}\sum\limits_{k=\tau_t}^{t} \mathbb{E}[y_k^2] &\leq \mathcal{O}(\log^{-0.5} t \cdot  t^{\nu - 1}) + \mathcal{O}(\log^{2.5} t \cdot t^{-\nu})   
    + 2G^2\frac{c_a^2}{c_d^2} \frac{1}{t - \tau_t + 1}\sum\limits_{k=\tau_t}^{t}\mathbb{E}\Vert M(\theta_k,v_k)\Vert^2 \\
    &\qquad + \mathcal{O}(t^{\nu - \beta})
\end{align*}

\subsubsection*{Convergence of Actor}\label{actor_convergence_modified_CCA}
Notations used here:

\begin{align}
    \begin{split}
        O_t :&= (s_t , a_t , s_{t+1})\\
        r_t :&= r(t) - \sum_{k=1}^{N}\gamma_k(t)(h_{k}(t)-\alpha_{k})\\
        h(O_t,\theta_t,L_t,v_t,\gamma(t),G(t)) :&= ( r_t- L_t  + \phi(s_{t+1})^{\top} v_{t} - \phi(s_t)^{\top} v_{t})G(t)^{-1}\nabla \log\pi_{\theta_t}(a_t|s_t)\\
        I(O_t,L_t,\theta_t,v_t,\gamma(t),G(t)) :&= \langle \nabla L(\theta_t,\gamma(t)) ,h(O_t,\theta_t,L_t,v_t,\gamma(t),G(t))\\
        &\qquad - \mathbb{E}_{s_t \sim \mu_{\theta_t},a_t \sim \pi_{\theta_t},s_{t+1} \sim p}[h(O_t,\theta_t,L_t,v_t,\gamma(t),G(t))] \rangle  \\
        \bar{h}(O_t,\theta_t,v_t,\gamma(t),G(t)) :&= ( r_t- L(\theta_t,\gamma(t))  + \phi(s_{t+1})^{\top} v_{t}\\
        &\qquad - \phi(s_t)^{\top} v_{t})G(t)^{-1}\nabla \log\pi_{\theta_t}(a_t|s_t)\\
        \hat{h}(O_t,\theta_t,v_t,\gamma(t)) :&= ( r_t- L(\theta_t,\gamma(t))  + \phi(s_{t+1})^{\top} v_{t}\\
        &\qquad - \phi(s_t)^{\top} v_{t})\nabla \log\pi_{\theta_t}(a_t|s_t)\\
        M(\theta_t,v_t,\gamma(t),G(t)) :& = \mathbb{E}_{s_t \sim \mu_{\theta_t},a_t \sim \pi_{\theta_t},s_{t+1} \sim p}[\bar{h}(O_t,\theta_t,v_t,\gamma(t),G(t))]\\
        \bar{M}(\theta_t,v_t,\gamma(t)) :& = \mathbb{E}_{s_t \sim \mu_{\theta_t},a_t \sim \pi_{\theta_t},s_{t+1} \sim p}[\hat{h}(O_t,\theta_t,v_t,\gamma(t))].
    \end{split}
\end{align}

The following lemmas are needed for deriving finite time bounds while proving convergence of the actor recursion. 

\begin{lemma}\label{V_lipschitz_theta_CCA}
$\exists L_{v} > 0$ such that for any $s \in S$, and for any $\gamma \in \RR^N$,
    \begin{align*}
        \Vert V^{\theta_1,\gamma}(s) - V^{\theta_2,\gamma}(s) \Vert \leq L_{v}\Vert \theta_1 - \theta_2 \Vert , \forall \theta_1,\theta_2 \in \RR^{d}.
    \end{align*}
\end{lemma}

\begin{proof}
For finite state and action MDPs, under bounded costs and with  policy continuously differentiable and which selects each action with a positive probability, one can show that the stationary distribution and so the average reward is continuously differentiable and bounded, making it Lipschitz (see Theorem A.1, \cite{Panda_Bhatnagar_2025}) . The differential value function can be seen to then be differentiable from the fact that 
\[
\sum_{k=0}^K P_\pi^k(H^\pi - Q(\pi)) = V^\pi-P_\pi^{K+1}V^\pi,
\]
where $H^\pi = E[r+ \sum_{j=1}^{N} \gamma_j (h_j-\alpha_j)|X_0=i, \pi]$ and $Q(\pi) = J(\pi) + \sum_{j=1}^{N} \gamma_j (G_j(\pi) - \alpha_j)$, 
see p.284,\cite{Bertsekas2013OptimalControlVol2}, upon noting that $\lim_{K\rightarrow\infty} P_\pi^{K+1}V^\pi=0$.
\end{proof}

\begin{lemma}\label{V_lipschitz_gamma}
$\exists L_{w} > 0$ such that for any $s \in S$, for any $\theta \in \RR^{d}$, for all $\gamma(1),\gamma(2) \in R^N$ with $0 \leq \gamma_{i}(j) \leq M$, where $i \in \{1,2,...,N\}$, $j = 1,2$,
    \begin{align*}
        \Vert V^{\theta,\gamma(1)}(s) - V^{\theta,\gamma(2)}(s) \Vert \leq C\vert \gamma_m(1) - \gamma_m(2)\vert 
    \end{align*}
    where $\vert \gamma_m(1) - \gamma_m(2)\vert = \max\limits_{i=1,2,..,N}\vert \gamma_i(1) - \gamma_i(2)\vert$.
\end{lemma}

\begin{proof}
    Lemma \ref{V_lipschitz_gamma} can be easily seen to hold because the single-stage costs are all bounded, and $|V^{\pi,\gamma}(s)|<\infty$ for any $\pi, s, \gamma$.
\end{proof}
\subsubsection*{Proof of 
Theorem \ref{actor_convergence_2}:}

After applying Lemma 1 of \cite{panda_and_bhatnagar} to the update rule of the actor, we have,

\begin{align*}
    L(\theta_{t+1},\gamma(t)) &\geq L(\theta_t,\gamma(t)) + a(t) \langle \nabla L(\theta_t,\gamma(t)) ,\delta_{t}G(t)^{-1}\nabla \log\pi_{\theta_t}(a_t|s_t) \rangle\\ 
&\qquad - M_{L}a(t)^2\Vert\delta_{t}G(t)^{-1}\nabla \log\pi_{\theta_t}(a_t|s_t)\Vert^2.
\end{align*}

For the term $\langle \nabla L(\theta_t,\gamma(t)) ,\delta_{t}G(t)^{-1}\nabla \log\pi_{\theta_t}(a_t|s_t) \rangle $, we have,
\begin{align*}
    &\langle \nabla L(\theta_t,\gamma(t)) ,\delta_{t}G(t)^{-1}\nabla \log\pi_{\theta_t}(a_t|s_t) \rangle \\
    &= \langle \nabla L(\theta_t,\gamma(t)) ,( r_t - L_t + \phi(s_{t+1})^{\top} v_{t} - \phi(s_t)^{\top} v_{t})G(t)^{-1}\nabla \log\pi_{\theta_t}(a_t|s_t) \rangle\\
    & = I(O_t,\theta_t,L_t,v_t,\gamma(t),G(t))\\
    &\qquad + \langle \nabla L(\theta_t,\gamma(t)) , E_{s_t \sim \mu_{\theta_t},a_t \sim \pi_{\theta_t},s_{t+1} \sim p}[h(O_t,\theta_t,L_t,v_t,\gamma(t),G(t))] \rangle.
\end{align*}

Hence,
\begin{align}\label{inequality_L(theta)_2}
    & L(\theta_{t+1},\gamma(t))\\
    &\geq L(\theta_t,\gamma(t)) + a(t)I(O_t,\theta_t,L_t,v_t,\gamma(t),G(t)) + a(t)\langle \nabla L(\theta_t,\gamma(t)) , M(\theta_t,v_t,\gamma(t),G(t)) \rangle\notag\\
&\qquad +a(t)\langle \nabla L(\theta_t,\gamma(t)) ,  \mathbb{E}_{\theta_t}[(L(\theta_t) - L_t)G(t)^{-1}\nabla \log\pi_{\theta_t}(a_t|s_t)] \rangle \notag\\
&\qquad -M_{L}a(t)^2\Vert\delta_{t}G(t)^{-1}\nabla \log\pi_{\theta_t}(a_t|s_t)\Vert^2\notag\\
& = L(\theta_t,\gamma(t)) + a(t) I(O_t,\theta_t,L_t,v_t,\gamma(t),G(t)) + a(t)\langle \bar{M}(\theta_t,v_t,\gamma(t)) , M(\theta_t,v_t,\gamma(t),G(t))\rangle \notag \\
&\qquad+a(t)\langle   \mathbb{E}_{\theta_t}[(V^{\theta_t,\gamma(t)}(s_{t+1}) -  \phi(s_{t+1})^Tv_t - V^{\theta_t,\gamma(t)}(s_t) + \phi(s_{t})^Tv_t)\nabla \log\pi_{\theta_t}(a_t|s_t)] ,  M(\theta_t,v_t,\gamma(t),G(t))] \rangle\notag\\
&\qquad + a(t)\langle \nabla L(\theta_t,\gamma(t)),  \mathbb{E}_{\theta_t}[(L(\theta_t,\gamma(t)) - L_t)G(t)^{-1}\nabla \log\pi_{\theta_t}(a_t|s_t)] \rangle \notag \\
&\qquad- M_{L}a(t)^2\Vert\delta_{t}G(t)^{-1}\nabla \log\pi_{\theta_t}(a_t|s_t)\Vert^2.\notag
\end{align}

Now,

\begin{align*}
    &\mathbb{E}_{\theta_t}[(-V^{\theta_t,\gamma(t)}(s_{t+1}) +  \phi(s_{t+1})^Tv_t + V^{\theta_t,\gamma(t)}(s_t) - \phi(s_{t})^Tv_t)\nabla \log\pi_{\theta_t}(a_t|s_t)]\\
    =& \mathbb{E}_{\theta_t}[(-V^{\theta_t,\gamma(t)}(s_{t+1}) +  \phi(s_{t+1})^Tv_t )\nabla \log\pi_{\theta_t}(a_t|s_t)] + \mathbb{E}_{\theta_t}[(V^{\theta_t,\gamma(t)}(s_{t}) - \phi(s_{t})^Tv_t )\nabla \log\pi_{\theta_t}(a_t|s_t)]\\
    =& \mathbb{E}_{\theta_t}[(-V^{\theta_t,\gamma(t)}(s_{t+1}) +  \phi(s_{t+1})^Tv_t )\nabla \log\pi_{\theta_t}(a_t|s_t)] + \mathbb{E}_{s_t}[(V^{\theta_t,\gamma(t)}(s_{t}) - \phi(s_{t})^Tv_t )\sum\limits_{a}\pi_{\theta_t}(a|s_t)\nabla \log\pi_{\theta_t}(a|s_t)]\\
    =& \mathbb{E}_{\theta_t}[(-V^{\theta_t,\gamma(t)}(s_{t+1}) +  \phi(s_{t+1})^Tv_t )\nabla \log\pi_{\theta_t}(a_t|s_t)] + 0\\
\end{align*}

Let $ H(O_t, \theta_t,v_t,\gamma(t)) = (-V^{\theta_t,\gamma(t)}(s_{t+1}) + \phi(s_{t+1})^Tv_t )\nabla \log\pi_{\theta_t}(a_t|s_t) $.
Now, 

\begin{align*}
    &\Vert \EE_{\theta_t}[H(O_t, \theta_t,v_t,\gamma(t))]\Vert\\
    \leq & \Vert \EE_{\theta_t}[H(O_t, \theta_t,v_t,\gamma(t))] - \EE_{s \sim \mu_{\theta_t},  a \sim \pi_{\theta_t}, s' \sim \mu_{\theta_t}}[H(O,\theta_t,v_t,\gamma(t))]\Vert + \Vert \EE_{s \sim \mu_{\theta_t},  a \sim \pi_{\theta_t}, s' \sim \mu_{\theta_t}}[H(O,\theta_t,v_t,\gamma(t))]\Vert\\
    =& \mathcal{O}(bk)
\end{align*}

Looking back at inequality (\ref{inequality_L(theta)_2}), we have the following:

\begin{align*}
    &\lambda_G\sum\limits_{k=\tau_t}^{t}\mathbb{E}\Vert \bar{M}(\theta_k,v_k,\gamma(k)) \Vert^2 \\
&\leq \underbrace{\sum\limits_{k=\tau_t}^{t}\frac{1}{a(k)}\mathbb{E}[(L(\theta_{k+1},\gamma(k)) -L(\theta_{k},\gamma(k))  )]}_{I_1} - \underbrace{\sum\limits_{k=\tau_t}^{t}\mathbb{E}[I(O_k,\theta_k,L_k,v_k,\gamma(k),G(k))]}_{I_2} \\
&\qquad + \underbrace{\sum\limits_{k=\tau_t}^{t} \EE\langle \EE_{\theta_k}[H(O_k, \theta_k,v_k,\gamma(k))], M(\theta_k,v_k,\gamma(k), G(k))\rangle]}_{I_3} \\
&\qquad \underbrace{-\sum\limits_{k=\tau_t}^{t}\mathbb{E}[\langle \nabla L(\theta_k,\gamma(k)) , \mathbb{E}_{\theta_k}[(L(\theta_k,\gamma(k)) - L_k)G(k)^{-1}\nabla \log\pi_{\theta_k}(a_k|s_k)] \rangle]}_{I_4} \\
&\qquad +\underbrace{M_{L}\sum\limits_{k=\tau_t}^{t}a(k)\mathbb{E}[\Vert\delta_{k}G(k)^{-1}\nabla \log\pi_{\theta_k}(a_k|s_k)\Vert^2]}_{I_5}.
\end{align*}

Now, for term $I_1$ we have,
\begin{align*}
    &\sum\limits_{k=\tau_t}^{t}\frac{1}{a(k)}\mathbb{E}[(L(\theta_{k+1},\gamma(k)) -L(\theta_{k},\gamma(k))  )]\\
    &= \mathcal{O}(1/a(t)) + \mathcal{O}(\sum\limits_{k=\tau_t}^{t}\frac{c(k)}{a(k)})\\
    &= \mathcal{O}((\log t)^{-0.5}\cdot t^{\nu}) + \mathcal{O}((\log t)^{-0.5}\cdot t^{1 +\nu-\beta})
\end{align*}

The analysis of term $I_1$ is similar to the analysis of term $I_1$ in the proof of average reward estimate  convergence above.

We are assuming $\tau_t \geq 4$.

For term $I_2$, we have,
\begin{align*}
    I_2 = \mathcal{O}((\log t)^{2.5}\cdot t^{1-\nu})
\end{align*}

For term $I_3$ we have,
\begin{align*}
    &\sum\limits_{k=\tau_t}^{t} \EE\langle \EE_{\theta_k}[H(O_k, \theta_k,v_k,\gamma(k))], M(\theta_k,v_k,\gamma(k), G(k))\rangle]\\
    \leq & \sum\limits_{k=\tau_t}^{t} \sqrt{\EE \Vert \EE_{\theta_k}[H(O_k, \theta_k,v_k,\gamma(k))] \Vert^2}\sqrt{\EE \Vert M(\theta_k,v_k,\gamma(k), G(k))\Vert^2}\\
    \leq& U_{G} \sum\limits_{k=\tau_t}^{t} \sqrt{\EE \Vert \EE_{\theta_k}[H(O_k, \theta_k,v_k,\gamma(k))] \Vert^2}\sqrt{\EE \Vert \bar{M}(\theta_k,v_k,\gamma(k))\Vert^2}
\end{align*}

For term $I_4$ we have,

\begin{align*}
    -&\sum\limits_{k=\tau_t}^{t}\mathbb{E}[\langle \nabla L(\theta_k,\gamma(k)) ,\mathbb{E}_{\theta_k}[(L(\theta_k,\gamma(k)) - L_k)G(k)^{-1}\nabla \log\pi_{\theta_k}(a_k|s_k)] \rangle]\\
    =&\ \sum\limits_{k=\tau_t}^{t}\mathbb{E}[\langle \mathbb{E}_{\theta_k}[(r(s,a,\gamma(k)) - L(\theta_k,\gamma(k)) + V^{\theta_k,\gamma(k)}(s^{'}) - V^{\theta_k,\gamma(k)}(s))\nabla \log \pi_{\theta_k}(a|s)] \\
    &\qquad \qquad \qquad, ( L_k - L(\theta_k,\gamma(k)))G(k)^{-1}\mathbb{E}_{\theta_k}[\nabla \log\pi_{\theta_k}(a_k|s_k)]] \rangle]\notag\\
    = &\ \sum\limits_{k=\tau_t}^{t}E[\langle \mathbb{E}_{\theta_k}[(r(s,a,\gamma(k)) - L(\theta_k,\gamma(k)) + (\phi(s^{'}) - \phi(s))^Tv(k))\nabla \log \pi_{\theta_k}(a|s)]\\
    &\qquad, ( L_k - L(\theta_k,\gamma(k)))G(k)^{-1}\mathbb{E}_{\theta_k}[\nabla \log\pi_{\theta_k}(a_k|s_k)]] \rangle]\notag\\
    &+ \sum\limits_{k=\tau_t}^{t}\mathbb{E}[\langle\mathbb{E}_{\theta_k}[( V^{\theta_k,\gamma(k)}(s^{'})- \phi(s^{'})^Tv_{k} + \phi(s)^Tv_{k} - V^{\theta_k,\gamma(k)}(s))\nabla \log \pi_{\theta_k}(a|s)] \\
    &\qquad \qquad, ( L_k - L(\theta_k,\gamma(k)))G(k)^{-1}\mathbb{E}_{\theta_k}[\nabla \log\pi_{\theta_k}(a_k|s_k)]] \rangle]\notag\\
    \leq  &\ BU_{G}\sqrt{\sum\limits_{k=\tau_t}^{t}\mathbb{E}\Vert \bar{M}(\theta_k,v_k,\gamma(k))\Vert^2}\sqrt{\sum\limits_{k=\tau_t}^{t}\mathbb{E}\vert L_k - L(\theta_k,\gamma(k)) \vert^2} + I_{4a}.
\end{align*}

where 
\begin{align*}
    I_{4a} = &\sum\limits_{k=\tau_t}^{t}\mathbb{E}[\langle \mathbb{E}_{\theta_k}[( V^{\theta_k,\gamma(k)}(s^{'})- \phi(s^{'})^Tv_{k} + \phi(s)^Tv_{k} - V^{\theta_k,\gamma(k)}(s))\nabla \log \pi_{\theta_k}(a|s)] \\
    &\qquad \qquad, ( L_k - L(\theta_k,\gamma(k)))G(k)^{-1}\mathbb{E}_{\theta_k}[\nabla \log\pi_{\theta_k}(a_k|s_k)]] \rangle]\\
     =& \mathcal{O}\bigg( \sum\limits_{l=\tau_t}^{t} bk\bigg)
\end{align*}

Hence, putting all these results back in , we obtain,

\begin{align*}
    I_4 \leq  &\ BU_{G}\sqrt{\sum\limits_{k=\tau_t}^{t}\mathbb{E}\Vert \bar{M}(\theta_k,v_k,\gamma(k))\Vert^2}\sqrt{\sum\limits_{k=\tau_t}^{t}\mathbb{E}\vert L_k - L(\theta_k,\gamma(k)) \vert^2} + \mathcal{O}\bigg( \sum\limits_{l=\tau_t}^{t} bk\bigg).
\end{align*}

For term $I_5$, we have,

\begin{align*}
  &M_{L}\sum\limits_{k=\tau_t}^{t}a(k)\mathbb{E}[\Vert\delta_{k}G(k)^{-1}\nabla \log\pi_{\theta_k}(a_k|s_k)\Vert^2] \\
    & = \mathcal{O}(\log ^{0.5} t \cdot t^{1-\nu}).
\end{align*}

Now, gathering all the terms we have,

\begin{align*}
    \sum\limits_{k=\tau_t}^{t}\mathbb{E}\Vert \bar{M}(\theta_k,v_k,\gamma(k)) \Vert^2 &\leq \mathcal{O}((\log t)^{-0.5}\cdot t^{\nu}) + \mathcal{O}((\log t)^{-0.5}\cdot t^{1 +\nu-\beta}) +  \mathcal{O}(\log ^{2.5} t \cdot t^{1-\nu})\\
    &\qquad + \frac{U_{G}}{\lambda_G} \sum\limits_{k=\tau_t}^{t} \sqrt{\EE \Vert \EE_{\theta_k}[H(O_k, \theta_k,v_k,\gamma(k))] \Vert^2}\sqrt{\EE \Vert \bar{M}(\theta_k,v_k,\gamma(k))\Vert^2}\\
    &\qquad + \frac{BU_{G}}{\lambda_G}\sqrt{\sum\limits_{k=\tau_t}^{t}\mathbb{E}\Vert \bar{M}(\theta_k,v_k,\gamma(k))\Vert^2}\sqrt{\sum\limits_{k=\tau_t}^{t}\mathbb{E}\vert L_k - L(\theta_k,\gamma(k)) \vert^2}\\
    &\qquad + \mathcal{O}\bigg( \sum\limits_{l=\tau_t}^{t} bk\bigg).
\end{align*}

After applying the squaring technique we have,

\begin{align*}
    \sum\limits_{k=\tau_t}^{t}\mathbb{E}\Vert \bar{M}(\theta_k,v_k,\gamma(k)) \Vert^2& \leq \mathcal{O}((\log t)^{-0.5}\cdot t^{\nu}) +  \mathcal{O}(\log ^{2.5} t \cdot t^{1-\nu}) + \mathcal{O}((\log t)^{-0.5}\cdot t^{1 +\nu-\beta}) \\
    &\qquad + 2\frac{B^2U_{G}^2}{\lambda_G^2}\sum\limits_{k=\tau_t}^{t}\mathbb{E}\vert L_k - L(\theta_k,\gamma(k)) \vert^2 + 2 \sum\limits_{k=\tau_t}^{t}\EE \Vert \EE_{\theta_k}[H(O_k, \theta_k,v_k,\gamma(k))] \Vert^2\\
    &\qquad + \mathcal{O}\bigg( \sum\limits_{l=\tau_t}^{t} bk\bigg)\\
    &\leq \mathcal{O}((\log t)^{-0.5}\cdot t^{\nu}) +  \mathcal{O}(\log ^{2.5} t \cdot t^{1-\nu}) + \mathcal{O}( t^{1 +\nu-\beta}) + \mathcal{O}\bigg( \sum\limits_{l=\tau_t}^{t} b^2\bigg)\\
    &\qquad + 4\frac{B^2U_{G}^2}{\lambda_G^2}G^{2}\frac{c_a^2}{c_d^2}\sum\limits_{k=\tau_t}^{t}\mathbb{E}\Vert \bar{M}(\theta_k,v_k,\gamma(k))\Vert^2
\end{align*}

The last inequality follows from (\ref{ineq1_CCA}).

Now if we select the values for $c_a$ and $c_d$ such that $4\frac{B^2U_{G}^2}{\lambda_G^2}G^{2}\frac{c_a^2}{c_d^2}< 1$, we shall obtain,
 \begin{align*}
     \sum\limits_{k=\tau_t}^{t}\mathbb{E}\Vert \bar{M}(\theta_k,v_k,\gamma(k)) \Vert^2 &= \mathcal{O}((\log t)^{-0.5}\cdot t^{\nu}) +  \mathcal{O}(\log ^{2.5} t \cdot t^{1-\nu}) + \mathcal{O}(t^{1 + \nu - \beta}) + \mathcal{O}\bigg( \sum\limits_{l=\tau_t}^{t} b^2\bigg).
 \end{align*}
 Dividing by $(1 + t - \tau_t)$  and assuming $t \geq 2\tau_t + 1$, we have,
 \begin{align}\label{ineq2_CCA}
     \frac{1}{(1 + t - \tau_t)}\sum\limits_{k=\tau_t}^{t}\mathbb{E}\Vert \bar{M}(\theta_k,v_k,\gamma(k))\Vert^2 &= \mathcal{O}((\log t)^{-0.5} \cdot t^{\nu - 1}) + \mathcal{O}(\log^{2.5} t \cdot t^{-\nu}) + \mathcal{O}(t^{ \nu - \beta}) + \mathcal{O}(b^2).
 \end{align}

 In the above equation, the $\mathcal{O}(b^2)$ term represents the mixing-time error arising from the ergodicity of the Markov chain. So after ignoring this mixing-time error we have,
 
\begin{align*}
     \frac{1}{(1 + t - \tau_t)}\sum\limits_{k=\tau_t}^{t}\mathbb{E}\Vert \bar{M}(\theta_k,v_k,\gamma(k))\Vert^2 &= \mathcal{O}((\log t)^{-0.5} \cdot t^{\nu - 1}) + \mathcal{O}(\log^{2.5} t \cdot t^{-\nu}) + \mathcal{O}(t^{ \nu - \beta}) .
 \end{align*}

\subsubsection*{Convergence of the Critic:}

Recall that we have the following update rule for the critic:
\begin{align*}
    v_{n+1} = \Gamma(v_{n} + b(n)\delta_{n}f_{s_n}).
\end{align*}

Notations:
\begin{align}
    \begin{split}
         O_t :&= (s_t , a_t , s_{t+1})\\
        z_t &:= v_t - v^{*}(\theta_t,\gamma(t))\\
        g(O_t,v_t,\theta_t,\gamma(t)) &:= (r_t - L(\theta_t,\gamma(t)) + \phi(s_{t+1})^{\top} v_{t} - \phi(s_t)^{\top} v_{t})\phi(s_t)\\
        \bar{g}(v_t,\theta_t,\gamma(t)) &:= \mathbb{E}_{s_t \sim \mu_{\theta_t},a_t \sim \pi_{\theta_t},s_{t+1} \sim p(.|s_t,a_t)}[(r_t - L(\theta_t,\gamma(t)) + \phi(s_{t+1})^{\top} v_{t} - \phi(s_t)^{\top} v_{t})\phi(s_t)]\\
        \bar{Q}(O_t,v_t,\theta_t,\gamma(t)) &:= \langle z_t , g(O_t,v_t,\theta_t,\gamma(t)) - \bar{g}(v_t,\theta_t,\gamma(t))\rangle \\
        \bar{U}(O_t,v_t,\theta_t,\gamma(t),G(k)) &:= (\nabla v_t^*)^T (r_t - L(\theta_t,\gamma(t)) + \phi(s_{t+1})^{\top} v_{t} - \phi(s_t)^{\top} v_{t}) G(k)^{-1}\nabla_{\theta} \log \pi_{\theta_{t}}(a_t|s_t)\\
        \Psi(O_t,v_t,\theta_t,\gamma(t),G(k)) &:= \langle z_t ,   \mathbb{E}_{\theta_t}[\bar{U}(O_t,v_t,\theta_t,\gamma(t),G(k))] - \bar{U}(O_t,v_t,\theta_t,\gamma(t),G(k))\rangle.
    \end{split}
\end{align}

\subsubsection*{Proof of 
Theorem \ref{critic_convergence_2}:}
From the critic update rule, we have,
\begin{align*}
    \Vert z_{t+1} \Vert^2 &= \Vert  v_{t+1} - v^{*}(\theta_{t+1},\gamma(t+1))\Vert^2\\
    &=\Vert \Gamma(v_t +b(t)\delta_t\phi(s_t)) - v^{*}(\theta_{t+1},\gamma(t+1))  \Vert^2\\
    &\leq \Vert v_t + b(t)\delta_t\phi(s_t) - v^{*}(\theta_{t+1},\gamma(t+1)) \Vert^2\\
    &= \Vert z_t + b(t)\delta_t\phi(s_t)  + v^{*}(\theta_t,\gamma(t))- v^{*}(\theta_{t+1},\gamma(t+1)) \Vert^2\\
    &\leq  \Vert z_t \Vert^2 + 2b(t)\langle z_t ,  \delta_t\phi(s_t) \rangle + 2\langle z_t , v^{*}(\theta_t,\gamma(t)) - v^{*}(\theta_{t+1},\gamma(t+1)) \rangle + 2b(t)^2\delta_t^2\Vert \phi(s_t)\Vert^2 \\
    &\qquad + 2\Vert v^{*}(\theta_t,\gamma(t)) - v^{*}(\theta_{t+1},\gamma(t+1))  \Vert^2\\
    &=\Vert z_t \Vert^2 + 2b(t)\langle z_t ,  \delta_t\phi(s_t) -  \mathbb{E}_{\theta_t}[\delta_t\phi(s_t)]\rangle +
    2b(t)\langle z_t ,    \mathbb{E}_{\theta_t}[\delta_t\phi(s_t)]\rangle\\
    &\qquad +
    2\langle z_t , v^{*}(\theta_t,\gamma(t)) - v^{*}(\theta_{t+1},\gamma(t+1)) \rangle
    + 2b(t)^2\delta_t^2\Vert \phi(s_t)\Vert^2 + 2\Vert v^{*}(\theta_t,\gamma(t)) - v^{*}(\theta_{t+1},\gamma(t+1))  \Vert^2\\
    &\leq \Vert z_t \Vert^2 + 2b(t)\langle z_t ,  \delta_t\phi(s_t) -  \mathbb{E}_{\theta_t}[\delta_t\phi(s_t)]\rangle -
    2b(t)\hat{\lambda_e}\Vert z_t\Vert^2 +
    2\langle z_t , v^{*}(\theta_t,\gamma(t)) - v^{*}(\theta_{t+1},\gamma(t+1)) \rangle\\
    &\qquad+ 2b(t)^2\delta_t^2\Vert \phi(s_t)\Vert^2 + 2\Vert v^{*}(\theta_t,\gamma(t)) - v^{*}(\theta_{t+1},\gamma(t+1))  \Vert^2.
\end{align*}

The final inequality follows from Assumption 2.

After rearranging the terms we obtain,
\begin{align*}
    \hat{\lambda_e}\Vert z_t \Vert^2 &\leq \frac{1}{2b(t)}(\Vert z_t \Vert^2 - \Vert z_{t+1} \Vert^2) + \langle z_t ,  \delta_t\phi(s_t) -  \mathbb{E}_{\theta_t}[\delta_t\phi(s_t)]\rangle + \frac{1}{b(t)}\langle z_t , v^{*}(\theta_t,\gamma(t)) - v^{*}(\theta_{t+1},\gamma(t+1)) \\
    &\qquad + (\nabla v_t^*)^T(\theta_{t+1} - \theta_t)\rangle
    + \frac{1}{b(t)}\langle z_t , (\nabla v_t^*)^T(\theta_{t} - \theta_{t+1}) \rangle +b(t)\delta_t^2\Vert \phi(s_t)\Vert^2 \\
    &\qquad + \frac{1}{b(t)}\Vert v^{*}(\theta_t,\gamma(t)) - v^{*}(\theta_{t+1},\gamma(t+1))  \Vert^2.
\end{align*}
Taking summation of terms from indices $\tau_t$ to $t$ we have,
\begin{align}
    \hat{\lambda_e}\sum\limits_{k=\tau_t}^{t}\mathbb{E}\Vert z_k \Vert^2 &\leq \underbrace{\sum\limits_{k=\tau_t}^{t}\frac{1}{2b(k)}(\mathbb{E}\Vert z_k \Vert^2 - \mathbb{E}\Vert z_{k+1} \Vert^2)}_{I_1} + \underbrace{\sum\limits_{k=\tau_t}^{t}\mathbb{E}[\langle z_k ,  \delta_k\phi(s_k) - \mathbb{E}_{\theta_k}[\delta_k\phi(s_k)]\rangle]}_{I_2}\notag\\
    &\qquad + \underbrace{\sum\limits_{k=\tau_t}^{t}\frac{1}{b(k)}\mathbb{E}[\langle z_k , v^{*}(\theta_k,\gamma(k)) - v^{*}(\theta_{k+1},\gamma(k+1)) 
    + (\nabla v_k^*)^T(\theta_{k+1} - \theta_k)\rangle]}_{I_3}\notag\\
    &\qquad+ \underbrace{\sum\limits_{k=\tau_t}^{t}\frac{1}{b(k)}\mathbb{E}[\langle z_k , (\nabla v_k^*)^T(\theta_{k} - \theta_{k+1}) \rangle]}_{I_4} +\underbrace{\sum\limits_{k=\tau_t}^{t}b(k)\mathbb{E}[\delta_k^2\Vert \phi(s_k)\Vert^2]}_{I_5}\notag \\
    &\qquad + \underbrace{\sum\limits_{k=\tau_t}^{t}\frac{1}{b(k)}\mathbb{E}\Vert v^{*}(\theta_k,\gamma(k)) - v^{*}(\theta_{k+1},\gamma(k+1))  \Vert^2}_{I_6}.
\end{align}

For term $I_1$ we have,
\begin{align*}
    &\sum\limits_{k=\tau_t}^{t}\frac{1}{2b(k)}(\mathbb{E}\Vert z_k \Vert^2 - E\Vert z_{k+1} \Vert^2) \\
     & = 
    \sum_{k=\tau_{t}}^{t} \bigg(\frac{1}{b(k)} 
    -
    \frac{1}{b(k-1)} 
    \bigg)
    \EE \norm{z_k}^2 
    +
    \frac{1}{b(\tau_{t}-1)} 
    \EE \norm{z_{\tau_{t}}}^2 
    -
    \frac{1}{b(t)} 
    \EE \norm{z_{t+1}}^2 
    \\ 
    & \le 
     \sum_{k=\tau_{t}}^{t} \bigg(\frac{1}{b(k)} 
    -
    \frac{1}{b(k-1)} 
    \bigg)
    \EE \norm{z_k}^2 
    +
    \frac{1}{b(\tau_{t}-1)} 
    \EE \norm{z_{\tau_{t}}}^2 
    \\
    &= \mathcal{O}(1/b(t))\\
    &= \mathcal{O}( t^{\nu})
    \end{align*}

For term $I_2$ we have,
\begin{align*}
    I_2 = \mathcal{O}(\log^{2.5} t \cdot t^{1-\nu})
\end{align*}

The analysis of term $I_2$ is similar to analysis of term $I_2$ in critic convergence proof of two-timescale critic actor algorithm before.\\
For term $I_3$ above, we have,
\begin{align*}
   &\sum\limits_{k=\tau_t}^{t}\frac{1}{b(k)}\mathbb{E}[\langle z_k , v^{*}(\theta_k,\gamma(k)) - v^{*}(\theta_{k+1},\gamma(k+1)) 
    + (\nabla v_k^*)^T(\theta_{k+1} - \theta_k)\rangle]\\
    &= \sum\limits_{k=\tau_t}^{t}\frac{1}{b(k)}\mathbb{E}[\langle z_k , v^{*}(\theta_k,\gamma(k)) - v^{*}(\theta_{k+1},\gamma(k))  + (\nabla v_k^*)^T(\theta_{k+1} - \theta_k)\rangle]\\
    &\qquad + \sum\limits_{k=\tau_t}^{t}\frac{1}{b(k)}\mathbb{E}[\langle z_k , v^{*}(\theta_{k+1},\gamma(k)) - v^{*}(\theta_{k+1},\gamma(k+1))\rangle]\\
    &\leq \frac{L_{m}}{2}\sum_{k=\tau_t}^{t}\frac{1}{b(k)}\mathbb{E}\Vert z_k\Vert \Vert \theta_{k+1} - \theta_k \Vert^2 + \sum\limits_{k=\tau_t}^{t}\frac{1}{b(k)}\mathbb{E}[\langle z_k , v^{*}(\theta_{k+1},\gamma(k)) - v^{*}(\theta_{k+1},\gamma(k+1))\rangle]\\
    &=\mathcal{O}(\sum_{k=\tau_t}^{t}\frac{a(k)^2}{b(k)}) + \mathcal{O}(\sum_{k=\tau_t}^{t}\frac{c(k)}{b(k)})\\
    &= \mathcal{O}(\log t \cdot t^{1-\nu}) + \mathcal{O}(t^{\nu - \beta + 1})
\end{align*}

For term $I_4$ we have,

\begin{align*}
    &\sum\limits_{k=\tau_t}^{t}\frac{1}{b(k)}\mathbb{E}[\langle z_k , (\nabla v_k^*)^T(\theta_{k} - \theta_{k+1}) \rangle]\\
    &= -\sum_{k=\tau_t}^{t} \frac{1}{b(k)}\mathbb{E}\langle z_k , (\nabla v_k^*)^T a(k) \delta_k G(k)^{-1}\nabla_{\theta} \log \pi_{\theta_{k}}(a_k|s_k) \rangle\\
    &= -\sum_{k=\tau_t}^{t} \frac{1}{b(k)}\mathbb{E}\langle z_k , (\nabla v_k^*)^T a(k) (r_k - L_k + \phi(s_{k+1})^{\top} v_{k} - \phi(s_k)^{\top} v_{k})G(k)^{-1} \nabla_{\theta} \log \pi_{\theta_{k}}(a_k|s_k) \rangle\\
    &= -\sum_{k=\tau_t}^{t} \frac{1}{b(k)}\mathbb{E}\langle z_k , (\nabla v_k^*)^T a(k) (r_k - L(\theta_k,\gamma(k)) + \phi(s_{k+1})^{\top} v_{k} - \phi(s_k)^{\top} v_{k}) G(k)^{-1}\nabla_{\theta} \log \pi_{\theta_{k}}(a_k|s_k) \rangle\\
    &\qquad - \sum_{k=\tau_t}^{t} \frac{1}{b(k)}\mathbb{E}\langle z_k , (\nabla v_k^*)^T a(k) (L(\theta_k,\gamma(k)) - L_k) G(k)^{-1}\nabla_{\theta} \log \pi_{\theta_{k}}(a_k|s_k) \rangle\\
    & = -\sum_{k=\tau_t}^{t} \frac{a(k)}{b(k)}\mathbb{E}\langle z_k , (\nabla v_k^*)^T (r_k - L(\theta_k,\gamma(k)) + \phi(s_{k+1})^{\top} v_{k} - \phi(s_k)^{\top} v_{k}) G(k)^{-1}\nabla_{\theta} \log \pi_{\theta_{k}}(a_k|s_k) \rangle\\
    &\qquad + \sum_{k=\tau_t}^{t} \frac{a(k)}{b(k)}\mathbb{E}\langle z_k , (\nabla v_k^*)^T \mathbb{E}_{\theta_k}[(r_k - L(\theta_k,\gamma(k)) + \phi(s_{k+1})^{\top} v_{k} - \phi(s_k)^{\top} v_{k}) G(k)^{-1}\nabla_{\theta} \log \pi_{\theta_{k}}(a_k|s_k) ]\rangle\\
    &\qquad -  \sum_{k=\tau_t}^{t} \frac{a(k)}{b(k)}\mathbb{E}\langle z_k , (\nabla v_k^*)^T \mathbb{E}_{\theta_k}[(r_k - L(\theta_k,\gamma(k)) + \phi(s_{k+1})^{\top} v_{k} - \phi(s_k)^{\top} v_{k}) G(k)^{-1}\nabla_{\theta} \log \pi_{\theta_{k}}(a_k|s_k) ]\rangle\\
    &\qquad - \sum_{k=\tau_t}^{t} \frac{a(k)}{b(k)}\mathbb{E}\langle z_k , (\nabla v_k^*)^T  (L(\theta_k,\gamma(k)) - L_k) G(k)^{-1}\nabla_{\theta} \log \pi_{\theta_{k}}(a_k|s_k) \rangle\\
    &= \sum_{k=\tau_t}^{t}\mathbb{E}[\frac{a(k)}{b(k)}\Psi(O_k,v_k,\theta_k,\gamma(k),G(k))]\\
    &\qquad-  \sum_{k=\tau_t}^{t} \frac{a(k)}{b(k)}\mathbb{E}\langle z_k , (\nabla v_k^*)^T \mathbb{E}_{\theta_k}[(r_k - L(\theta_k,\gamma(k)) + \phi(s_{k+1})^{\top} v_{k} - \phi(s_k)^{\top} v_{k}) G(k)^{-1}\nabla_{\theta} \log \pi_{\theta_{k}}(a_k|s_k) ]\rangle\\
    &\qquad - \sum_{k=\tau_t}^{t} \frac{a(k)}{b(k)}\mathbb{E}\langle z_k , (\nabla v_k^*)^T  (L(\theta_k,\gamma(k)) - L_k) G(k)^{-1}\nabla_{\theta} \log \pi_{\theta_{k}}(a_k|s_k) \rangle\\
    &\leq \frac{c_a}{c_b} \log^{0.5} t \sum_{k=\tau_t}^{t}\vert \mathbb{E}[\Psi(O_k,v_k,\theta_k,\gamma(k),G(k))]\vert + L_{*}U_{G}\frac{c_a}{c_b}\log^{0.5} t\sqrt{\sum\limits_{k=\tau_t}^{t}\mathbb{E}\Vert z_k\Vert^2}\sqrt{\sum\limits_{k=\tau_t}^{t}\mathbb{E}[\Vert \bar{M}(\theta_k,v_k,\gamma(k))\Vert^2]}\\
    &\qquad + L_{*}BU_{G}\frac{c_a}{c_b}\log^{0.5} t\sqrt{\sum\limits_{k=\tau_t}^{t}\mathbb{E}\Vert z_k\Vert^2}\sqrt{\sum\limits_{k=\tau_t}^{t}\mathbb{E}[(L(\theta_k,\gamma(k))- L_k)^2]} \\
    &= \mathcal{O}(\log^{2.5} t \cdot t^{- \nu + 1}) + L_{*}U_{G}\frac{c_a}{c_b}\log^{0.5} t\sqrt{\sum\limits_{k=\tau_t}^{t}\mathbb{E}\Vert z_k\Vert^2}\sqrt{\sum\limits_{k=\tau_t}^{t}\mathbb{E}[\Vert \bar{M}(\theta_k,v_k,\gamma(k))\Vert^2]}\\
    &\qquad + L_{*}BU_{G}\frac{c_a}{c_b}\log^{0.5} t\sqrt{\sum\limits_{k=\tau_t}^{t}\mathbb{E}\Vert z_k\Vert^2}\sqrt{\sum\limits_{k=\tau_t}^{t}\mathbb{E}[(L(\theta_k,\gamma(k))- L_k)^2]}.
    \end{align*}

    For the term $I_5$, we have,
\begin{align*}
    \sum\limits_{k=\tau_t}^{t}b(k)\mathbb{E}[\delta_k^2\Vert \phi(s_k)\Vert^2] = \mathcal{O}(t^{1-\nu}).
\end{align*}

Next, for the term $I_6$, we have,

\begin{align*}
   &\sum\limits_{k=\tau_t}^{t}\frac{1}{b(k)}\mathbb{E}\Vert v^{*}(\theta_k,\gamma(k)) - v^{*}(\theta_{k+1},\gamma(k+1))  \Vert^2\\
   &= \mathcal{O}(\sum\limits_{k=\tau_t}^{t}\frac{a(k)^2}{b(k)}) + \mathcal{O}(\sum\limits_{k=\tau_t}^{t}\frac{c(k)^2}{b(k)})\\
   &= \mathcal{O}(\log t \cdot t^{1 - \nu }). 
\end{align*}

For detailed analysis of terms $I_5$ and $I_6$ refer to the analysis of term $I_2$ in page number 22 of \cite{wu2022finitetimeanalysistimescale}.\\

After gathering all the terms we have,

\begin{align*}
     \hat{\lambda_e}\sum\limits_{k=\tau_t}^{t}\mathbb{E}\Vert z_k \Vert^2 &\leq \mathcal{O}( t^{\nu}) +  \mathcal{O}(\log^{2.5} t \cdot t^{1-\nu})  + \mathcal{O}(t^{\nu - \beta + 1}) \\
     &\qquad +  L_{*}U_{G}\frac{c_a}{c_b}\log^{0.5} t\sqrt{\sum\limits_{k=\tau_t}^{t}\mathbb{E}\Vert z_k\Vert^2}\sqrt{\sum\limits_{k=\tau_t}^{t}\mathbb{E}[\Vert \bar{M}(\theta_k,v_k,\gamma(k))\Vert^2]}\\
    &\qquad + L_{*}BU_{G}\frac{c_a}{c_b}\log^{0.5} t\sqrt{\sum\limits_{k=\tau_t}^{t}\mathbb{E}\Vert z_k\Vert^2}\sqrt{\sum\limits_{k=\tau_t}^{t}\mathbb{E}[(L(\theta_k,\gamma(k))- L_k)^2]}
\end{align*}

After applying the square technique we have,

\begin{align*}
    \sum\limits_{k=\tau_t}^{t}\mathbb{E}\Vert z_k \Vert^2  &= \mathcal{O}( t^{\nu}) +  \mathcal{O}(\log^{2.5} t \cdot t^{1-\nu})  + \mathcal{O}(t^{\nu - \beta + 1})\\
    &\qquad + \mathcal{O}(\log^{0.5} t \cdot \sum\limits_{k=\tau_t}^{t}\mathbb{E}[\Vert \bar{M}(\theta_k,v_k,\gamma(k))\Vert^2]) + \mathcal{O}(\log^{0.5} t \cdot \sum\limits_{k=\tau_t}^{t}\mathbb{E}[(L(\theta_k,\gamma(k))- L_k)^2])\\
    &= \mathcal{O}( t^{\nu}) +  \mathcal{O}(\log^{3} t \cdot t^{1-\nu})  + \mathcal{O}(\log^{0.5} t \cdot t^{\nu - \beta + 1}) 
\end{align*}

Assuming $t \geq 2\tau_{t} - 1$, we have,

\begin{align*}
    \frac{1}{1+t-\tau_t}\sum\limits_{k=\tau_t}^{t}\mathbb{E}\Vert z_k \Vert^2 = \mathcal{O}( t^{\nu - 1}) +  \mathcal{O}(\log^{3} t \cdot t^{-\nu})  + \mathcal{O}(\log^{0.5} t \cdot t^{\nu - \beta}) 
\end{align*}

Optimising over the values of $\nu$ and $\beta$ we have $\nu = 0.5$ and $\beta = 1$. Hence we have  the following :-
\begin{align*}
    \frac{1}{1+t-\tau_t}\sum_{k=\tau_t}^{t}\mathbb{E}\Vert z_k \Vert^2 &= \mathcal{O}(\log^3 t \cdot t^{ - 0.5})
\end{align*}

Therefore in order for the mean squared error of the critic to be upper bounded by $\epsilon$, namely,

\begin{align*}
     \frac{1}{1+t-\tau_t}\sum_{k=\tau_t}^{t}\mathbb{E}\Vert z_k \Vert^2  =  \mathcal{O}(\log^3 T \cdot T^{- 0.5}) \leq \epsilon,
\end{align*}
we need to set $T = \tilde{\mathcal{O}}(\epsilon^{-2})$.\\

\subsection{Other Experimental Results}

We show here the results of experiments done on the constrained algorithms with neural network approximators, see Table~\ref{tab:experiment} and Figure~\ref{fig:experiments}. 
Further, we show the results of experiments in the unconstrained setting when linear function approximators are used, see Figure~\ref{fig:experiments_CA_L} -- Table~\ref{tab:linear_critic_reward}. 

\subsubsection{Experimental Results with Algorithms for Constrained MDPs}

We show here the results of experiments on three different SafetyGym environments, namely SafetyAntCircle1-v0, SafetyCarGoal1-v0, and SafetyPointPush1-v0, respectively, see Table~\ref{tab:experiment} and Figure~\ref{fig:experiments}. 
We compare here the performance of C-AC, C-NAC, C-CA, C-NCA, C-CAM and C-NCAM algorithms. Each experiment is run with ten different initial seeds and the mean and standard error from these experiments 
C-NCAM shows the best results on two out of the three environments, while C-AC shows the best results on one environment.

{\small 
\begin{table*}
  \centering
  \begin{tabular}{|p{2.6 cm}|p{1.5 cm}|p{1.5 cm}|p{2 cm}|p{2 cm}|p{2 cm}| p{2 cm}|}
    \hline
    \textbf{Environment} & \textbf{C-AC} & \textbf{C-NAC} & \textbf{C-CA} & \textbf{C-CAM} & \textbf{C-NCA} & \textbf{C-NCAM}\\
    \hline
    SafetyAntCircle1-v0 & $\textbf{0.0003} \pm \textbf{0.00037}$  & $-0.000024 \pm 0.0003$  & $-0.00016 \pm 0.00034$   & $0.000066 \pm 0.0001$  & $-0.000033 \pm 0.0001$ & $-0.0005 \pm 0.0002$ \\ \hline
    SafetyCarGoal1-v0 & $-0.00209 \pm 0.0006$  & $-0.0132 \pm 0.0018$ & $-0.0038 \pm 0.001$  & $-0.003 \pm 0.0009$  & $-0.009 \pm 0.0015$ & $\textbf{-0.0001} \pm \textbf{0.0004}$\\ \hline
   SafetyPointPush1-v0 & $ -0.0018 \pm 0.0004$  & $-0.0004 \pm 0.0003$ & $-0.001 \pm 0.0003$ & $-0.0006 \pm 0.0003$ & $-0.002 \pm  0.0005$ & $\textbf{-0.0003} \pm \textbf{0.0001}$\\
    \hline
  \end{tabular}
   \caption{Comparison of Constrained Natural Critic-Actor with Various Algorithms in terms of Average reward $\pm$ Standard Error upon Convergence.}
  \label{tab:experiment}
\end{table*}}

{\small 
\begin{table*}
  \centering
  \begin{tabular}{|p{2.6 cm}|p{1.5 cm}|p{1.5 cm}|p{2 cm}|p{2 cm}|p{2 cm}| p{2 cm}|}
    \hline
    \textbf{Environment} & \textbf{C-AC} & \textbf{C-NAC} & \textbf{C-CA} & \textbf{C-CAM} & \textbf{C-NCA} & \textbf{C-NCAM}\\
    \hline
    SafetyAntCircle1-v0 & $0.0654 \pm 0.0325$  & $0.2095 \pm 0.0504$  & $0.1720 \pm 0.0672$   & $0.0529 \pm 0.0172$  & $0.1600 \pm 0.0377$ & $0.0941 \pm 0.0435$ \\ \hline
    SafetyCarGoal1-v0 & $0.0300 \pm 0.0128$  & $0.0612 \pm 0.0180$ & $0.0193 \pm 0.0062$  & $0.0591 \pm 0.0181$  & $0.0215 \pm 0.0075$ & $0.0366 \pm 0.0113$\\ \hline
   SafetyPointPush1-v0 & $0.0208 \pm 0.0092$  & $0.0552 \pm 0.0272$ & $0.0372 \pm 0.0358$ & $0.0378 \pm 0.0331$ & $0.0308 \pm 0.0203$ & $0.0329 \pm 0.0183$\\
    \hline
  \end{tabular}
  \caption{Comparison of Constrained Natural Critic-Actor with Different Algorithms in terms of Average Constrained Cost $\pm$ Standard Error upon Convergence.}
  \label{tab:constraint_cost}
\end{table*}}

\begin{figure*}
\vspace{.3in}
\centering\includegraphics[scale =0.28]{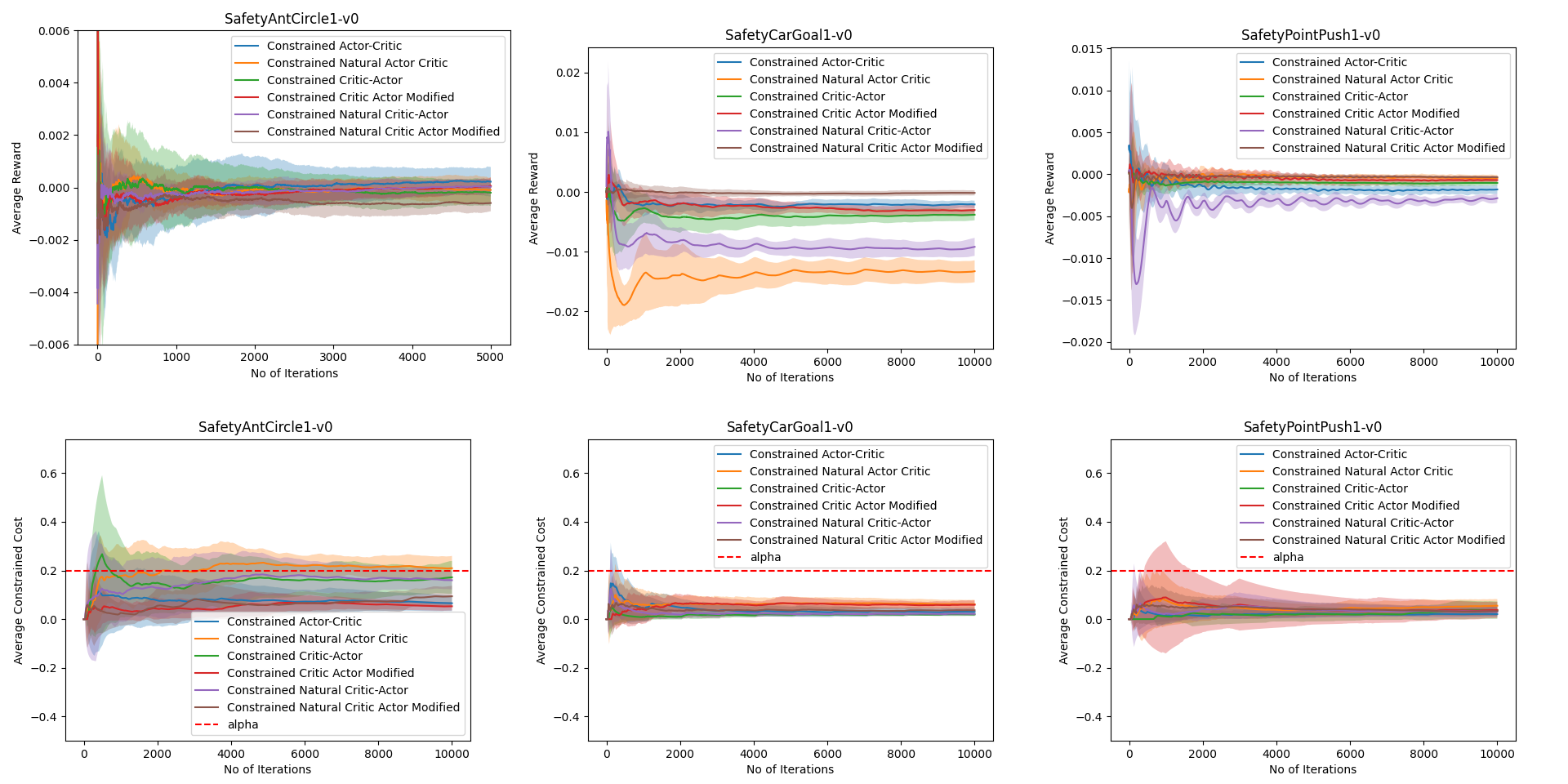}
\vspace{.3in}
\caption{Comparison of Average Reward and Constraint Cost Performance of C-AC, C-NAC, C-CA, C-NCA, C-CAM and C-NCAM Algorithms.}
\label{fig:experiments}
\end{figure*}
\subsubsection{Experimental Results with Unconstrained RL Algorithms with Linear Function Approximation}

We now show the performance comparisons in terms of average reward performance of our algorithm CA-Modified (CA-M) with other algorithms (AC, CA, N-CA, NCA-M) when linear function approximation is used. These experiments are meant to supplement our theoretical findings since we showed our theoretical analysis in the case when linear function approximators are used.
While we did not present the sample complexity of NCA-M, we expect that it would also be $\tilde{{\cal O}}(\epsilon^{-2})$.

{\small 
\begin{table*}
  \centering
  \begin{tabular}{|p{2 cm}|p{1.5 cm}|p{1.5 cm}|p{1.5 cm}|p{2 cm}|p{2 cm}|p{2 cm}|p{2 cm}|}
    \hline
    \textbf{Environment} & \textbf{AC} & \textbf{CA} & \textbf{CAM} & \textbf{NAC} & \textbf{NCA} & \textbf{NCAM} & \textbf{1AC}\\
    \hline
    Acrobot-v1 & $-1.0000 \pm 0.0000$ & $-1.0000 \pm 0.0000$ & $-0.9982 \pm 0.0016$ & $\textbf{-0.9981} \pm \textbf{0.0018}$ & $-0.9984 \pm 0.0018$ & $-0.9985 \pm 0.0013$ & $-1.0000 \pm 0.0000$\\ \hline
    Blackjack-v1 & $-0.284 \pm 0.155$ & $-0.350 \pm 0.163$ & $\textbf{-0.225} \pm \textbf{0.014}$ & $-0.435 \pm 0.017$ & $-0.442 \pm 0.091$ & $-0.279 \pm 0.018$ & $-0.320 \pm 0.165$\\ \hline
   Swimmer-v4 & $-0.0094 \pm 0.0164$ & $-0.459 \pm 1.351$ & $\textbf{-0.0029} \pm \textbf{0.0137}$ & $-0.0041 \pm 0.0156$ & $-0.0071 \pm 0.0158$ & $-0.0032 \pm 0.0136$ & $-0.0163 \pm 0.0224$\\
    \hline
  \end{tabular}
  \caption{Performance Comparison of the Various Algorithms upon Convergence in terms of Average Reward $\pm$ Standard Error when Linear Function Approximation is used for the Critic.}\label{tab:linear_critic_reward}
\end{table*}}

\subsubsection{Hyperparameters details}

The learning rates for the actor, critic, and average reward estimates used in the various algorithms for unconstrained MDP were selected based on the results reported in the following work:

\begin{table*}[t]
  \centering
  \label{tab:ac_algorithms}
  \begin{tabular}{|l|l|}
    \toprule
    \textbf{Algorithm} & \textbf{Reference} \\ \hline
    \midrule
    Actor-Critic
      & \cite{wu2022finitetimeanalysistimescale} \\ \hline
    Critic--Actor
      & \cite{Panda_Bhatnagar_2025} \\ \hline
    Critic--Actor Modified
      & Our Work \\
    Natural Actor--Critic
      & \cite{wu2022finitetimeanalysistimescale} \\ \hline
    Natural Critic--Actor
      & \cite{Panda_Bhatnagar_2025} \\ \hline
    Natural Critic--Actor Modified
      & Our Work \\ \hline
    Single-timescale Actor--Critic
      & \cite{chen_zhao} \\
    \bottomrule
  \end{tabular}
   \caption{References from where learning rates  for various unconstrained algorithms were chosen in our experiments }
\end{table*}

\begin{table*}[t]
  \centering
  \label{tab:cac_algorithms}
  \begin{tabular}{|l|l|}
    \toprule
    \textbf{Algorithm} & \textbf{Reference} \\ \hline
    \midrule
    Constrained Actor-Critic
      & \cite{panda_and_bhatnagar} \\ \hline
    Constrained Critic-Actor
      & \cite{Panda_Bhatnagar_2025} \\ \hline
    Constrained Critic-Actor Modified
      & Our Work \\ \hline
    Constrained Natural Actor-Critic
      & \cite{panda_and_bhatnagar} \\ \hline
    Constrained Natural Critic-Actor
      & \cite{Panda_Bhatnagar_2025} \\ \hline
    Constrained Natural Critic-Actor Modified
      & Our Work \\
    \bottomrule
  \end{tabular}
  \caption{References from where learning rates  for various constrained algorithms were chosen in our experiments }
\end{table*}

The learning rates for the actor, critic, average reward estimates and lagrange multiplier  used in the various algorithms for constrained MDP were selected based on the results reported in the following work:

{\small
\begin{table*}[t]
  \centering
  \begin{tabular}{|p{4.2cm}|p{11.3cm}|}
    \hline
    \textbf{Hyperparameter} & \textbf{Setting} \\
    \hline
    Environments & Acrobot-v1, Blackjack-v1, Swimmer-v4 \\
    \hline
    Training steps & 5000 \\
    \hline
    Random seeds & 42, 123, 256, 512, 789, 1001, 2027, 3141, 4096, 7777 \\
    \hline
    Reported metric & Cumulative average reward: $\frac{1}{t}\sum_{i=1}^{t} r_i$ \\
    \hline
    Episode handling & Reset environment on termination/truncation \\
    \hline
    Critic architecture & $\mathrm{Linear}(d \!\to\! 64) \!\to\! \tanh \!\to\! \mathrm{Linear}(64 \!\to\! 1)$ \\
    \hline
    Critic loss & $\mathcal{L}_c = \tfrac{1}{2}\delta^2$ \\
    \hline
    Critic optimizer & Adam (base learning rate $=1.0$; step size set by schedule below) \\
    \hline
    Actor hidden size & 64 \\
    \hline
    Actor (discrete) & Acrobot-v1, Blackjack-v1: $\mathrm{Linear}(d \!\to\! 64) \!\to\! \tanh \!\to\! \mathrm{Linear}(64 \!\to\! |\mathcal{A}|)$, Categorical policy \\
    \hline
    Actor (continuous) & Swimmer-v4: Gaussian policy, mean net $\mathrm{Linear}(d \!\to\! 64) \!\to\! \tanh \!\to\! \mathrm{Linear}(64 \!\to\! 2) \!\to\! \tanh$, learnable $\log\sigma$ (init.\ 0) \\
    \hline
    Hidden-dimension search & Value net: $\{16,32,64\}$;\; Policy net: $\{16,32,64\}$ (grid search) \\ \hline
    Actor loss & $\mathcal{L}_a = -\log\pi(a|s)\,\delta$ \\
    \hline
    Actor optimizer & Adam (base learning rate $=1.0$) for AC, CA, CAM, 1-AC \\
    \hline
    Model selection criterion & Best training cumulative average reward over 1000 steps \\ \hline
    State/input dimension $d$ &
    Acrobot-v1: 6 (raw Box state);\;
    Blackjack-v1: 45 (one-hot of tuple obs.\ $32+11+2$);\;
    Swimmer-v4: 8 (raw Box state) \\
    \hline
    Action space &
    Acrobot-v1: Discrete(3);\;
    Blackjack-v1: Discrete(2);\;
    Swimmer-v4: Continuous Box$(2)\subset[-1,1]^2$ \\
    \hline
  \end{tabular}

  \vspace{0.4em}
  \begin{tabular}{|p{1.6cm}|p{3.6cm}|p{3.6cm}|p{3.6cm}|p{3.3cm}|}
    \hline
    \textbf{Algorithm} & Actor LR & Critic LR & Average Reward LR & \textbf{Actor update} \\
    \hline
    AC   & $(1+t)^{-0.6}$ & $(1+t)^{-0.4}$ & $(1+t)^{-0.4}$ & Policy gradient (Adam) \\
    \hline
    CA   & $(1+t)^{-0.4}$ & $(1+t)^{-0.6}$ & $(1+t)^{-0.4}$ & Policy gradient (Adam) \\
    \hline
    CA-M & $0.001\,\frac{\sqrt{\ln(1+t)}}{\sqrt{1+t}}$ & $(1+t)^{-0.5}$ & $\frac{\sqrt{\ln(1+t)}}{\sqrt{1+t}}$ & Policy gradient (Adam) \\
    \hline
    NAC  & $(1+t)^{-0.6}$ & $(1+t)^{-0.4}$ & $(1+t)^{-0.4}$ & Natural policy gradient \\
    \hline
    NCA  & $(1+t)^{-0.4}$ & $(1+t)^{-0.6}$ & $(1+t)^{-0.4}$ & Natural policy gradient \\
    \hline
    NCA-M& $0.001\,\frac{\sqrt{\ln(1+t)}}{\sqrt{1+t}}$ & $(1+t)^{-0.5}$ & $\frac{\sqrt{\ln(1+t)}}{\sqrt{1+t}}$ & Natural policy gradient \\
    \hline
    1-AC & $(1+t)^{-0.5}$ & $(1+t)^{-0.5}$ & $(1+t)^{-0.4}$ & Policy gradient (Adam) \\
    \hline
  \end{tabular}
    \caption{Hyperparameters for average-reward actor--critic with \textbf{neural network critic}
  on Acrobot-v1, Blackjack-v1, and Swimmer-v4.}
  \label{tab:hp_nn_critic}
\end{table*}
}

{\small
\begin{table*}[t]
  \centering
  \begin{tabular}{|p{4.2cm}|p{11.3cm}|}
    \hline
    \textbf{Hyperparameter} & \textbf{Setting} \\
    \hline
    Environments & Acrobot-v1, Blackjack-v1, Swimmer-v4 \\
    \hline
    Training steps & 5000 \\
    \hline
    Random seeds & 42, 123, 256, 512, 789, 1001, 2027, 3141, 4096, 7777 \\
    \hline
    Reported metric & Cumulative average reward: $\frac{1}{t}\sum_{i=1}^{t} r_i$ \\
    \hline
    Critic representation & Linear value function: $V(s)=w^\top\phi(s)+b$ \\
    \hline
    Feature map $\phi(s)$ &
    Acrobot-v1, Swimmer-v4: raw state vector;\;
    Blackjack-v1: concatenated one-hot encoding of tuple observation \\
    \hline
    Feature dimension &
    Acrobot-v1: 6;\;
    Blackjack-v1: 45;\;
    Swimmer-v4: 8 \\
    \hline
    Model selection criterion & Best training cumulative average reward over 1000 steps \\ \hline
    Critic initialization & $w=0$, $b=0$ \\
    \hline
    Critic update (semi-gradient TD) &
    $w \leftarrow w + \alpha_c(t)\,\delta\,\phi(s)$,\;
    $b \leftarrow b + \alpha_c(t)\,\delta$ \\
    \hline
    Critic optimizer & None (manual semi-gradient update) \\
    \hline
    Actor architecture & Same neural actor as NN-critic experiments (hidden size 64) \\
    \hline
    Actor (discrete) & Acrobot-v1, Blackjack-v1: $\mathrm{Linear}(d \!\to\! 64) \!\to\! \tanh \!\to\! \mathrm{Linear}(64 \!\to\! |\mathcal{A}|)$, Categorical policy \\
    \hline
    Actor (continuous) & Swimmer-v4: Gaussian policy, mean net $\mathrm{Linear}(d \!\to\! 64) \!\to\! \tanh \!\to\! \mathrm{Linear}(64 \!\to\! 2) \!\to\! \tanh$, learnable $\log\sigma$ (init.\ 0) \\
    \hline
    Actor loss & $\mathcal{L}_a = -\log\pi(a|s)\,\delta$ \\
    \hline
    Actor optimizer & Adam (base learning rate $=1.0$) for AC, CA, CA-M, 1-AC \\
    \hline
    Gradient clipping & $\|\nabla\|_2 \le 0.5$ (actor only) \\
    \hline
    Action space &
    Acrobot-v1: Discrete(3);\;
    Blackjack-v1: Discrete(2);\;
    Swimmer-v4: Continuous Box$(2)\subset[-1,1]^2$ \\
    \hline
  \end{tabular}

  \vspace{0.4em}
  \begin{tabular}{|p{1.6cm}|p{3.6cm}|p{3.6cm}|p{3.6cm}|p{3.3cm}|}
    \hline
    \textbf{Algorithm} & Actor LR & Critic LR & Average Reward LR & \textbf{Actor update} \\
    \hline
    AC   & $(1+t)^{-0.6}$ & $(1+t)^{-0.4}$ & $(1+t)^{-0.4}$ & Policy gradient (Adam) \\
    \hline
    CA   & $(1+t)^{-0.4}$ & $(1+t)^{-0.6}$ & $(1+t)^{-0.4}$ & Policy gradient (Adam) \\
    \hline
    CA-M & $0.001\,\frac{\sqrt{\ln(1+t)}}{\sqrt{1+t}}$ & $(1+t)^{-0.5}$ & $\frac{\sqrt{\ln(1+t)}}{\sqrt{1+t}}$ & Policy gradient (Adam) \\
    \hline
    NAC  & $(1+t)^{-0.6}$ & $(1+t)^{-0.4}$ & $(1+t)^{-0.4}$ & Natural policy gradient \\
    \hline
    NCA  & $(1+t)^{-0.4}$ & $(1+t)^{-0.6}$ & $(1+t)^{-0.4}$ & Natural policy gradient \\
    \hline
    NCA-M& $0.001\,\frac{\sqrt{\ln(1+t)}}{\sqrt{1+t}}$ & $(1+t)^{-0.5}$ & $\frac{\sqrt{\ln(1+t)}}{\sqrt{1+t}}$ & Natural policy gradient \\
    \hline
    1-AC & $(1+t)^{-0.5}$ & $(1+t)^{-0.5}$ & $(1+t)^{-0.4}$ & Policy gradient (Adam) \\
    \hline
  \end{tabular}
  \caption{Hyperparameters for average-reward actor--critic with \textbf{linear function approximation critic}
  on Acrobot-v1, Blackjack-v1, and Swimmer-v4.}
  \label{tab:hp_linear_critic}
\end{table*}
}

{\small
\begin{table*}[t]
  \centering
  \begin{tabular}{|p{4.0cm}|p{3.6cm}|p{3.6cm}|p{3.6cm}|}
    \hline
    \textbf{Hyperparameter} &
    \textbf{SafetyAntCircle1-v0} &
    \textbf{SafetyCarGoal1-v0} &
    \textbf{SafetyPointPush1-v0} \\
    \hline
    Environment ID & SafetyAntCircle1-v0 & SafetyCarGoal1-v0 & SafetyPointPush1-v0 \\
    \hline
    State dimension $d$ & 56 & 72 & 76 \\
    \hline
    Action dimension & 8 (continuous) & 2 (continuous) & 2 (continuous) \\
    \hline
    Action clipping & $[-1,1]$ on sampled actions & $[-1,1]$ on sampled actions & $[-1,1]$ on sampled actions \\
    \hline
    Constraint threshold $\alpha$ & 0.2 & 0.2 & 0.2 \\
    \hline
    Hidden-dimension search & \multicolumn{3}{c|}{Value net: $\{16,32,64\}$;\; Policy net: $\{16,32,64\}$ (grid search)} \\
    \hline
    Model selection criterion & \multicolumn{3}{c|}{Best training cumulative average reward over 1000 steps} \\
    \hline
    Training steps & \multicolumn{3}{c|}{1000} \\
    \hline
    Evaluation steps & \multicolumn{3}{c|}{10000} \\
    \hline
    Evaluation seeds & \multicolumn{3}{c|}{5, 23, 45, 67, 89, 123, 256, 512, 768, 1024 (10 seeds)} \\
    \hline
  \end{tabular}

  \vspace{0.5em}
  \begin{tabular}{|p{4.0cm}|p{11.2cm}|}
    \hline
    \textbf{Hyperparameter} & \textbf{Setting (all algorithms/environments unless noted)} \\
    \hline
    Critic architecture &
    $\mathrm{Linear}(d\!\to\! h_v)\!\to\!\sigma\!\to\!\mathrm{Linear}(h_v\!\to\!1)$ \\
    \hline
    Critic activation &
    Tanh for C-AC, C-CA, C-CA-M;\;
    LeakyReLU for C-NAC, C-NCA, C-NCA-M \\
    \hline
    Policy architecture &
    $\mathrm{Linear}(d\!\to\! h_\pi)\!\to\!\mathrm{LeakyReLU}\!\to\!
     \mathrm{Linear}(h_\pi\!\to\!|\mathcal{A}|)\!\to\!\tanh$ \\
    \hline
    Policy distribution & Gaussian; $\log\sigma$ initialized to $-0.5$ \\
    \hline
    Weight initialization & Orthogonal (gain $\sqrt{2}$) for policy linear layers \\
    \hline
    Critic optimizer & Adam \\
    \hline
    Policy optimizer & Adam for C-AC, C-CA, C-CA-M \\
    \hline
    Gradient clipping & $\|\nabla\|_2 \le 0.5$ (value net and/or policy net) \\
    \hline
  \end{tabular}
  
  \vspace{0.5em}
    \begin{tabular}{|p{1.6cm}|p{3.1cm}|p{3.1cm}|p{3.1cm}|p{3.1cm}|}
    \hline
    \textbf{Algorithm} &
    \textbf{Actor LR}  &
    \textbf{Critic LR}  &
    \textbf{Avg.\ reward LR}  &
    \textbf{Lagrange LR}  \\
    \hline
    C-AC &
    $5\!\times\!10^{-4}\,(1+t)^{-0.6}$ &
    $10^{-3}\,(1+t)^{-0.4}$ &
    $\dfrac{1.5}{(t+1)^{0.4}}$ &
    $\dfrac{1.5}{(1+t)}$ \\
    \hline
    C-CA &
    $5\!\times\!10^{-4}\,(1+t)^{-0.5}$ &
    $10^{-3}\,(1+t)^{-0.51}$ &
    $\dfrac{1.5}{(1+t)^{0.5}}$ &
    $\dfrac{1.5}{(1+t)}$ \\
    \hline
    C-CA-M &
    $5\!\times\!10^{-4}\,\dfrac{\sqrt{\ln(1+t)}}{(1+t)^{0.5}}$ &
    $10^{-3}\,(1+t)^{-0.51}$ &
    $1.5\,\dfrac{\sqrt{\ln (1+t)}}{(t+1)^{0.5}}$ &
    $\dfrac{1.5}{t+1}$ \\
    \hline
    C-NAC &
    $\dfrac{1.5}{(1+t)^{0.6}}$ &
    $3\!\times\!10^{-4}\,(1+t)^{-0.4}$ &
    $\dfrac{1.5}{(1+t)^{0.4}}$ &
    $\dfrac{1.5}{(1+t)}$ \\
    \hline
    C-NCA &
    $\dfrac{0.0005}{(1+t)^{0.5}}$ &
    $3\!\times\!10^{-4}\,(1+t)^{-0.51}$ &
    $\dfrac{1.5}{(1+t)^{0.5}}$ &
    $\dfrac{1.5}{(1+t)}$ \\
    \hline
    C-NCA-M &
    $0.0005\,\dfrac{\sqrt{\ln (1+t)}}{(t+1)^{0.5}}$ &
    $3\!\times\!10^{-4}\,(1+t)^{-0.5}$ &
    $1.5\,\dfrac{\sqrt{\ln (1+t)}}{(t+1)^{0.5}}$ &
    $\dfrac{1.5}{t+1}$ \\
    \hline
  \end{tabular}
  \caption{Hyperparameters for constrained average-reward actor--critic algorithms
  on SafetyAntCircle1-v0, SafetyCarGoal1-v0, and SafetyPointPush1-v0.}
 \label{tab:hp_constrained_safety}
\end{table*}
}


\begin{figure*}
\vspace{.3in}
\centering\includegraphics[scale =0.11]{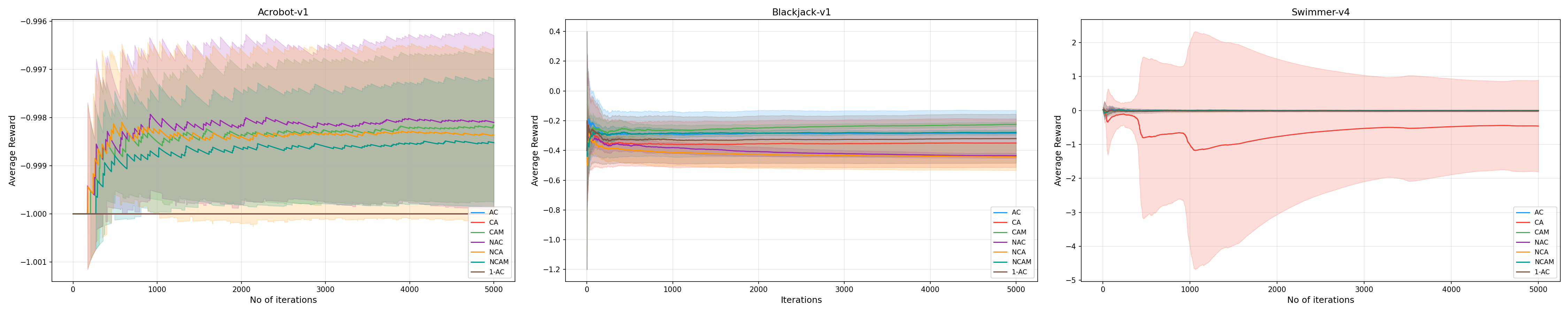}
\vspace{.3in}
\caption{Comparison of Average Reward Performance for AC, CA, CAM, NAC, NCA, N-CAM and 1AC Algorithms for the Case when Linear Function Approximation is used for the Critic.} \label{fig:experiments_CA_L}
\end{figure*}


{\small
\begin{table*}[t]
  \centering
  \begin{tabular}{|p{4.5cm}|p{10.7cm}|}
    \hline
    \textbf{Component} & \textbf{Specification} \\
    \hline
    \multicolumn{2}{|l|}{\textbf{Hardware}} \\
    \hline
    Processor & 2$\times$ AMD EPYC 7713 (64 cores/socket, 128 cores total; 256 hardware threads with SMT) \\
    \hline
    System memory & 1.0~TiB RAM \\
    \hline
    GPU(s) available & 8$\times$ NVIDIA A100-SXM4-40GB (40~GB HBM2e each) \\
    \hline
    GPU driver & NVIDIA 535.216.03 \\
    \hline
    Compute device used & \textbf{CPU} (PyTorch CPU backend; code does not invoke CUDA) \\
    \hline
    \multicolumn{2}{|l|}{\textbf{Operating system}} \\
    \hline
    OS & Ubuntu 20.04.2 LTS (Focal Fossa) \\
    \hline
    Kernel & Linux 5.15.0-130-generic (x86\_64) \\
    \hline
    \multicolumn{2}{|l|}{\textbf{Software environment}} \\
    \hline
    Environment manager & Conda \\
    \hline
    Python & 3.8.18 \\
    \hline
    PyTorch & 2.4.1+cu121 \\
    \hline
    CUDA toolkit (PyTorch build) & 12.1 \\
    \hline
    cuDNN & 9.1.0 \\
    \hline
    NumPy & 1.23.5 \\
    \hline
    Gymnasium & 0.28.1 \\
    \hline
    MuJoCo & 2.3.3 \\
    \hline
    Matplotlib & 3.7.5 \\
    \hline
    cloudpickle & 3.1.1 \\
    \hline
    Farama-Notifications & 0.0.4 \\
    \hline
    \multicolumn{2}{|l|}{\textbf{Experiment execution notes}} \\
    \hline
    Parallelization & Single-process, single-threaded training per run \\
    \hline
    Typical run setting & 5000 training steps $\times$ 10 random seeds per algorithm/environment \\
    \hline
    Result plotting & Matplotlib (CPU), saved as PNG files \\
    \hline
  \end{tabular}
  \caption{Computing infrastructure used for experiments.}
  \label{tab:computing_infrastructure}
\end{table*}
}

\end{document}